\documentclass[twoside,11pt]{article}

\usepackage{blindtext}
\usepackage{bbm}
\usepackage{amsmath}
\usepackage{amssymb}
\usepackage{algorithm}
\usepackage{mathtools}
\usepackage{amsthm}
\usepackage{comment}
\usepackage{natbib}

\usepackage{color}
\definecolor{darkgreen}{rgb}{0,0.5,0}
\definecolor{purple}{rgb}{1,0,1}

\newcount\Comments  % 0 suppresses notes to selves in text
\newcommand{\kibitz}[2]{\ifnum\Comments=1\textcolor{#1}{#2}\fi}
\newcommand{\vinod}[1]{\kibitz{red}      {[VIN: #1]}}
\newcommand{\ambuj}[1]  {\kibitz{blue}   {[AT: #1]}}

\PassOptionsToPackage{algo2e,ruled, linesnumbered}{algorithm2e}
\RequirePackage{algorithm2e}

\usepackage{jmlr2e}

\DeclarePairedDelimiter\floor{\lfloor}{\rfloor}
\newcommand{\reals}{\mathbb{R}}
\newcommand{\naturals}{\mathbb{N}}

\newcommand{\Acal}{\mathcal{A}}
\newcommand{\Bcal}{\mathcal{B}}
\newcommand{\Dcal}{\mathcal{D}}
\newcommand{\Fcal}{\mathcal{F}}
\newcommand{\Gcal}{\mathcal{G}}
\newcommand{\Hcal}{\mathcal{H}}

\newcommand{\Ncal}{\mathcal{N}}

\newcommand{\Tcal}{\mathcal{T}}
\newcommand{\Zcal}{\mathcal{Z}}
\newcommand{\Xcal}{\mathcal{X}}
\newcommand{\Ycal}{\mathcal{Y}}

\DeclareMathOperator*{\argmin}{arg\,min}

\newcommand{\expect}{\operatorname{\mathbb{E}}}

\newcommand{\indicator}{\mathbbm{1}}

\newcommand{\norm}[1]{\left\lVert#1\right\rVert}
\newcommand{\inner}[2]{\left\langle #1, #2 \right\rangle}

\newtheorem{assumption}{Assumption}
\newtheorem*{theorem*}{Theorem}

\usepackage{lastpage}
\jmlrheading{25}{2024}{1-\pageref{LastPage}}{6/23 ; Revised 6/24}{10/24}{23-0787}{Vinod Raman, Unique Subedi, and Ambuj Tewari}
\ShortHeadings{Multioutput Learnability}{Raman, Subedi, and Tewari}

\firstpageno{1}

\begin{document}

\title{A Characterization of Multioutput Learnability}

\author{\name Vinod Raman \thanks{Equal contribution} \email vkraman@umich.edu \\
       \addr Department of Statistics\\
        University of Michigan\\
       Ann-Arbor, MI 48109, USA
       \AND
       \name Unique Subedi \footnotemark[1]\email subedi@umich.edu \\
       \addr Department of Statistics\\
        University of Michigan\\
       Ann-Arbor, MI 48109, USA
       \AND
       \name Ambuj Tewari \email tewaria@umich.edu \\
       \addr Department of Statistics\\
       Department of Electrical Engineering and Computer Science\\
       University of Michigan\\
       Ann-Arbor, MI 48109, USA
       }

\editor{Gergely Neu}

\maketitle

\begin{abstract}%
    We consider the problem of learning multioutput function classes in the batch and online settings. In both settings, we show that a multioutput function class is learnable if and only if each single-output restriction of the function class is learnable. This provides a complete characterization of the learnability of multilabel classification and multioutput regression in both batch and online settings.  As an extension, we also consider multilabel learnability in the bandit feedback setting and show a similar characterization as in the full-feedback setting.     
\end{abstract}

\section{Introduction}
Multioutput learning is a problem where an instance is labeled by a vector-valued target.  This is a generalization of scalar-valued-target learning settings such as multiclass classification and regression. Multioutput learning has enjoyed a wide range of practical applications like image tagging, document categorization, recommender systems, an weather forecasting to name a few. This widespread applicability has motivated the development of several practical methods \citep{Bayes_mmultilabel, 
borchani2015survey,mult_label_deep_forest, xu2013multi, RNN_multilabel}, as well as theoretical analysis \citep{multilabel_consistent, optimality_multilabel}. However, the most fundamental question of learnability in a multioutput setting remains unanswered. 

Characterizing learnability is the foundational step toward understanding any statistical learning problem. The fundamental theorem of statistical learning characterizes the learnability of binary function class in terms of the finiteness of a combinatorial quantity called the Vapnik-Chervonenkis (VC) dimension \citep{vapnik1971uniform, vapnik74theory}. Extending VC theory,  \cite{Natarajan1989} proposed and studied the Natarajan dimension, which was later shown by \cite{David_Bianchi} to characterize learnability in multiclass settings with a finite number of labels.  In fact, \cite{David_Bianchi} observed that the learnability of multiclass function class $\Fcal \subseteq \{e_1, \ldots, e_K\}^{\Xcal}$ can also be characterized in terms of the learnability of each component-wise binary function class $\mathcal{F}_k = \{x \mapsto \inner{f(x)}{e_k}\, :\, f \in \Fcal\}$, where $\{e_1, \ldots, e_K\}$ is the standard basis of $\mathbb{R}^K$ and $\inner{\cdot}{\cdot}$ is the Euclidean inner-product.  Furthermore, they define an equivalence class of loss functions for the $0$-$1$ loss and characterize the learnability of multiclass problems with respect to all losses in the equivalence class. We take a similar approach in this paper.
% by relating the complexity of a vector-valued function class to the complexity of each component-wise restriction and characterizing the learnability with respect to an entire class of loss functions. 
\begin{comment}
 \ambuj{need to say a bit more here--this paper has elements of our basic insights: relating complexities down to component complexity and the identity of indiscernables stuff are both present in seed form here}   
\end{comment}
Closing the question of learnability for multiclass problems, \cite{Brukhimetal2022} shows that the Daniely-Shwartz (DS) dimension, originally proposed by \cite{DanielyShwartz2014}, characterizes multiclass learnability in the infinite labels setting. Similarly, in the online setting, the Littlestone dimension \citep{Littlestone1987LearningQW} characterizes the online learnability of a binary function class and a generalization of the Littlestone dimension \citep{DanielyERMprinciple} characterizes online learnability in the multiclass setting with finite labels. As for scalar-valued regression, the fat shattering \citep{BARTLETT1996} and the sequential fat shattering \citep{rakhlin2015online} dimensions characterize batch and online learnability respectively. Surprisingly, to our best knowledge, no such characterization of the learnability of multioutput function classes exists in the literature. 

In this paper, we close this gap by characterizing the learnability of function classes $\Fcal \subseteq\Ycal^{\Xcal}$, where $\Ycal  \subseteq\reals^K$ is vector-valued target space for some $K \in \naturals$. Let us define scalar-valued function classes $\mathcal{F}_k = \{x \mapsto \inner{f(x)}{e_k}\, :\, f \in \Fcal\}$ for each $k \in [K]$, where $\{e_1, \ldots, e_K\}$ is the standard basis of $\mathbb{R}^K$. Similarly, define $\Ycal_k := \{\inner{y}{e_k}\, :\, y \in \Ycal\}$.  Our main result, informally stated below, asserts that $\Fcal$ is learnable if and only if each coordinate restriction $\Fcal_k$ is learnable. 
\begin{theorem*}(\emph{Informal})
    A multioutput function class $\Fcal \subseteq\Ycal^{\Xcal}$ is learnable if and only if each restriction $\Fcal_k \subseteq\Ycal_k^{\Xcal}$ is learnable.
\end{theorem*}
We prove a version of this result in four canonical settings: batch classification, online classification, batch regression, and online regression.  For the batch settings, we consider the PAC framework and for the online settings, we consider the fully adversarial model. In addition, our result holds for a wide family of loss functions. A unifying theme throughout all four learning settings is our ability to \textit{constructively} convert a learning algorithm $\Acal$ for $\Fcal$ into learning algorithm $\Acal_k$ for $\Fcal_k$ for each $k \in \{1, ..., K\}$ and vice versa.  We show that even for multioutput losses that tightly ``couple" the $K$ coordinates of a function class, their learnability still depends on the learnability of each coordinate. For the batch setting, our algorithmic techniques use the realizable-to-agnostic conversion introduced by \cite{hopkins22a}. In the online setting, we provide a new realizable-to-agnostic conversion similar in the spirit of \cite{hopkins22a}. In principle, both ours and \cite{hopkins22a}'s realizable-to-agnostic conversion is based on the idea of using algorithms to construct a cover of function classes, originally introduced in the seminal work of \cite{ben2009agnostic}.

Our proof techniques, however, do not extend naturally to the case when $K$ is infinite. So, characterizing learnability for an infinite-dimensional target space is an interesting open problem. Moreover, our reductions are computationally inefficient and lead to sub-optimal sample complexities and regret bounds. As such, we leave the construction of efficient and optimal multioutput learning algorithms as an interesting direction for the future work. 

\subsection{Related works}

Multilabel classification has been extensively studied in the batch setting. We review a few works here and also refer the reader to the references therein. \cite{dembczynski2010regret} quantify the 0-1 risk of multilabel classifiers trained by minimizing the Hamming loss and vice versa. \cite{dembczynski2012label} and \cite{chekina2013exploiting} study how exploiting dependencies between labels can improve the predictive performance of multilabel classifiers and how such exploitation interacts with loss minimization. \cite{jain2016extreme} consider the case where the label set is extremely large and design new loss functions to handle these settings. \cite{busa2022regret} derive upper and lower bounds on the excess risk for non-parametric and parametric function classes for various loss functions assuming label sparsity. \cite{gao2011consistency} and 
\cite{koyejo2015consistent} study the consistency of surrogate loss functions for multilabel classification. Finally, \cite{gentile2012multilabel} consider  online multilabel classification under partial feedback and present a novel algorithm based on second-order descent methods.

There is also a long history of studying least squares estimators for multioutput linear models in the statistical literature, see \citep{rao1965theory, brown1980adaptive} and references therein. The topic received widespread attention in learning theory following the seminal work of \cite{micchelli2005learning} in RKHS methods for vector-valued regression.  We refer the reader to a comprehensive review of kernel methods for vector-valued regression by \cite{alvarez2012kernels}. An early work of \cite{gnecco2008estimates} provides estimation and approximation error of vector-valued functions using Rademacher complexity.  Following the influential work of \cite{maurer2016vector} on Rademacher contraction inequalities for vector-valued functions, there have been works on the Rademacher analysis of vector-valued functions (see \cite{cortes2016structured, reeve2020optimistic, yousefi2018local, Foster2019contraction}). Finally, we point out a recent work by \cite{pmlr-v201-park23a} towards developing empirical process theory for vector-valued functions. 

\section{Preliminaries}

Let $\Xcal$ denote the instance space and $\Ycal \subseteq\reals^K$ be the target space for some $K \in \naturals$.  For a space $\Zcal$, we let $\Zcal^{\star}$ be the set of all finite sequences of elements from $\Zcal$. Consider a vector-valued function class $\Fcal \subseteq\Ycal^{\Xcal}$, where $\Ycal^{\Xcal}$ denotes set of all functions from $\Xcal$ to $\Ycal$. For an unlabeled sample $S_U \in \Xcal^{\star}$, let $\Fcal_{|S_U}$ denote the projection of $\Fcal$ onto $S_U.$ Let $\inner{\cdot}{\cdot}$ denote the Euclidean inner-product on $\reals^K$. Define scalar-valued function classes $\mathcal{F}_k = \{x \mapsto \inner{f(x)}{e_k}\, :\, f \in \Fcal\}$ for each $k \in [K]$, where $\{e_1, \ldots, e_K\}$ is the standard basis of $\mathbb{R}^K$.
% Define a scalar-valued function class $\Fcal_k = \{f_k \mid (f_1, \ldots, f_K) \in \Fcal\}$ by restricting each function in $\Fcal$ to its $k^{th}$ coordinate output.
Here, each $\Fcal_k \subseteq\Ycal_k^{\Xcal}$, where $\Ycal_k = \{\inner{y}{e_k}\, :\, y \in \Ycal\}$ denotes the restriction of the target space to its $k^{th}$ component.
% Conveniently, we write $\Fcal = (\Fcal_1, \ldots, \Fcal_K)$ and $\Ycal = (\Ycal_1, \ldots, \Ycal_K)$ \ambuj{I'm confused. Taking coordinate restrictions and putting them back together  not give the original class back. So this $\Fcal$ and $\Ycal$ may not be the original ones, right?}.

%VR: we never actually use this notation so we removed it%
For a function $f \in \Fcal$, we use $f_k(x) := \inner{f(x)}{e_k}$ to denote the $k^{th}$ coordinate output of $f(x)$. On the other hand, we  use $y^{k} := \inner{y}{e_k}$ to denote $k^{th}$ coordinate of $y \in \Ycal$. Additionally, it is useful to distinguish between the range space $\Ycal$ and the image of functions $f \in \Fcal$. We  define the image of function class $\Fcal$ as $\text{im}(\Fcal) := \cup_{f \in \Fcal}\, \text{im}(f)$, where $\text{im}(f) = \{f(x) \, :\, x \in \Xcal\}$. Finally, we take $[N] := \{1, 2, \ldots, N\}.$   
 
 In this work, we only consider bounded, non-negative loss functions  $\ell: \Ycal \times \Ycal \to \reals_{\geq 0}$ that satisfy the identity of the indiscernibles. For the remainder of the paper, we drop the adjectives ``bounded" and ``non-negative" when referring to loss functions.  
 
\begin{definition}[Identity of Indiscernibles]
    A loss function $\ell : \Ycal \times \Ycal \to \reals_{\geq 0}$ satisfies identity of indiscernibles whenever $\ell(y_1, y_2) =0$ if and only if $y_1 = y_2$. 
\end{definition}
Note that if $\ell_1$ and $\ell_2$ are two losses defined on $\Ycal \times \Ycal$ that satisfy the identity of indiscernibles, then $\ell_1(y_1, y_2) =0$ if and only if $\ell_2(y_1, y_2) = 0$. We also define a notion of approximate subadditivity, although not all the loss functions we consider have this property.
\begin{definition}[$c$-subadditive]
 A loss function $\ell$ is $c$-subadditive if there exists a constant $c> 0$ that only depends on the loss function $\ell$ such that $\ell(y_1, y_2) \leq c\, \ell(y_1, y) + \, \ell(y, y_2)$ for all $y, y_1, y_2 \in \Ycal$.   
\end{definition}
If $|\Ycal| < \infty$, $\ell$ being $c$-subadditive is an immediate consequence of $\ell$ satisfying the identity of indiscernibles. In fact, the value of $c$ in this case is $\frac{\max_{r \neq t} \ell(r, t)}{\min_{r \neq t} \ell(r, t)}$. To see why this is true, it suffices to only consider the case when $\ell(y_1, y_2) > \ell(y, y_2)$ because the inequality is trivially true otherwise. Since the loss values are distinct, we must have $y \neq y_1$. Using the identity of indiscernible, we obtain $\ell(y_1, y) \geq \min_{r \neq t} \ell(r,t)$, thus implying that $c \, \ell(y_1, y) \geq \max_{r \neq t} \ell(r, t) \geq \ell(y_1, y_2)$. The case when $|\Ycal| = \infty$ is a bit delicate because $\min_{r \neq t} \ell(r, t)$ may not exist. So, one needs extra structure in the loss function to infer $c$-subadditivity. For instance, if $\ell$ is a distance metric, then it is trivially $1$-subadditive due to the triangle inequality. 

\subsection{Batch Setting}

In the batch setting, we are interested in characterizing the learnability of $\Fcal$ under the classical PAC models: both in the original realizable formulation \citep{Valiant1984ATO} and in the agnostic extension \citep{kearns1994toward}.

\begin{definition}[Agnostic Multioutput Learnability]
A function class $\Fcal$ is agnostic  learnable with respect to\ loss $\ell: \Ycal \times \Ycal \to \reals_{\geq 0}$, if there exists a function $m:(0,1)^2 \to \naturals$ and a learning algorithm $\Acal: (\Xcal \times \Ycal)^{\star} \to \Ycal^{\Xcal}$ with the following property:  for every $\epsilon, \delta \in (0, 1)$ and for every distribution $\Dcal $ on $\Xcal \times \Ycal$, running algorithm $\Acal$ on $n \geq m(\epsilon, \delta)$ iid samples from $\Dcal$ outputs a predictor $g=\Acal(S)$ such that with probability at least $ 1-\delta$ over $S \sim\Dcal^{n}$,
\[\expect_{\Dcal}[\ell(g(x), y )] \leq \inf_{f \in \Fcal} \expect_{\Dcal}[\ell(f(x),y)] + \epsilon.\]
\end{definition}
\noindent Note that we do not require the output predictor $\Acal(S)$ to be in $\Fcal$,  but only require $\Acal(S)$ to compete with the best predictor in $\Fcal$. If we restrict distribution $\Dcal$ to a class such that $\inf_{f \in \Fcal} \expect_{\Dcal}[\ell(f(x),y )] =0$, then we get realizable learnability. 

The learnability of a function class is generally characterized in terms of the complexity measure of the function class. As stated in the introduction, the VC dimension  characterizes the learnability of binary function classes \citep{vapnik74theory}, and so does the Natarajan dimension for multiclass classification \citep{David_Bianchi}. In a scalar-valued regression problem, the fat-shattering dimension of the function class characterizes learnability with respect to the absolute-value and squared loss \citep{BARTLETT1996, alon_scale}. In particular, a real-valued function class $\mathcal{G} \subseteq[0, B]^{\Xcal}$ is learnable if and only if its fat-shattering dimension, denoted as $\text{fat}_{\gamma}(\mathcal{G})$, is finite for every scale $\gamma > 0$.  We extend this characterization to a wide range of loss functions in Lemma \ref{lem:batch_NFLT}.
Finally, for a real-valued class $\mathcal{G}$, we also use a more general notion of complexity measure called Rademacher complexity, denoted as $\mathfrak{R}_n(\mathcal{G})$, that provides a sufficient condition for learnability \citep{Bartlett_Mendelson}.  Precise definitions of all these complexity measures are provided in Appendix \ref{appdx:complexity-batch}.

One recurring theme in this work is to first construct a realizable multioutput learner and then convert it into an agnostic multioutput learner. It is well known that realizable learnability and agnostic learnability are equivalent for multiclass classification problems with respect to $0$-$1$ loss, (see  \cite{David_Bianchi}, \cite[Theorem 6.7]{ShwartzDavid}). Lemma \ref{realizable_agnostic_equiv}, which is an immediate consequence of \cite[Theorem 18]{hopkins22a}, extends this equivalence between realizable and agnostic learning to general loss functions and target spaces. 

\begin{lemma}[\cite{hopkins22a}]
\label{realizable_agnostic_equiv}
    Consider a  function class $\Fcal \subseteq \Ycal^{\Xcal}$ such that $|\text{im}(\Fcal)| < \infty$ and a general loss function $\ell: \Ycal \times \Ycal \to \reals_{\geq 0}$ that is $c$-subadditive. Then, $\Fcal$ is realizable PAC learnable with respect to\ $\ell$ if and only if $\Fcal$ is agnostic PAC learnable with respect to\ $\ell$.
\end{lemma}

 The result of \cite[Theorem 18]{hopkins22a} is stated for the case when $|\Ycal|< \infty$. However, in the regression setting, we need a slightly general version of their result to handle $\alpha$-discretized function classes $\Fcal^{\alpha} \subseteq \Ycal^{\Xcal}$ (see Proof of Theorem \ref{thm:batch_reg_nondecom}) where $|\text{im}(\Fcal^{\alpha})|<\infty$ but $|\Ycal| = |\, [0,1]^K|=\infty$. Nevertheless, the proof of Lemma \ref{realizable_agnostic_equiv} requires only a minor modification to that of \cite[Theorem 18]{hopkins22a}. Given the central role of this result in our characterization, we provide full proof of Lemma \ref{realizable_agnostic_equiv} in Section \ref{sec:batchclass_genloss}. Finally, we note that agnostic-to-realizable conversions in the batch setting are also possible via boosting and compression-based arguments \citep{montasser2019vc, attias2023adversarially}. We use the conversion of \cite{hopkins22a} due to its generality and simplicity.

\subsection{Online Setting} 

In the online setting, an adversary plays a sequential game with the learner over $T$ rounds. In each round $t \in [T]$, an adversary selects a labeled instance $(x_t, y_t) \in \mathcal{X} \times \mathcal{Y}$ and reveals $x_t$ to the learner. The learner makes a (potentially randomized) prediction $\hat{y}_t \in \mathcal{Y}$. Finally, the adversary reveals the true label $y_t$, and the learner suffers the loss $\ell(y_t, \hat{y_t})$, where $\ell$ is some pre-specified loss function. Given a function class  $\mathcal{F} \subseteq \mathcal{Y}^{\mathcal{X}}$, the goal of the learner is to output predictions $\hat{y_t}$ such that its cumulative loss is close to the best possible cumulative loss over functions in $\mathcal{F}$. A function class is online learnable if there exists an algorithm such that for any sequence of labeled examples $(x_1, y_1), ..., (x_T, y_T)$, the difference in cumulative loss between its predictions and the predictions of the best possible function in $\mathcal{F}$ is small.

\begin{definition} [Online Multioutput Learnability]
\label{def:agnOL}
A multioutput function class $\Fcal$ is online learnable with respect to loss $\ell$, if there exists an (potentially randomized) algorithm $\mathcal{A}$ such that for any adaptively chosen sequence of labelled examples $(x_t, y_t) \in \mathcal{X} \times \mathcal{Y}$, the algorithm outputs $\mathcal{A}(x_t) \in \mathcal{Y}$ at every iteration $t \in [T]$ such that 

$$\mathbb{E}\left[\sum_{t=1}^T \ell(\mathcal{A}(x_t), y_t) - \inf_{f \in \mathcal{F}}\sum_{t=1}^T \ell(f(x_t), y_t)\right] \leq R(T) $$
where the expectation is taken with respect to\ the randomness of $\mathcal{A}$ and that of the possibly adaptive adversary, and $R(T): \mathbb{N} \rightarrow \mathbb{R}^+$ is the additive regret: a non-decreasing, sub-linear function of $T$.
\end{definition}

% If for the sequence of labelled examples, there exists a $f \in \mathcal{F}$ s.t for all $t \in [T]$, $f(x_t) = y_t$, then we say that the sequence is \textit{realizable}.
%
If it is guaranteed that the learner always observes a sequence of examples labeled by some function $f \in \mathcal{F}$, then we say we are in the \textit{realizable} setting. On the other hand, if the true label $y_t$ is not revealed to the learner in each round $t \in [T]$ and the adversary only reveals the learner's \textit{loss} $\ell(\mathcal{A}(x_t), y_t)$ then we say we are in the \textit{bandit} setting. 

% We note that the realizable and bandit online settings are typically considered only when the image space $\text{im}(\mathcal{F})$ and label space $\mathcal{Y}$ respectively are finite. This distinction becomes important since in online multioutput regression the label space is infinite, but we still consider a discretized function class with finite image space.

The online learnability of scalar-valued function classes $\mathcal{H} \subseteq\mathcal{Y}^{\mathcal{X}}$ has been characterized. For example, when $\mathcal{Y}$ is finite (i.e.  $\mathcal{Y} = [K]$ for some $K \in \mathbb{N}$), the Multiclass Littlestone Dimension (MCLdim) of $\mathcal{H} \subseteq\mathcal{Y}^{\mathcal{X}}$  characterizes online learnability with respect to\ the 0-1 loss. 
% Before we give the definition, it will be useful to define that a binary tree $\mathcal{T}$ is $\mathcal{Z}$-valued if its internal nodes are labelled by elements from $\mathcal{Z}$.
% \begin{definition} [Multiclass Littlestone Dimension] Let $\mathcal{T}$ denote a $\mathcal{X}$-valued binary tree of depth $d$  whose edges are labelled by elements from $\mathcal{Y}$, such that the edges from a single parent to its child-nodes are each labeled with a different label. The tree $\mathcal{T}$ is shattered by a function class $\mathcal{F} \subseteq\mathcal{Y}^{\mathcal{X}}$ if, for every path from root to leaf which traverses the nodes $x_1, ..., x_d$, there is a function $f \in \mathcal{F}$ such that $f(x_i)$ is the label of the edge $(x_i, x_{i+1})$. The Multiclass Littlestone Dimension (MCLdim) of $\mathcal{F}$, denoted $\text{MCLdim}(\mathcal{F})$, is the maximal depth of a complete binary tree that is shattered by $\mathcal{H}$. If $\text{MCLdim} = \infty$, then there exists shattered trees of arbitrarily large depth. 
% \end{definition}
A function class $\mathcal{H}$ is online learnable with respect to\ the 0-1 loss if and only if $\text{MCLdim}(\mathcal{H})$ is finite \citep{DanielyERMprinciple}. Moreover, $\text{MCLdim}(\mathcal{H})$ tightly captures the best achievable regret in both the realizable and agnostic settings \citep{DanielyERMprinciple}. When the label space $\mathcal{Y}$ is a bounded subset of $\mathbb{R}$, the \textit{sequential} fat-shattering dimension of a real-valued function class $\mathcal{H} \subseteq\mathbb{R}^{\mathcal{X}}$, denoted $\text{fat}^{\text{seq}}_{\gamma}(\mathcal{H})$, characterizes the online learnability of $\mathcal{H}$ with respect to the absolute value loss \citep{rakhlin2015online}. Unlike the Littlestone dimension, note that the sequential fat-shattering dimension is a scale-sensitive dimension. That is, $\text{fat}^{\text{seq}}_{\gamma}(\mathcal{H})$ is defined at every scale $\gamma > 0$. Accordingly,  a real-valued function class $\mathcal{H} \subseteq\mathbb{R}^{\mathcal{X}}$ is learnable with respect to the absolute value loss if and only if its sequential fat-shattering dimension is finite at \textit{every} scale $\gamma > 0$ \citep{rakhlin2015online}.  We extend this characterization to a wide range of loss functions in Lemma \ref{lem:onlinereg_scalar}. Beyond scalar-valued learnability, for any label space $\mathcal{Y}$, function class $\mathcal{H} \subseteq\mathcal{Y}^{\mathcal{X}}$, and loss function $\ell: \mathcal{Y} \times \mathcal{Y} \rightarrow \mathbb{R}_{\geq 0}$, the \textit{sequential} Rademacher complexity of the loss class $\ell \circ \mathcal{H}$, denoted $\mathfrak{R}^{\text{seq}}_T(\ell \circ \mathcal{H})$, is a useful tool for providing sufficient conditions for online learnability \citep{rakhlin2015online}. See Appendix \ref{app:complexonline} for complete definitions.

% \begin{definition} [Sequential Rademacher Complexity] Let $\mathcal{F} \subseteq\mathbb{R}^{\mathacl{X}}$ be a real-valued function class and $\ell$ be a bounded loss function. Let $\mathcal{Z} = \ell \circ \mathcal{F}$ denote the loss class. Let $\sigma_1, ... \sigma_T$ be independent Rademacher random variables. Given a $\mathcal{X}$
% \end{definition}

% Theorem [blank] shows that a function class $\mathcal{F} \subseteq\mathcal{Y}^{\mathcal{X}}$ is online learnable with respect to to loss $\ell$ if the sequential Rademacher complexity of its loss class, denoted $\mathfrak{R}_T(\ell \circ \mathcal{F})$, is a sublinear function of $T$.

% \begin{theorem}
% \end{theorem}

Like the batch setting, a key technique we use to prove online learnability is to first construct a realizable online learner and then convert it into an agnostic online learner. However, unlike the batch setting, there is no known generic algorithm that converts a (potentially randomized) realizable online learner into an agnostic online learner. Thus, one of the contributions of this work is Theorem \ref{thm:onlreal2agn}, informally stated below, which provides an online analog of the realizable-to-agnostic conversion from \cite{hopkins22a}.

\begin{theorem*} (Informal)
Let $\mathcal{F} \subseteq\mathcal{Y}^{\mathcal{X}}$ be a multioutput function class such that $|\text{im}(\mathcal{F})| < \infty$ and $\ell: \mathcal{Y} \times \mathcal{Y} \rightarrow \mathbb{R}_{\geq 0}$ be any bounded, $c$-subadditive loss function. If $\mathcal{F}$ is online learnable with respect to $\ell$ in the realizable setting, then $\mathcal{F}$ is online learnable with respect to $\ell$ in the agnostic setting. 
\end{theorem*}

\section{Batch Multioutput Learnability}

\subsection{Batch Multilabel Classification}\label{sec:batch_classification}

In this section, we study the learnability of batch multilabel classification. Accordingly, let $\Ycal = \{-1,1\}^{K}$. First, we consider the learnability of a natural decomposable loss. Then, we extend the result to more general non-decomposable losses that satisfy the identity of indiscernible.  We want to point out that a multilabel classification with $K$ labels can be viewed as a multiclass classification with $2^K$ labels. With this viewpoint, the Natarajan dimension of $\mathcal{F}$ continues to characterize the batch multilabel learnability for any loss satisfying the identity of indiscernibles (see \cite[Section 4]{David_Bianchi}). For the sake of completeness, we also provide proof of this characterization in Appendix \ref{app:natdimchar}. However, it is more natural to view a multilabel classification as $K$ different binary classification problems as opposed to a multi-class classification problem with $2^K$ labels. Exploiting this  natural decomposability of a multilabel function class, we relate the learnability of $\Fcal$ to the learnability of each component $\Fcal_k$.

\subsubsection{Characterizing Batch Learnability for the Hamming Loss} \label{sec:batch-hamming}

A canonical and natural loss function for multilabel classification is the Hamming loss, defined as $\ell_{H}(f(x), y) := \sum_{i=1}^K \indicator \left\{f_i(x) \neq y^{i} \right\},$
where $f(x) = (f_1(x), \ldots, f_K(x))$ and $y = (y^{1}, \ldots, y^{K})$. The following result establishes an equivalence between the learnability of $\Fcal$ with respect to Hamming loss and the learnability of each $\Fcal_k$ with respect to $0\text{-}1$ loss. 
\begin{theorem}\label{thm:multi_label_batch_hamming}
    A function class $\Fcal \subseteq\Ycal^{\Xcal}$ is agnostic PAC learnable with respect to $\ell_H$ if and only if each of $\Fcal_k \subseteq\Ycal_k^{\Xcal}$ is agnostic PAC learnable with respect to the $0\text{-}1$ loss.

\end{theorem}

    \begin{proof} We first prove that learnability of each component is sufficient followed by the proof of necessity.

    \noindent\textbf{Part 1: Sufficiency.}
        Here our goal is to prove that the agnostic PAC learnability of each $\Fcal_k$ is sufficient for agnostic PAC learnability of $\Fcal$. Our proof is constructive:  given oracle access to agnostic PAC learners $\Acal_k$  for each $\Fcal_k$ with respect to $0\text{-}1$ loss, we construct an agnostic PAC learner $\Acal$ for $\Fcal$ with respect to $\ell_H$. Let $\Dcal$ be arbitrary distribution on $\Xcal \times \Ycal$ and $S = \{(x_i, y_i)\}_{i=1}^n \sim \Dcal^n$ be iid samples from distribution $\Dcal$. Denote $\Dcal_k$ to be the marginal distribution of $\Dcal$ restricted to $\Xcal \times \Ycal_k$. Then, for all $k \in [K]$, the marginal samples  $S_k = \{(x_i, y_{i}^k)\}_{i=1}^n $  with scalar-valued targets are iid samples from $\Dcal_k$. For each $k \in [K]$, define $h_k = \Acal_k(S_k)$
to be the hypothesis returned by algorithm $\Acal_k$ when trained on $S_k$. 

% We will now show that the concatenated predictor $h = (h_1, \ldots, h_K)$ achieves agnostic PAC bounds for $\Fcal$ with respect to $\ell_H$. 

\begin{comment}
\begin{algorithm}
\caption{Agnostic PAC Learner for $\Fcal$ with respect to $\ell_H$}
\label{alg:mult_suff}
\setcounter{AlgoLine}{0}
\KwIn{ Agnostic PAC learners $\{\Acal_k\}_{k=1}^K$ for $\Fcal_k$'s and  samples $S = \{(x_i, y_i)\}_{i=1}^n \sim \Dcal^n$ on $\Xcal \times \Ycal$.}

Construct marginal $S_k = \{(x_i, y_{i}^k)\}_{i=1}^n $ for all $k \in [K]$ with scalar-valued targets  for $\Acal_k$'s.
 
Get hypothesis $h_k = \Acal_k(S_k)$ for all $k \in [K]$.

Output $h = (h_1, \ldots, h_K)$.
\end{algorithm}

\noindent Note that Algorithm \ref{alg:mult_suff} could be improper as the predictor $h$ may not necessarily be in $\Fcal$. In fact, each of the algorithms $\Acal_k$ could be improper. Next, we will show that Algorithm \ref{alg:mult_suff} is an agnostic PAC learner for $\Fcal$ with respect to $\ell_H$. 
\end{comment}

 Let $m_k(\epsilon,\delta)$ denote the sample complexity of $\Acal_k$.  Since $\Acal_k$ is an agnostic PAC learner for $\Fcal_k$, we have that for $n \geq \max_k m_k(\frac{\epsilon}{K}, \frac{\delta}{K}) $, with probability at least $1- \delta/K$ over samples $S_k \sim \Dcal_k^n$, 
\[\expect_{\Dcal_k} \left[\indicator \left\{h_k(x) \neq y^k \right\} \right] \leq \inf_{f_k \in \Fcal_k} \expect_{\Dcal_k} \left[\indicator \left\{f_k(x) \neq y^k \right\} \right] + \frac{\epsilon}{K}.\]
Summing these risk bounds over all coordinates $k$ and using union bounds over probabilities, we get that with probability at least $1- \delta$ over samples $S \sim \Dcal^n$, we obtain $\sum_{k=1}^K \expect_{\Dcal_k}\left[\indicator \left\{h_k(x) \neq y^k \right\} \right] \leq \sum_{k=1}^K\inf_{f_k \in \Fcal_k} \expect_{\Dcal_k}\left[\indicator \left\{f_k(x) \neq y^k \right\} \right] + \epsilon.$ Now using the fact that $\Fcal \subseteq \Fcal_1 \times \ldots \times \Fcal_K$ followed by the linearity of expectation gives 
\[\expect_{\Dcal}\left[\sum_{k=1}^K \indicator \left\{h_k(x) \neq y^k \right\} \right] \leq  \inf_{f \in \Fcal}\expect_{\Dcal}\left[\sum_{k=1}^K\indicator \left\{f_k(x) \neq y^k \right\} \right] + \epsilon.\]
This completes our proof as it shows that the learning rule that concatenates the predictors returned by each $\Acal_k$ on the marginalized samples $S_k$ is an agnostic PAC learner for $\mathcal{F}$ with respect to $\ell_H$ with sample complexity at most $\max_k m_{k}(\epsilon/K, \delta/K)$. 

\noindent
\textbf{Part 2: Necessity.}
Next, we show that if $\Fcal$ is learnable with respect to $\ell_H$, then each $\Fcal_k$ is PAC learnable with respect to the $0\text{-}1$ loss. Our proof is again based on reduction: given oracle access to agnostic PAC learner $\Acal$ for $\Fcal$, we construct an agnostic PAC learner $\Acal_1$ for $\Fcal_1$. A similar construction can be used for all other $\Fcal_k$'s.

Let $\Dcal_1$ be arbitrary distribution on $\Xcal \times \Ycal_1$ and $S = \{(x_i, y_i^1)\}_{i=1}^n$ be iid samples from $\Dcal_1$. In order to use the algorithm $\Acal$, we first augment the samples $S$ to create samples with $K$-variate target. Define an augmented sample $\widetilde{S} = \{(x_i, (y_{i}^1, \ldots, y_{i}^K))\}_{i=1}^n$ such that $y_{ik} \sim \{-1,1\} $ each with probability $1/2$ for all $i \in [n]$ and $k \in \{2, \ldots, K\}$ . Next, we run $\Acal$ on $\widetilde{S}$ and obtain the hypothesis $h = (h_1, \ldots, h_K)= \Acal(\widetilde{S})$. We now show that $h_1$ obtains agnostic PAC bounds.

Consider a distribution $\widetilde{\Dcal}$ on $\Xcal \times \Ycal$  such that a sample $(x, (y^1, \ldots, y^K))$ from $\widetilde{\Dcal}$ is obtained by first sampling $(x, y^1) \sim \Dcal_1$ and appending $ y^k$'s sampled independently from uniform distribution on $\{-1,1\}$ for each $k \in \{2, \ldots, K\}$.  Let $m( \epsilon,\delta, K)$ denote the sample complexity of $\Acal$. Since $\Acal$ is an agnostic PAC learner for $\Fcal$,  for $n \geq m(\epsilon, \delta, K)$, with probability at least $1-\delta$,  we have 
\[\expect_{\widetilde{\Dcal}}\left[ \sum_{k=1}^K\indicator\left\{h_k(x) \neq y^k \right\}\right] \leq \inf_{f \in \Fcal}\expect_{\widetilde{\Dcal}}\left[ \sum_{k=1}^K\indicator \left\{f_k(x) \neq y^k \right\}\right] + \epsilon.\]
 For $k \geq 2$, since the target is chosen uniformly at random from $\{-1,1\}$, the $0\text{-}1$ risk of any predictor is $1/2$. Therefore, the expression above can be written as 
\[\expect_{\Dcal_1}\left[\indicator \left\{h_1(x) \neq y^1 \right\} \right] + \sum_{k=2}^K 1/2 \leq \inf_{f \in \Fcal} \left(\expect_{\Dcal_1}\left[\indicator \left\{f_1(x) \neq y^1 \right\} \right] + \sum_{k=2}^K 1/2 \right) + \epsilon,\]
which  reduces to $\expect_{\Dcal_1}[\indicator \left\{h_1(x) \neq y^1 \right\} ] \leq  \inf_{f_1 \in \Fcal_1} \expect_{\Dcal_1}[\indicator \left\{f_1(x) \neq y^1 \right\}] + \epsilon$. Therefore, $\Fcal_1$ is agnostic PAC learnable with respect to $0$-$1$ loss with sample complexity at most $m(\epsilon, \delta, K)$.
\end{proof}

\subsubsection{ Characterizing Batch Learnability for General Losses} \label{sec:batchclass_genloss}

In this section, we characterize the learnability for general multilabel losses. Our main technical tool in characterizing the learnability for general loss functions is the equivalence between realizable and agnostic learning guaranteed by Lemma \ref{realizable_agnostic_equiv}. Thus, we first provide the proof of that lemma before we proceed further. 
\begin{proof}(of Lemma \ref{realizable_agnostic_equiv}) Note that agnostic learnability implies realizable learnability by definition. So, it suffices to show that realizable learnability of $\Fcal$ with respect to $\ell$ implies agnostic learnability. Our proof here is constructive. That is, given a realizable algorithm $\Acal$ for $\Fcal$, we provide an algorithm, stated as Algorithm \ref{alg:batch_real_agnos_equiv}, that  constructs an agnostic learner for $\Fcal$. 
  \begin{algorithm}
\caption{Agnostic  learner  for $\Fcal$ with respect to $\ell$}
\label{alg:batch_real_agnos_equiv}
\setcounter{AlgoLine}{0}
\KwIn{Realizable learner $\Acal$ for $\Fcal$ with respect to $\ell$, unlabeled samples $S_U \sim \Dcal_{\Xcal}$, and  different labeled samples $S_L \sim \Dcal$ independent from $S_{U}$ }

Run $\Acal$ over all possible labelings of $S_U$ by $\Fcal$ to construct a concept class 
\[C(S_U) := \left\{\Acal\big(S_U, f(S_U) \big) 
 \mid f \in \Fcal_{|S_U}\right\}.\]

Return $\hat{g} \in C(S_U)$  with the lowest empirical error over $S_L$ with respect to $\ell$. 
\end{algorithm}

Let $\Dcal$ be an arbitrary distribution over $\Xcal \times \Ycal$. Define 
\[f^{\star} := \inf_{f \in \Fcal} \expect_{\Dcal}[\ell(f(x), y)] \]
to be the optimal predictor in $\Fcal$. Now consider a predictor $g = \Acal(S_U, f^{\star}(S_U)) \in C(S_U)$ returned by $\Acal$ when trained on samples $S_U$ labeled by $f^{\star}$. Note that $g$ exists in $C(S_U)$ because we consider all possible labelings of $S_U$ by $\Fcal$ and there must be a sample labeled by $f^{\star}$ as well. Let $m_{\Acal}(\epsilon, \delta,K)$ be the sample complexity of the algorithm $\Acal$. Since $\Acal$ is a realizable learner for $\Fcal$, for $|S_U| \geq m_{\Acal}\left(\frac{\epsilon}{2c}, \frac{\delta}{2}, K\right)$ with probability $1-\delta/2$ over sample $S_U \sim \Dcal_{\Xcal}$, we have
\[\expect_{\Dcal_{\Xcal}}\left[ \ell(g(x), f^{\star}(x))\right] \leq \frac{\epsilon}{2c}, \]
where $\Dcal_{\Xcal}$ is the marginal distribution of $\Dcal$ restricted to $\Xcal$ and $c$ is the subaddtivity constant of $\ell$. Since the loss function is $c$-subadditive, for any $(x, y) \in \Xcal \times \Ycal$, we have $\ell(g(x), y) \leq \ell(f^{\star}(x), y) + c\, \ell(g(x), f^{\star}(x))$ pointwise.
Taking expectation with respect to $(x, y) \sim \Dcal$, we obtain
\[\expect_{\Dcal}[\ell(g(x), y)] \leq  c \expect_{\Dcal_{\Xcal}}[\ell(g(x), f^{\star}(x))] + \expect_{\Dcal}[ \ell(f^{\star}(x), y)] \leq \expect_{\Dcal}[ \ell(f^{\star}(x), y)] + \frac{\epsilon}{2}, \]
where the inequality holds with probability $\geq 1-\delta/2$.
Thus, we have shown that there exists a predictor $g \in C(S_U)$ that achieves agnostic PAC bounds for $\Fcal$ with respect to $\ell$. Let $\ell(\cdot, \cdot ) \leq b$ be the upper bound on $\ell$. Recall that by Hoeffding's inequality and union bound, with probability at least $1-\delta/2$ over sample $S_L \sim \Dcal$, the empirical risk of every hypothesis in  $C(S_U)$ on a sample of size $\geq \frac{8b^2}{\epsilon^2} \log{\frac{4|C(S_U)|}{\delta}}$ is at most $\epsilon/4$ away from its population risk. So, if $|S_L| \geq  \frac{8b^2}{\epsilon^2} \log{\frac{4|C(S_U)|}{\delta}}$, then with probability at least $1-\delta/2$ over sample $S_L \sim \Dcal$, we have
\[\frac{1}{|S_L|} \sum_{(x, y) \in S_L} \ell(g(x), y)  \leq \expect_{\Dcal}\left[ \ell(g(x), y)\right] + \frac{\epsilon}{4} \leq \expect_{\Dcal}[\ell(f^{\star}(x), y)] + \frac{3\epsilon}{4}. \]

Next, consider the predictor $\hat{g}$ returned by Algorithm \ref{alg:batch_real_agnos_equiv}. Since it is an empirical risk minimizer, its empirical risk can be at most the empirical risk of $g$. Given that the population risk of $\hat{g}$ can be at most $\epsilon/4$ away from its empirical risk, we have that
\[\expect_{\Dcal}\left[ \ell(\hat{g}(x), y)\right] \leq \expect_{\Dcal}[\ell(f^{\star}(x), y)] + \frac{3\epsilon}{4} + \frac{\epsilon}{4} \leq \inf_{f \in \Fcal}\expect_{\Dcal}[\ell(f(x), y)] + \epsilon, \]
where the second inequality above uses the definition of $f^{\star}$. Note that this inequality holds with probability at least $1- \delta$, where the probability is taken over both samples $S_U$ and $S_L$. Thus, we have shown that Algorithm \ref{alg:batch_real_agnos_equiv} is an agnostic PAC learner for $\Fcal$ with respect to $\ell$. 

We now upper bound the sample complexity of Algorithm \ref{alg:batch_real_agnos_equiv}, denoted $m(\epsilon, \delta, K)$ hereinafter. Note that $m_{\Acal}(\epsilon, \delta, K)$ is at most the number of unlabeled samples required for the realizable algorithm $\mathcal{A}$ to succeed plus the number of labeled samples for the ERM step to succeed. Thus,

\begin{equation*}
    \begin{split}
       m(\epsilon, \delta, K) &\leq m_{\Acal}\left(\frac{\epsilon}{2c}, \frac{\delta}{2}, K\right) +  \frac{8b^2}{\epsilon^2} \log{\frac{4|C(S_U)|}{\delta}}\\
       &\leq  m_{\Acal}\left(\frac{\epsilon}{2c}, \frac{\delta}{2}, K\right) + \frac{8b^2}{\epsilon^2} \left( m_{\Acal}\left(\frac{\epsilon}{2c}, \frac{\delta}{2}, K\right)\,  \log{(|\text{im}(\Fcal)|)}\, + \log{\frac{4}{\delta}} \right),
    \end{split}
\end{equation*}

where the second inequality follows due to $|C(S_U)| \leq |\text{im}(\Fcal)|^{|S_U|}$ and we need $|S_U|$ to be of size $m_{\Acal}\left(\frac{\epsilon}{2c}, \frac{\delta}{2}, K\right)$. 
\end{proof}

With Lemma \ref{realizable_agnostic_equiv} in hand, we can now relate the learnability of $\mathcal{H}$ with respect to any $\ell$ satisfying identity of indiscernibles to the learnability of $\mathcal{H}$ with respect to $\ell_H$. To that end, we prove a result establishing the equivalence of learnability between any two loss functions satisfying the identity of indiscernibles.  

 \begin{lemma}\label{lem:hamming_general_batch}
    Let $\ell$ and $\ell^{\prime}$ be any two loss functions satisfying the identity of indiscernibles. Then, a function class $\Fcal \subseteq\Ycal^{\Xcal}$ is agnostic PAC learnable with respect to $\ell$ if and only if $\Fcal\subseteq\Ycal^{\Xcal}$ is agnostic PAC learnable with respect to  $\ell^{\prime}$. 
\end{lemma}

 The key idea behind the proof of Lemma \ref{lem:hamming_general_batch} is to use a realizable learner for $\ell$ to construct a realizable learner for $\ell^{\prime}.$ This is possible because $\ell^{\prime}(y_1, y_2) =0$ if and only if $\ell(y_1, y_2) =0$ for any $y_1,y_2 \in \Ycal$. Given such realizable learner for $\ell^{\prime}$,  Lemma  \ref{realizable_agnostic_equiv}  guarantees the existence of an agnostic learner for $\ell^{\prime}$.

\begin{proof} Since $\ell$ and $\ell^{\prime}$ are arbitrary, it suffices to prove only one direction.  So, let us assume that $\Fcal$ is learnable with respect to $\ell$. We will now show that $\Fcal$ is learnable with respect to $\ell^{\prime}$ as well.  First, we show this for any realizable distribution $\Dcal$ with respect to $\ell^{\prime}$. Since, for any $y_1, y_2 \in \Ycal$, we have $\ell^{\prime}(y_1, y_2) =0$ if and only if $\ell(y_1, y_2) =0$, the distribution $\Dcal$ is also realizable with respect to $\ell$. Furthermore, as there are at most $2^{2K}$ distinct possible inputs to  $\ell^{\prime}(\cdot, \cdot)$, the loss function can only take a finite number of values. So, we can always find universal constants $a>0$ and $b>0$ (that only depends on $\ell$ and $\ell^{\prime}$) such that $a \ell \leq \ell^{\prime} \leq b\ell$. Given that $\Fcal$ is learnable with respect to $\ell$, there exists a learning algorithm $\Acal$ with the following property: for any $\epsilon, \delta > 0$, for $S \sim \Dcal^n$, such that $n = m_{\Acal}(\frac{\epsilon}{b}, \delta, K)$, the algorithm outputs a predictor $h = \Acal(S)$ such that, with probability $1-\delta$ over $S \sim \Dcal^n$, we have $\expect_{\Dcal}[\ell(h(x), y)] \leq \frac{\epsilon}{b}. $
This inequality upon using the fact that  $\ell^{\prime}(h(x), y) \leq b\ell(h(x), y)$ pointwise  reduces to $\expect_{\Dcal}[\ell^{\prime}(h(x), y)] \leq \epsilon. $ Therefore, any realizable learner $\Acal$ for $\ell$ is also a realizable learner for $\ell^{\prime}$. Finally, as $\ell^{\prime}$ satisfies the identity of indiscernibles and thus $c$-subadditivity, Lemma \ref{realizable_agnostic_equiv} guarantees the existence of agnostic PAC learner $\Bcal$ for $\Fcal$ with respect to $\ell^{\prime}$. In particular, one such agnostic PAC learner $\Bcal$ is Algorithm \ref{alg:batch_real_agnos_equiv} that has sample complexity
\[m_{\Bcal}(\epsilon, \delta, K) 
 \leq m_{\Acal}\left( \frac{\epsilon}{2bc} , \frac{\delta}{2}, K \right) + \frac{b^2}{\epsilon^2} \,\,O \left( m_{\Acal}\left(\frac{\epsilon}{2bc}, \frac{\delta}{2}, K\right) \, K + \log{\frac{1}{\delta}} \right),\]
where  $c > 1$ is the subadditivity constant of $\ell^{\prime}$ and $m_{\Acal}(\cdot,\cdot,K)$ is the sample complexity of any realizable algorithm $\Acal$. 

% \noindent\textbf{Part 2: Necessity.}
% The necessity proof is virtually identical to the sufficiency proof provided above, so we only sketch the argument here. Since $\Fcal$ is learnable with respect to $\ell$, we are guaranteed the existence of a realizable learner for $\ell$. Recall that $\ell$ and $\ell_{H}$ satisfying identity of indiscernibles imply that $\ell(y_1,y_2)=0$ if and only if $\ell_H(y_1, y_2)=0$. Thus, there must exist $a,b>0$ such that $a\ell_H \leq \ell \leq b \ell_H$, which subsequently implies that any realizable learner for $\ell$ is a realizable learner for $\ell_H$. Then, Lemma \ref{realizable_agnostic_equiv} implies that $\Fcal$ is agnostic learnable with respect to $\ell_H$ as one can do realizable to agnostic reduction using Algorithm \ref{alg:batch_real_agnos_equiv}. 
\end{proof}
 
\begin{comment}

\vinod{necessity proof feels very repetitive. I would just get rid of it, have only one single proof with both suff and necc. Keep the suff proof u have, and right a couple sentences at the end of it saying that nec direction is exactly the same and uses the fact blah blah and blah}
\end{comment}

An immediate consequence of Lemma \ref{lem:hamming_general_batch} is that learnability with respect to any loss $\ell$ satisfying the identity of indiscernibles is equivalent to learnability with respect to the Hamming loss $\ell_H$. Thus, given Theorem \ref{thm:multi_label_batch_hamming}, we can deduce the following result. 
\begin{theorem}
\label{thm:batchgenloss}
    Let $\ell$ be any multilabel loss function satisfying the identity of indiscernibles. A function class $\Fcal\subseteq\Ycal^{\Xcal}$ is agnostic PAC learnable with respect to $\ell$ if and only if each restriction $\Fcal_k \subseteq\Ycal_k^{\Xcal}$ is agnostic PAC learnable with respect to the $0\text{-}1$ loss. 
\end{theorem}

 \noindent \textbf{Remark.} Since the learnability of a binary function class with respect to the $0$-$1$ loss is characterized by its VC dimension \cite[Theorem 6.7]{ShwartzDavid}, Theorem \ref{thm:batchgenloss} implies that $\Fcal$ is learnable with respect to $\ell$ satisfying the identity of indiscernibles if and only if $\text{VC}(\Fcal_k) < \infty$ for each $k \in [K]$.

\subsection{Batch Multioutput Regression}
In this section, we consider the case when $\Ycal =[0,1]^K \subseteq\reals^K$ for  $K \in \naturals$. For bounded targets (with a known bound), this target space is without loss of generality because one can always normalize each $\Ycal_k$ to [0,1] by subtracting the lower bound and dividing by the upper bound of $\Ycal_k$. As usual, we consider an arbitrary multioutput function class $\Fcal \subseteq\Ycal^{\Xcal}$. Following our outline in classification, we first study the learnability of $\mathcal{F}$ under decomposable losses and then non-decomposable losses. 
\subsubsection{Characterizing Learnability for Decomposable Losses}
A canonical loss for the scalar-valued regression is the absolute value metric,  $d_1(f_k(x), y^k) := |f_k(x) -y^k|$. Analogously, we define $d_p(f_k(x), y^k) := |f_k(x) -y^k|^p$ for $p > 1$ are other natural scalar-valued losses. For multioutput regression, we consider decomposable losses that are natural multivariate extensions of the $d_1$ metric. In particular, we consider loss functions with the following properties.

\begin{assumption}\label{assumption1}
The loss can be written as $\ell(f(x), y)  = \sum_{k=1}^K \psi_k \circ \,d_1(f_k(x), y^k) $ where for each $k \in [K]$, the function $\psi_k: \reals_{\geq 0} \to \reals_{\geq 0}$ is $L$-Lipschitz and satisfies $\psi_k(0) =0$. 
\end{assumption}
Here, $\psi_k \circ d_1$ is a composition function defined as $ \psi_k \circ \,d_1(f_k(x), y^k) := \psi_k\big(d_1(f_k(x), y^k)\big)$. Note that $\psi_k\circ d_1$ is a large family of loss functions that effectively contains all natural decomposable multioutput regression losses. For instance, taking $\psi_k(z) = |z|^p$ for $p \geq 1$ gives $\ell_p$ norms raised to their $p$-th power.  Considering $\psi_k(z) =  |z|^2/2\, \indicator[|z| \leq \delta] + \delta (|z| - \delta/2) \indicator[|z| > \delta]$ for some $1>\delta > 0$ yields multivariate extension of Huber loss used for robust regression.  One may also construct a multioutput loss by considering different scalar-valued losses for each coordinate output. Next, we establish an equivalence between the learnability of $\Fcal \subseteq\Ycal^{\Xcal} $ with respect to $\ell$ and the learnability of each $\Fcal_k$ with respect to the loss $\psi_k \circ d_1$. 
\begin{theorem}\label{thm:batch_reg_decom}
     Let $\ell$ be any loss function that satisfies Assumption \ref{assumption1}. Then, a function class $\Fcal \subseteq\Ycal^{\Xcal}$ is agnostic learnable with respect to   $\ell$  if and only if each of $\Fcal_k \subseteq\Ycal_k^{\Xcal}$ is agnostic  learnable with respect to  $\psi_k \circ d_1$.
\end{theorem}

\begin{proof}  The proof of the sufficiency direction is similar to that of Theorem \ref{thm:multi_label_batch_hamming}, so we defer it to Appendix \ref{appdx:batch_reg_decom_suff}. We now focus on the necessity direction. To that end, we show that if $\Fcal$ is agnostic learnable with respect to $\ell$, then each $\Fcal_k$ is agnostic learnable with respect to $\psi_k \circ d_1$. In particular, given oracle access to agnostic learner $\Acal$ for $\Fcal$, we construct agnostic learner $\Acal_1$ for  $\Fcal_1$. By symmetry, a similar reduction can then be used to construct an agnostic learner for each component $\Fcal_k$. 

Since we are given a sample with a single, univariate target, the main problem is to find the right way to augment  samples to a $K$-variate target.  In the proof of Theorem \ref{thm:multi_label_batch_hamming}, we showed that randomly choosing $y_{ik} \sim \text{Uniform}(\{-1,1\})$ for $k \geq 2$ results in all predictors having a constant $1/2$ risk--leaving only the risk of the first component on both sides. Unfortunately, in regression under general losses, no single augmentation works for every distribution on $\Xcal$.  Thus, we augment the samples by considering all possible behaviors of $(\Fcal_2, \ldots, \Fcal_K)$ on the sample. Since the function class maps to a potentially uncountably infinite space, we first discretize each component of the function class and consider all possible labelings over the discretized space. Fix $1>\alpha > 0$. For $k \geq 2$, define the discretization
\begin{equation}\label{eq:discretization}
    f_k^{\alpha}(x) =  \floor*{\frac{f(x)}{\alpha}}\alpha
\end{equation}
for every $f_k \in \Fcal_k$ and the discretized component class $\Fcal_k^{\alpha} = \{f^{\alpha}_k | f_k \in \Fcal_k\}$. 
Note that a function in $\Fcal_k$ maps to $\{0,  \alpha, 2\alpha, \ldots,  \floor{1/\alpha}\alpha \}$ and the size of the range of the discretized function class $\Fcal_k^{\alpha}$ is $1+ \floor{1/\alpha} \leq (\alpha+1)/\alpha \leq 2/\alpha$. For the convenience of exposition, let us define $\Fcal_{2:K}^{\alpha}$ to be $\Fcal^{\alpha}$ without the first component, and we denote $f_{2:K}^{\alpha}$ to be an element of $\Fcal_{2:K}^{\alpha}$. 
\begin{algorithm}
\caption{Agnostic  learner for $\Fcal_1$ with respect to $\psi_1 \circ d_1$}
\label{alg:reg_ness}
\setcounter{AlgoLine}{0}
\KwIn{Agnostic  learner $\Acal$ for $\Fcal$, samples $S = (x_{1:n}, y_{1:n}^1) \sim \Dcal_1^n$ , and  another independent samples $\widetilde{S}$ from $\Dcal_1$}

Define $S_{\text{aug}} = \{(x_{1:n}, y_{1:n}^1, f_{2:K}^{\alpha}(x_{1:n}) \mid f_{2:K}^{\alpha} \in \Fcal_{2:K}^{\alpha}\}$, all possible augmentations of $S$ by $\Fcal_{2:K}^{\alpha}$.

Run $\Acal$ over all possible augmentations to get $C(S) := \left\{\Acal\big(S_a \big) 
 \mid S_a \in S_{\text{aug}}\right\}.$

Define  $C_1(S) = \{g_1 \mid (g_1, \ldots, g_k) \in C(S)\}$, a restriction of $C(S)$ to its first coordinate output. 

Return $\hat{g}_1$, the predictor in $C_1(S)$ with the lowest empirical error over $\widetilde{S}$ with respect to $\psi_1 \circ d_1$. 
\end{algorithm}

\noindent We now show that Algorithm \ref{alg:reg_ness} is an agnostic  learner for $\Fcal_1$. First, let us define 
\[f_1^{\star} := \argmin_{f_1 \in \Fcal_1} \expect_{\Dcal_1} [\psi_1 \circ d_1(f_1(x), y^1)],\]
to be optimal predictor in $\Fcal_1$ with respect to $\Dcal_1$. By definition of $\Fcal_1$, there must exist $f_{2:K}^{\star} \in \Fcal_{2:K}$ such that $(f_1^{\star}, f_{2:K}^{\star}) \in \Fcal$. We note that $f_k^{\star}$ need not be optimal predictors in $\Fcal_k$ for $k \geq 2$, but we use the $\star$ notation just to associate these component functions with the first component function $f_1^{\star}$. Define $f_{2:K}^{\star, \alpha} \in \Fcal_{2:K}^{\alpha}$ to be the corresponding discretization of $f_{2:K}^{\star}$. At a high level, the key idea of this proof is to show that the algorithm $\Acal$ when run on the sample $(x_{1:n}, y_{1:n}^1, f_{2:K}^{\star, \alpha}(x_{1:n})$ produces a predictor $g= \Acal(x_{1:n}, y_{1:n}^1, f_{2:K}^{\star, \alpha}(x_{1:n}))$ such that its restriction $g_1$ is a valid agnostic learner for $\Fcal_1$.  Although one such augmentation is enough to produce an agnostic learner for $\Fcal_1$, all possible augmentations are required in step 2 of Algorithm \ref{alg:reg_ness} because $f_1^{\star}$ is not known to the learner apriori.  We now make this argument precise. 

Suppose $g = \Acal((x_{1:n}, y_{1:n}^1, f_{2:K}^{\star, \alpha}(x_{1:n}))$ is the predictor obtained by running $\Acal$ on the sample augmented by $f_{2:K}^{\star, \alpha}$. Note that $g \in C(S)$ by definition. Let $m_{\Acal}(\epsilon, \delta, K)$ be the sample complexity of $\Acal$. Since $\Acal$ is an agnostic  learner for $\Fcal$ with respect to $\ell$, we have that for $n \geq m_{\Acal}(\epsilon/4, \delta/2, K)$, with probability at least $1- \delta/2$,
\begin{equation*}
    \begin{split}
        \expect_{\Dcal_1}\Big[\psi_1 \circ d_1(&g_1(x), y^1)\Big] + \sum_{k=2}^K\expect_{\Dcal_{\Xcal}}\left[\psi_k \circ d_1(g_k(x), f_k^{\star, \alpha}(x)) \right] \\
        &\leq \inf_{f \in \Fcal} \left(\expect_{\Dcal_1}\left[\psi_1 \circ d_1(f_1(x), y^1)\right] + \sum_{k=2}^K\expect_{\Dcal_{\Xcal}}\left[\psi_k \circ d_1(f_k(x), f_k^{\star, \alpha}(x)) \right] \right)  + \frac{\epsilon}{4}
    \end{split}
\end{equation*}

\noindent Note that the quantity on the left is trivially lower bounded by the risk of the first component. To handle the right-hand side, we first note that the optimal risk  is trivially upper bounded by the risk of $(f_1^{\star}, f_{2:K}^{\star})$, yielding
\begin{equation*}
    \begin{split}
        \expect_{\Dcal_1}\left[\psi_1 \circ d_1(g_1(x), y^1)\right] 
        \leq \expect_{\Dcal_1}\left[\psi_1 \circ d_1(f_1^{\star}(x), y^1)\right] + \sum_{k=2}^K\expect_{\Dcal_{\Xcal}}\left[\psi_k \circ d_1(f_k^{\star}(x), f_k^{\star, \alpha}(x)) \right]   + \frac{\epsilon}{4}.
    \end{split}
\end{equation*}
Next, using the $L$-Lipschitzness of $\psi_k$ and the fact that $\psi_k(0)=0$ implies $\psi_k \circ d_1(f_k^{\star}(x), f_k^{\star, \alpha}(x)) \leq L\, d_1(f_k^{\star}(x), f_k^{\star, \alpha}(x)) \leq L \alpha$ for all $k \geq 2$.  
Thus, picking $\alpha = \frac{\epsilon}{4LK}$ and using the definition of $f_1^{\star}$, we obtain 
\[ \expect_{\Dcal_1}\left[\psi_1 \circ d_1(g_1(x), y^1)\right]\leq \inf_{f_1 \in \Fcal_1}\expect_{\Dcal_1}[\psi_1 \circ d_1(f_1(x), y^1)] + \frac{\epsilon}{2}.\]

 Therefore, we have shown the existence of one predictor $g \in C(S)$ such that its restriction to the first component, $g_1$, obtains the agnostic  bound. Note that since $\psi_1$ is $L$-Lipschitz  and satisfies $\psi_1(0) =0$, we obtain that $\psi_1 \circ d_1 (\cdot, \cdot) \leq L$. The upper bound also uses the fact that $|f_1(x) -y^1| \leq 1$. Now  recall that by Hoeffding's Inequality and union bound, with probability at least $1-\delta/2$, the empirical risk of every hypothesis in $C_1(S)$ on a sample of size $\geq  \frac{8 L^2}{\epsilon^2} \log{\frac{4 |C_1(S)|}{\delta}} $ is at most $\epsilon/4$ away from its true error. So, if $|\widetilde{S}| \geq  \frac{8L^2}{\epsilon^2} \log{\frac{ 4|C_1(S)|}{\delta}} $, then with probability at least $1-\delta/2$, the empirical risk of the predictor $g_1$ is
 
\[\frac{1}{|\widetilde{S}|} \sum_{(x, y^1) \in \widetilde{S}} \psi_1 \circ d_1(g_1(x), y^1)  \leq  \expect_{\Dcal_1}\left[\psi_1 \circ d_1(g_1(x), y^1)\right] + \frac{\epsilon}{4} \leq \inf_{f_1 \in \Fcal_1}\expect_{\Dcal_1}[\psi_1 \circ d_1(f_1(x), y^1)] + \frac{3\epsilon}{4}. \]

Since $\hat{g}_1$, the output of Algorithm \ref{alg:reg_ness} is the ERM on $\widetilde{S}$ over $C_1(S)$, its empirical risk can be at most the empirical risk of $g_1$, which is at most $\inf_{f_1 \in \Fcal_1}\expect_{\Dcal_1}[\psi_1 \circ d_1(f_1(x), y^1)] + \frac{3\epsilon}{4}$. Given that the population risk of $\hat{g}_1$ is at most $\epsilon/4$ away from its empirical risk, we can conclude that the population risk of $\hat{g}_1$ is 
\[\expect_{\Dcal_1}[\psi_1 \circ d_1(\hat{g}_1(x), y^1)]  \leq \inf_{f_1 \in \Fcal_1}\expect_{\Dcal_1}[\psi_1 \circ d_1(f_1(x), y^1)] + \epsilon. \]
Applying union bounds, the entire process, algorithm $\Acal$ on the dataset augmented by $f_{2:K}^{\star, \alpha}$ and ERM in step 4, succeeds with probability $1- \delta$. The sample complexity of Algorithm \ref{alg:reg_ness} is the sample complexity of Algorithm $\Acal$ and the sample complexity of ERM in step 4, which is   
\begin{equation*}
    \begin{split}
       &\leq m_{\Acal}(\epsilon/4, \delta/2, K) + \frac{8L^2}{\epsilon^2} \log{\frac{ 4|C_1(S)|}{\delta}}\\
       &\leq  m_{\Acal}(\epsilon/4, \delta/2, K) + \frac{8KL^2}{\epsilon^2}\left(  m_{\Acal}(\epsilon/4, \delta/2, K)\,  \log\left({\frac{4K\,L}{\epsilon}}\right)+ \log{\frac{4}{\delta}}\right),
    \end{split}
\end{equation*}
where the second inequality follows due to $|C_1(S)| \leq (2/\alpha)^{m_{\Acal}(\epsilon/4, \delta/2, K)\, K} $ is the required size of $C_1(S)$ and  our choice of $\alpha = \epsilon/(4KL)$. This completes the proof as we have shown that Algorithm \ref{alg:reg_ness} is an agnostic  learner for $\mathcal{F}_1$ with respect to $\psi_1 \circ d_1$.
\end{proof}

\subsubsection{A More General Characterization of Learnability for Decomposable Losses }
Unlike Theorems \ref{thm:multi_label_batch_hamming} and \ref{thm:batchgenloss} in classification setting where we connected the learnability of $\Fcal$ to the learnability of $\Fcal_k $'s with respect to $0$-$1$ loss, Theorem \ref{thm:batch_reg_decom} relates the learnability of $\Fcal$ to the learnability of $\Fcal_k$ with respect to $\psi_k \circ d_1$ instead of the more canonical loss $d_1$. In this section, we complete that final step to characterizing learnability in terms of $d_1$ with an additional assumption on $\psi_k$.

\begin{assumption}\label{assumption2}
For all $k \in [K]$, the function $\psi_k: \reals_{\geq 0} \to \reals_{\geq 0}$ is monotonic. 
\end{assumption}
\noindent Under these assumptions, Theorem \ref{thm:extended_batch} provides a more general characterization than Theorem \ref{thm:batch_reg_decom}. 
\begin{theorem}\label{thm:extended_batch}
Let $\ell$ be any loss function that satisfies Assumptions \ref{assumption1} and \ref{assumption2}. Then, a multioutput function class $\Fcal \subseteq ([0, 1]^K)^{\Xcal}$ is agnostic  learnable with respect to  $\ell$ if and only if each $\Fcal_k \subseteq [0, 1]^{\Xcal}$ is agnostic learnable with respect to $d_1$.
\end{theorem}
 Since the fat-shattering dimension of a real-valued function class characterizes its learnability with respect to $d_1$ loss, Theorem \ref{thm:extended_batch} implies that a multioutput function class $\Fcal$ is learnable with respect to $\ell$ satisfying Assumptions \ref{assumption1} and \ref{assumption2} if and only if $\text{fat}_{\gamma }(\Fcal) < \infty$ for every fixed scale $\gamma > 0$.   Theorem \ref{thm:extended_batch} is an immediate consequence of Theorem \ref{thm:batch_reg_decom} and the following lemma, the proof of which is deferred to Appendix \ref{appdx:equiv_d1_psi_d1}.
\begin{lemma}\label{lem:batch_NFLT}
      Let $\Ycal = [0,1]$ be the label space and $\psi: \mathbb{R}_{\geq 0} \rightarrow \mathbb{R}_{\geq 0}$ be any Lipschitz and monotonic function that satisfies $\psi(0) =0$. A scalar-valued function class $\Gcal \subseteq [0,1]^{\Xcal}$ is agnostic learnable with respect to $\psi \circ d_1$ if and only if $\mathcal{G}$ is agnostic  learnable with respect to $d_1$.  
\end{lemma}

The part of the lemma showing that $d_1$ learnability implies $\psi \circ d_1$ is trivial using the Rademacher-based argument and Talagrand's contraction lemma. However, proving $\psi \circ d_1$ learnability implies $d_1$ learnability is non-trivial. The case $\psi(z) = |z|^2$ is considered  in \cite[Theorem 19.5]{anthony_bartlett_1999}, but their proof  requires a mismatch between the label space and the prediction space, namely $\Ycal = [-1,2]$ but $\Gcal \subseteq [0,1]^{\Xcal}$. In this work, we improve their result by showing the equivalence between $\psi \circ d_1$ learnability and $d_1$ learnability without requiring extended label space. 

\begin{comment}
   \vinod{I don't think you need to have a whole new proof for the sufficiency section. What I did in online was say that the sufficiency is immediate consequence of Talagrands contraction lemma and the fact that psi is Lipschtiz. Then I bundled this proof into necessity. } 
\end{comment}

An application of Lemma \ref{lem:batch_NFLT} is that the learnability of a real-valued function class $\Gcal$ with respect to losses $d_1$ and $d_p$ are equivalent for any $p>1$.
\begin{comment}
   \vinod{this statement is technically not true. Lemma only guarantees this undder those assumptions. Also should prolly use d1 and dp instead $| |$ and $| |^p$ for consistency} 
\end{comment}

\subsubsection{Characterizing Learnability for Non-Decomposable Losses}

Next, we study the learnability of function class $\Fcal$ with respect to non-decomposable losses. In the regression setting, the  natural non-decomposable loss to consider is $\ell_{p}$ norm, which is defined as $\ell_p(f(x), y):= \left(\sum_{k=1} |f_k(x) -y^k|^p \right)^{1/p}$ for $1 \leq p < \infty$. For $p= \infty$, the $p$-norm is defined as $\ell_{\infty}(f(x), y) := \max_{k} |f_k(x) -y^k|$. 
One might be interested in $\ell_p$ norms instead of their decomposable counterparts $\ell_p^p$ losses discussed in the previous section mainly for robustness to outliers. 
The following result characterizes the agnostic  learnability of $\Fcal$ with respect to $\ell_{p}$ norms.

\begin{theorem}\label{thm:batch_reg_nondecom}
      Fix $p\geq 1$. The function class $\Fcal \subseteq\Ycal^{\Xcal}$ is agnostic learnable with respect to  $\ell_{p}$ norm if and only if each of $\Fcal_k \subseteq\Ycal_k^{\Xcal}$ is agnostic  learnable with respect to the absolute value loss, $d_1$.
\end{theorem}

Using $\psi_k(z) = |z|$ in Theorem \ref{thm:batch_reg_decom} implies that $\Fcal$ is learnable with respect to $\ell_1$ if and only if each $\Fcal_k$ is learnable with respect to $d_1$. Thus, to prove Theorem \ref{thm:batch_reg_nondecom}, it suffices to show that for all $p > 1$,  $\Fcal$ is learnable with respect to $\ell_p$ if and only if $\Fcal$ is learnable with respect to $\ell_1$ norm.

\begin{proof}  We begin by proving the sufficiency direction. As discussed above, the learnability of each $\Fcal_k$ with respect to $d_1$ implies that $\Fcal$ is learnable with respect to $\ell_1$. Then, the high-level idea of the proof is to use an agnostic learner for $\Fcal$ with respect to $\ell_1$ to construct a realizable learner for $\Fcal^{\alpha}$ with respect to $\ell_1$. Using the fact $\ell_1(y_1,y_2)=0$ if and only if $\ell_p(y_1, y_2)=0$ for any $y_1, y_2 \in \Ycal$, the realizable learner for $\ell_1$ is also a realizable learner for $\ell_p$. Finally, as $|\text{im}(\Fcal^{\alpha})|<\infty$ for every $\alpha>0$, Lemma \ref{realizable_agnostic_equiv} guarantees the existence of an agnostic learner for $\Fcal^{\alpha}$ with respect to $\ell_p$. Then, a simple application of the triangle inequality shows that an agnostic learner for $\Fcal^{\alpha}$ is also an agnostic learner for $\Fcal$ with respect to $\ell_p$. We now proceed with the formal proof.

\noindent \textbf{Part 1: Sufficiency.}
Fix $p > 1$. Recall that the learnability of each $\Fcal_k$ with respect to $d_1$ implies that $\Fcal$ is learnable with respect to $\ell_1$. Let $\Dcal$ be arbitrary distribution on $\Xcal \times \Ycal$ and $\Acal$ be an agnostic PAC learner for $\Fcal$ with respect to $\ell_1$ with sample complexity $m(\epsilon, \delta)$ . For any $\epsilon, \delta > 0$, with $n$ sufficiently larger than  $m(\epsilon/2, \delta, K)$, with probability at least $1-\delta$ over $S \sim \Dcal^n$, we have
\[\expect_{\Dcal}[\ell_1(g(x), y)] \leq \inf_{f \in \Fcal}\expect_{\Dcal}[\ell_1(f(x), y)] + \frac{\epsilon}{2}, \]
where $g =\Acal(S)$. Define $\Fcal^{\alpha}$ to be a discretized function class obtained by discretizing $\Fcal$ component-wise using the scheme \eqref{eq:discretization}.
Consider $\Dcal$ to be a realizable distribution with respect to $\Fcal^{\alpha}$. Note that the triangle inequality implies $\ell_1(f(x), y) \leq \ell_1(f^{\alpha}(x), y) + \ell_1(f(x), f^{\alpha}(x)) \leq \ell_1(f^{\alpha}(x), y) + K \alpha$. Taking $\alpha = \frac{\epsilon}{2K}$ and using the fact that $\inf_{f \in \Fcal} \expect[\ell_1(f^{\alpha}(x), y)] = 0$, we obtain $\expect_{\Dcal}[\ell_1(g(x), y)] \leq \epsilon.$ Next, using the inequality $\ell_p(g(x), y) \leq \ell_1(g(x), y)$ pointwise yields
\[\expect_{\Dcal}[\ell_p(g(x), y)] \leq \epsilon.\]
Therefore, $\Acal$ is a realizable learner for $\Fcal^{\alpha}$ with respect to $\ell_p$ with sample complexity $m(\epsilon/2, \delta,K)$. Since $|\text{im}(\Fcal^{\alpha})| < \infty$ and $\ell_p(\cdot, \cdot)$ are $1$-subadditive, Lemma \ref{realizable_agnostic_equiv} implies that $\Fcal^{\alpha}$ is agnostic learnable with respect to $\ell_p$ via Algorithm \ref{alg:batch_real_agnos_equiv}, referred to as algorithm $\Bcal$ henceforth. Thus, for any $\epsilon, \delta > 0$, there exists a $n \geq m_{\Bcal}(\epsilon/2, \delta/2)$,  for any distribution $\widetilde{\Dcal}$ on $\Xcal \times \Ycal$, running $\Bcal$ on $S \sim \Dcal^n$ outputs a predictor $\tilde{g} \in \Ycal^{\Xcal}$ such that with probability at least $1-\delta$ over $S \sim \widetilde{\Dcal}^n$, we have
\[\expect_{\widetilde{\Dcal}}[\ell_p(\tilde{g}(x), y)] \leq \inf_{f^{\alpha} \in \Fcal^{\alpha}} \expect_{\widetilde{\Dcal}}[\ell_p(f^{\alpha}(x), y)] + \frac{\epsilon}{2} = \inf_{f \in \Fcal} \expect_{\widetilde{\Dcal}}[\ell_p(f^{\alpha}(x), y)] + \frac{\epsilon}{2}. \]
Using triangle inequality, we have $\ell_p(f^{\alpha}(x), y) \leq \ell_{p}(f(x), y) + \ell_p(f^{\alpha}(x), f(x))  \leq \ell_{p}(f(x), y) + \alpha K $ pointwise. Again taking $\alpha = \frac{\epsilon}{2K}$, we obtain
\[\expect_{\widetilde{\Dcal}}[\ell_p(\tilde{g}(x), y)] \leq \inf_{f \in \Fcal} \expect_{\widetilde{\Dcal}}[\ell_p(f(x), y)] +\epsilon.\]
Therefore, we have shown that $\Fcal$ is agnostic PAC learnable with respect to $\ell_p$.

The sample complexity of agnostic learner $\Bcal$  can be made  precise using the sample complexity of Algorithm \ref{alg:batch_real_agnos_equiv}. In particular, the sample complexity of $\Bcal$ is the sample complexity of the realizable learner $\Acal$ and the sample complexity of the ERM in step 2 of Algorithm \ref{alg:batch_real_agnos_equiv}. Proof of Lemma \ref{realizable_agnostic_equiv} shows that the sample complexity of $\Bcal$ must be
\begin{equation*}
    \begin{split}
       m_{\Bcal}(\epsilon, \delta,K) 
       &\leq  m_{\Acal}\left(\frac{\epsilon}{2c}, \frac{\delta}{2},K\right) + \frac{8b^2}{\epsilon^2} \left( m_{\Acal}\left(\frac{\epsilon}{2c}, \frac{\delta}{2}, K\right)\,  \log{|(\text{im}(\Fcal^{\alpha})|)}\, + \log{\frac{4}{\delta}} \right),
    \end{split}
\end{equation*}
where $b$ is the upperbound on the loss and $c$ is the subadditivity constant. Since $b \leq K$, and $c=1$ for all $\ell_p$ norms with $p \geq 1$, we obtain
\[ m_{\Bcal}(\epsilon, \delta, K) 
       \leq  m_{\Acal}(\epsilon/2, \delta/2, K) + \frac{8K^3}{\epsilon^2} \left( m_{\Acal}(\epsilon/2, \delta/2, K)\,\log {\left( \frac{4K}{\epsilon} \right)}\, + \log{\frac{4}{\delta}} \right),\]
where we also use the fact that $|\text{im}(\Fcal^{\alpha})| \leq (2/\alpha)^{K}$ for our choice of $\alpha = \frac{\epsilon}{2K}$. 

\noindent\textbf{Part 2: Necessity}. Fix $p > 1$.
 We now prove that  $\Fcal$ being learnable with respect to $\ell_p$ implies $\Fcal$ is learnable with respect to $\ell_1$. The proof is identical to the proof of sufficiency, so we only provide a sketch of the argument here. Our proof strategy follows a similar route through realizable learnability of the discretized class $\Fcal^{\alpha}$ and then the use of Lemma \ref{realizable_agnostic_equiv}.

Recall that any agnostic learner $\Acal$ for $\Fcal$ with respect to $\ell_p$ is a realizable learner for $\Fcal^{\alpha}$ with respect to $\ell_p$. Using the inequality $\ell_1(\cdot, \cdot) \leq K \ell_p(\cdot, \cdot) $ pointwise, we can deduce that $\Acal$ is also a realizable learner for $\Fcal^{\alpha}$ with respect to $\ell_1$. Since $|\text{im}(\Fcal^{\alpha})| \leq (2/\alpha)^{K} < \infty$,  Lemma \ref{realizable_agnostic_equiv} guarantees existence of an agnostic learner for $\Fcal^{\alpha}$ with respect to $\ell_1$. Using triangle inequality, we obtain $\ell_1(f^{\alpha}(x), y) \leq \ell_{1}(f(x), y) + \ell_1(f^{\alpha}(x), f(x))  \leq \ell_{1}(f(x), y) + \alpha K $ pointwise, and choosing appropriate discretization scale allows us to turn agnostic bound for $\Fcal^{\alpha}$ into an agnostic bound for $\Fcal$. 
\end{proof}

\begin{comment}
  \vinod{nec proof feels exactly same as suff, can you axe it and put a comment in the suff direction saying that nec direction is exactly the same the uses the fact that [blah blah blah]?}  
\end{comment}

 \noindent \textbf{Remark.} We note that Theorem $\ref{thm:batch_reg_nondecom}$ holds for any norm on $\reals^K$, but we only focus on $\ell_p$ norms here due to their practical significance. As the fat-shattering dimension of a real-valued function class characterizes its learnability with respect to $d_1$ loss \citep{BARTLETT1996}, Theorem \ref{thm:batch_reg_nondecom} implies that a multioutput function class $\Fcal$ is learnable with respect to \ $\ell_p$ for $ 1\leq p \leq \infty$ if and only if $\text{fat}_{\gamma }(\Fcal_k) < \infty$ for all $k \in [K]$ at every fixed scale $\gamma > 0$. 

Since we are only concerned with the question of learnability in this work, our focus is not on optimal sample complexity rates. However, we point out that for any $p \geq 1$, if each $\Fcal_k$ is learnable with respect to $d_1$, then $\Fcal$ is learnable with respect to $\ell_p$ via ERM with a better sample  complexity than Algorithm \ref{alg:batch_real_agnos_equiv}. The proof of this claim is based on Rademacher complexity and is provided in Appendix \ref{appdx:rademacher_batch}. 

\section{Online Multioutput Learnability}
Here, we study the online learnability of multioutput function classes. Throughout this section, we give regret bounds assuming an \textit{oblivious} adversary. A standard reduction (Chapter 4 in \cite{cesa2006prediction}) allows us to convert oblivious regret bounds to adaptive regret bounds in the full-information setting. A key requirement allowing an oblivious regret bound to generalize to an adaptive regret bound is that the learner’s predictions on round $t$ should not depend on any of its past predictions from previous rounds. This is true for all of the online learning algorithms in this section.

\subsection{Online Agnostic-to-Realizable Reduction}

Our strategy for constructively characterizing the  learnability of general losses in both the batch classification and regression setting required the ability to convert a realizable learner to an agnostic learner in a black-box fashion. In this section, we provide an analog of this conversion for the \textit{online} setting. More specifically, we focus on the setting where $\mathcal{F} \subseteq\mathcal{Y}^{\mathcal{X}}$ is a multioutput function class but $|\text{im}(\mathcal{F})| < \infty$ is finite. Then, for any $c$-subadditive loss function $\ell$, we constructively convert a (potentially randomized) realizable online learner for $\mathcal{F}$ with respect to $\ell$ into an agnostic online learner for $\mathcal{F}$ with respect to $\ell$. Theorem \ref{thm:onlreal2agn} formalizes the main result of this subsection. 

\begin{theorem}
\label{thm:onlreal2agn}
Let $\mathcal{F} \subseteq\mathcal{Y}^{\mathcal{X}}$ be a multioutput function class such that $|\text{im}(\mathcal{F})| < \infty$ and $\ell: \mathcal{Y} \times \mathcal{Y} \rightarrow \mathbb{R}_{\geq 0}$ be any $c$-subadditive loss function such that $\ell(\cdot, \cdot) \leq M$. If $\mathcal{A}$ is a realizable online learner for $\mathcal{F}$ with respect to $\ell$ with sub-linear expected regret $R(T, |\text{im}(\mathcal{F})|)$, then for every $\beta \in (0, 1)$, there exists an agnostic online learner for $\mathcal{F}$ with respect to $\ell$ with expected regret 
$$\frac{cT}{T^{\beta}}\,\overline{R}(T^{\beta},  |\text{im}(\mathcal{F})|) + M\sqrt{2T^{1+\beta}\ln(|\text{im}(\mathcal{F})|)},$$
where $\overline{R}(T,  |\text{im}(\mathcal{F})|)$ is any concave, sublinear upperbound on $R(T, |\text{im}(\mathcal{F})|)$.
\end{theorem}

Note that if $\overline{R}(T^{\beta}, |\text{im}(\mathcal{F})|)$ is sublinear in its first argument, then  $\frac{cT}{T^{\beta}}\overline{R}(T^{\beta}, |\text{im}(\mathcal{F})|) + M\sqrt{2T^{1+\beta}\ln(|\text{im}(\mathcal{F})|)}$ is sublinear in $T$ for any $\beta \in (0, 1)$. By Lemma \ref{lem:woess}, we are guaranteed the existence of $\overline{R}(T,  |\text{im}(\mathcal{F})|)$. 

\begin{lemma}\emph{\cite[Lemma 5.17]{woess2017groups}}\label{lem:woess}
    Let $g$ be a positive sublinear function. Then, $g$ is bounded from above by a concave sublinear function. 
\end{lemma}

Therefore, Theorem \ref{thm:onlreal2agn} and Lemma \ref{lem:woess} show that for any function class $\mathcal{F}$ with finite image space and any $c$-subadditive loss function, realizable and agnostic online learnability are equivalent. We now begin the proof of Theorem \ref{thm:onlreal2agn}.  

\begin{proof}  Let $\mathcal{A}$ be a (potentially randomized) online realizable learner for $\mathcal{F}$ with respect to $\ell$.  By definition, this means that for any  (realizable) sequence $(x_1, f(x_1)), ..., (x_T, f(x_T))$ labeled by a function $f \in \mathcal{F}$, we have 
    $$\mathbb{E}\left[\sum_{t=1}^T \ell(\mathcal{A}(x_t), f(x_t))\right] \leq R(T, |\text{im}(\mathcal{F})|),$$
    where $R(T, |\text{im}(\mathcal{F})|)$ is a sub-linear function of $T$. We now use $\mathcal{A}$ to construct an agnostic online learner $\mathcal{Q}$ for $\mathcal{F}$ with respect to $\ell$. Since we are assuming an oblivious adversary, let $(x_1, y_1), ..., (x_T, y_T) \in (\mathcal{X} \times \mathcal{Y})^T$ denote the stream of points to be observed by the online learner and $f^{\star} = \argmin_{f \in \mathcal{F}} \sum_{t=1}^T \ell(f(x_t), y_t)$ to be the optimal function in hindsight.   
    
    Our high-level strategy is to construct a large set of Experts that approximately cover all possible labelings of the instances $x_1, ..., x_T$ by functions in $\mathcal{F}$. In particular, each Expert uses an independent copy of $\mathcal{A}$ to make predictions, but update $\mathcal{A}$ using \textit{different} sequences of labeled instances. Together, our set of Experts update $\mathcal{A}$ using all possible sequences of labeled instances. In order to ensure that the number of Experts is not too large, we construct such a set of Experts over a sufficiently small \textit{sub-sample} of the stream. Finally, we run the celebrated Randomized Exponential Weights Algorithm (REWA) \citep{cesa2006prediction} using our set of experts and the scaled loss function $\frac{\ell}{M}$ over the original stream of points $(x_1, y_1), ... , (x_T, y_T).$ We now formalize this idea below. 

    For any bitstring $b \in \{0, 1\}^T$, let $\phi: \{t: b_t = 1\} \rightarrow \text{im}(\mathcal{F})$ denote a function mapping time points where $b_t = 1$ to elements in the image space $\text{im}(\mathcal{F})$. Let $\Phi_b \subseteq (\text{im}(\mathcal{F}))^{\{t: b_t = 1\}}$ denote all such functions $\phi$. For every $f \in \mathcal{F}$, let $\phi_b^f \in \Phi_b$ be the mapping such that for all $t \in \{t: b_t = 1\}$, $\phi_b^f(t) = f(x_t)$. Let $|b| = |\{t: b_t = 1\}|$. For every $b \in \{0, 1\}^T$ and $\phi \in \Phi_b$, define an Expert $E_{b, \phi}$. Expert $E_{b, \phi}$, formally presented in Algorithm \ref{alg:expert_olreal2agn}, uses $\mathcal{A}$ to make predictions in each round. However, $E_{b, \phi}$ only updates $\mathcal{A}$ on those rounds where $b_t = 1$, using $\phi$ to compute a labeled instance $(x_t, \phi(t))$. For every $b \in \{0, 1\}^T$, let $\mathcal{E}_b = \bigcup_{\phi \in \Phi_b} \{E_{b, \phi}\}$ denote the set of all Experts parameterized by functions $\phi \in \Phi_b$. If $b$ is the all zeros bitstring, then $\mathcal{E}_b$ is empty. Therefore, we actually define $\mathcal{E}_b = \{E_0\} \cup \bigcup_{\phi \in \Phi_b} \{E_{b, \phi}\}$, where $E_0$ is the expert that never updates $\mathcal{A}$ and plays $\hat{y}_t = \mathcal{A}(x_t)$ for all $t \in [T]$.  Note that $1 \leq |\mathcal{E}_b| \leq (|\text{im}(\mathcal{F})|)^{|b|}$.

    \begin{algorithm}
    \caption{Expert($b$, $\phi$)}
    \label{alg:expert_olreal2agn}
    \setcounter{AlgoLine}{0}
    \KwIn{Independent copy of realizable online learner $\mathcal{A}$ for $\mathcal{F}$ with respect to $\ell$}
    \For{$t = 1,...,T$} {
        Receive example $x_t$
    
        Predict $\hat{y}_t = \mathcal{A}(x_t)$
        
        \uIf{$b_t = 1$}{        
            Update $\mathcal{A}$ by passing $(x_t, \phi(t))$
        }
    }
    \end{algorithm}

\begin{algorithm}
\caption{Agnostic online learner $\mathcal{Q}$ for $\Fcal$ with respect to $\ell$}
\label{alg:onlreal2agn}
\setcounter{AlgoLine}{0}
\KwIn{ Parameter $0 < \beta < 1$}

Let $B \in \{0, 1\}^T$ such that  $B_t \overset{\text{iid}}{\sim} \text{Bernoulli}(\frac{T^{\beta}}{T})$ 

% Let $\Phi_B = (\mathcal{Y}_{2:K}^{\alpha})^{\{t: B_t = 1\}}$ denote the set of all possible functions mapping from  ${\{t: B_t = 1\}}$ to $\mathcal{Y}_{2:K}^{\alpha}$.

Construct the set of experts $\mathcal{E}_B = \{E_0\} \cup \bigcup_{\phi \in \Phi_B} \{E_{B, \phi}\}$ according to Algorithm \ref{alg:expert_olreal2agn}

Run REWA $\mathcal{P}$ using $\mathcal{E}_B$ and the loss function $\frac{\ell}{M}$ over the stream $(x_1, y_1), ..., (x_T, y_T)$

\end{algorithm}

With this notation in hand, we are now ready to present Algorithm \ref{alg:onlreal2agn}, our main agnostic online learner $\mathcal{Q}$ for $\mathcal{F}$ with respect to $\ell$. Our goal is to show that $\mathcal{Q}$ enjoys sublinear expected regret. There are three main sources of randomness: the randomness involved in sampling $B$, the internal randomness of $\mathcal{A}$, and the internal randomness of REWA. Let $B, A$ and $P$ denote the random variables associated with each source of randomness respectively. By construction, $B, A$, and $P$ are independent.

Using Theorem 21.11 in \cite{ShwartzDavid} and the fact that $B, A$ and $P$ are independent, REWA guarantees almost surely that 
$$\sum_{t=1}^T \mathbb{E}\left[\ell(\mathcal{P}(x_t), y_t)|B, A\right] \leq \inf_{E \in \mathcal{E}_B} \sum_{t=1}^T \ell(E(x_t), y_t) + M\sqrt{2T\ln(|\mathcal{E}_B|)}.$$
Taking an outer expectation gives
$$\mathbb{E}\left[\sum_{t=1}^T \ell(\mathcal{P}(x_t), y_t)\right] \leq \mathbb{E}\left[\inf_{E \in \mathcal{E}_B} \sum_{t=1}^T \ell(E(x_t), y_t) \right] + \mathbb{E}\left[M\sqrt{2T\ln(|\mathcal{E}_B|)}\right].$$
Therefore, 
\begin{align*}
    \mathbb{E}\left[\sum_{t=1}^T \ell(\mathcal{Q}(x_t), y_t) \right]  &= \mathbb{E}\left[\sum_{t=1}^T \ell(\mathcal{P}(x_t), y_t) \right] \\
    &\leq \mathbb{E}\left[\inf_{E \in \mathcal{E}_B} \sum_{t=1}^T \ell(E(x_t), y_t) \right] + \mathbb{E}\left[M\sqrt{2T\ln(|\mathcal{E}_B|)}\right]\\
    &\leq \mathbb{E}\left[\sum_{t=1}^T \ell(E_{B, \phi_B^{f^{\star}}}(x_t), y_t) \right] + M\mathbb{E}\left[\sqrt{2T\ln(|\mathcal{E}_B|)}\right].
\end{align*}
%
% For every fixed $B$, using Theorem 21.11 in \cite{ShwartzDavid}, REWA guarantees that 
% $$\mathbb{E}_{A, P}\left[\sum_{t=1}^T \ell(\mathcal{P}(x_t), y_t)\right] \leq \mathbb{E}_{A}\left[\inf_{E \in \mathcal{E}_B} \sum_{t=1}^T \ell(E(x_t), y_t)\right] + M\sqrt{2T\ln(|\mathcal{E}_B|)}.$$
%  Thus, using the independence of $B, A$, and $P$, we have that 
% \begin{align*}
%     \mathbb{E}\left[\sum_{t=1}^T \ell(\mathcal{Q}(x_t), y_t) \right]  &= \mathbb{E}_B\left[\mathbb{E}_{A, P}\left[\sum_{t=1}^T \ell(\mathcal{Q}(x_t), y_t)\right]\right]\\
%     &= \mathbb{E}_B\left[\mathbb{E}_{A, P}\left[\sum_{t=1}^T \ell(\mathcal{P}(x_t), y_t)\right]\right]\\
%     &\leq \mathbb{E}_B\left[\mathbb{E}_{A}\left[\inf_{E \in \mathcal{E}_B} \sum_{t=1}^T \ell(E(x_t), y_t)\right] + M\sqrt{2T\ln(|\mathcal{E}_B|)}\right] && \text{(REWA guarantee)}\\
%     &= \mathbb{E}_B\left[\mathbb{E}_{A}\left[\inf_{E \in \mathcal{E}_B} \sum_{t=1}^T \ell(E(x_t), y_t)\right]\right] + M\mathbb{E}_B\left[\sqrt{2T\ln(|\mathcal{E}_B|)}\right]\\
%     &= \mathbb{E}\left[\inf_{E \in \mathcal{E}_B} \sum_{t=1}^T \ell(E(x_t), y_t) \right] + M\mathbb{E}_B\left[\sqrt{2T\ln(|\mathcal{E}_B|)}\right]\\
%     &\leq \mathbb{E}\left[\sum_{t=1}^T \ell(E_{B, \phi_B^{f^{\star}}}(x_t), y_t) \right] + M\mathbb{E}_B\left[\sqrt{2T\ln(|\mathcal{E}_B|)}\right].
% \end{align*}
In the last step, we used the fact that for all $b \in \{0, 1\}^T$ and $f \in \mathcal{F}$, we have $E_{b, \phi_b^f} \in \mathcal{E}_b$. 

It now suffices to upperbound $\mathbb{E}\left[\sum_{t=1}^T \ell(E_{B, \phi_B^{f^{\star}}}(x_t), y_t) \right]$. To do so, we need some additional notation. Given the realizable online learner $\mathcal{A}$, an instance $x \in \mathcal{X}$, and an ordered finite sequence of labeled examples $L \in (\mathcal{X} \times \mathcal{Y})^*$, let $\mathcal{A}(x|L)$ be the random variable denoting the prediction of $\mathcal{A}$ on the instance $x$ after running and updating on $L$. For any $b\in \{0, 1\}^T$, $f \in \mathcal{F}$, and $t \in [T]$, let $L^f_{b_{< t}} = \{(x_i, f(x_i)): i < t \text{ and } b_i = 1\}$ denote the \textit{subsequence} of the sequence of labeled instances $\{(x_i, f(x_i))\}_{i=1}^{t-1}$ where $b_i = 1$. Using this notation, we can write

% Let $\mathcal{A}^{\star}$ denote the copy of $\mathcal{A}$ used by expert $E_{B, \phi_B^{f^{\star}}}$ for all $B \in \{0, 1\}^T$. By definition of $E_{B, \phi_B^{f^{\star}}}$ in Algorithm \ref{alg:expert_olreal2agn}, we can interpret this quantity as the expected cumulative loss of an online learning algorithm that first samples a bitstring $B \in \{0, 1\}^T$ such that $B_t \sim \text{Bernoulli}(\frac{T^{\beta}}{T})$, plays $\mathcal{A}^{\star}(x_t)$ in each round $t \in [T]$, but only updates the $\mathcal{A}^{\star}$ on rounds where $B_t = 1$ using labeled instances $(x_t, f^{\star}(x_t))$. From this perspective, we have that: 

\begin{align*}
    \mathbb{E}\left[\sum_{t=1}^T \ell(E_{B, \phi_B^{f^{\star}}}(x_t), y_t) \right] &=  \mathbb{E}\left[\sum_{t=1}^T \ell(\mathcal{A}(x_t|L_{B_{< t}}^{f^{\star}}) , y_t) \right]\\
    &= \mathbb{E}\left[\sum_{t=1}^T \ell(\mathcal{A}(x_t|L_{B_{< t}}^{f^{\star}}), y_t)\frac{\mathbbm{P}\left[B_t = 1 \right]}{\mathbbm{P}\left[B_t = 1 \right]} \right]\\
    &= \frac{T}{T^{\beta}}\sum_{t=1}^T \mathbb{E}\left[\ell(\mathcal{A}(x_t|L_{B_{< t}}^{f^{\star}}), y_t)\mathbbm{P}\left[B_t = 1 \right] \right]\\
    &= \frac{T}{T^{\beta}}\sum_{t=1}^T \mathbb{E}\left[\ell(\mathcal{A}(x_t|L_{B_{< t}}^{f^{\star}}), y_t)\mathbbm{1}\{B_t = 1 \}\right].
\end{align*}

% \begin{align*}
%     \mathbb{E}\left[\sum_{t=1}^T \ell(E_{B, \phi_B^{f^{\star}}}(x_t), y_t) \right] &=  \mathbb{E}\left[\sum_{t=1}^T \ell(\mathcal{A}^{\star}(x_t) , y_t) \right]\\
%     &= \mathbb{E}\left[\sum_{t=1}^T \ell(\mathcal{A}^{\star}(x_t), y_t)\frac{\mathbbm{P}\left[B_t = 1 \right]}{\mathbbm{P}\left[B_t = 1 \right]} \right]\\
%     &= \frac{T}{T^{\beta}}\sum_{t=1}^T \mathbb{E}\left[\ell(\mathcal{A}^{\star}(x_t), y_t)\mathbbm{P}\left[B_t = 1 \right] \right]\\
%     &= \frac{T}{T^{\beta}}\sum_{t=1}^T \mathbb{E}\left[\ell(\mathcal{A}^{\star}(x_t), y_t)\mathbbm{1}\{B_t = 1 \}\right].
% \end{align*}

To see the last equality, note that the prediction $\mathcal{A}(x_t|L_{B_{< t}}^{f^{\star}})$ only depends on bitstring ($B_1, \ldots, B_{t-1}$) and the internal randomness of $A$, both of which are independent of $B_t$. Thus, we have  

% To see the last equality,  observe that $E_{B, \phi_B^{f^{\star}}}$ updates $\Acal^{\star}$  by passing $(x_t, f^{\star}(x_t))$ whenever $B_t =1$. However, the prediction of $\Acal^{\star}$ on round $t$ only depends on bitstring ($B_1, \ldots, B_{t-1}$), but is independent of $B_t$. Thus, we have   

\begin{align*}
    \mathbb{E}\left[\ell(\mathcal{A}(x_t|L_{B_{< t}}^{f^{\star}}), y_t) \, \indicator\{B_t =1\}\right]  &= \mathbb{E}\left[\ell(\mathcal{A}(x_t|L_{B_{< t}}^{f^{\star}}), y_t) \right] \, \mathbb{E}\left[ \indicator \{B_t=1\}\right]\\
    &= \mathbb{E}\left[\ell(\mathcal{A}(x_t|L_{B_{< t}}^{f^{\star}}), y_t) \right] \mathbbm{P}[B_t =1]
\end{align*}
as needed. Continuing onwards, 
\begin{align*}
    \mathbb{E}\left[\sum_{t=1}^T \ell(E_{B, \phi_B^{f^{\star}}}(x_t), y_t) \right] &= \frac{T}{T^{\beta}}\mathbb{E}\left[\sum_{t=1}^T \ell(\mathcal{A}(x_t|L_{B_{< t}}^{f^{\star}}), y_t)\mathbbm{1}\{B_t = 1 \}\right]\\
    &= \frac{T}{T^{\beta}}\mathbb{E}\left[\sum_{t: B_t = 1} \ell(\mathcal{A}(x_t|L_{B_{< t}}^{f^{\star}}), y_t)\right]\\
    &\leq   \frac{cT}{T^{\beta}}\mathbb{E}\left[\sum_{t: B_t = 1} \ell(\mathcal{A}(x_t|L_{B_{< t}}^{f^{\star}}), f^{\star}(x_t))\right] + \frac{T}{T^{\beta}}\mathbb{E}\left[\sum_{t: B_t = 1} \ell(f^{\star}(x_t), y_t)\right]\\
    &= \frac{cT}{T^{\beta}}\mathbb{E}\left[\sum_{t: B_t = 1} \ell(\mathcal{A}(x_t|L_{B_{< t}}^{f^{\star}}), f^{\star}(x_t))\right] + \sum_{t=1}^T \ell(f^{\star}(x_t), y_t)\\
    &= \frac{cT}{T^{\beta}}\mathbb{E}\left[\sum_{t: B_t = 1} \ell(\mathcal{A}(x_t|L_{B_{< t}}^{f^{\star}}), f^{\star}(x_t))\right] + \inf_{f \in \mathcal{F}}\sum_{t=1}^T \ell(f(x_t), y_t)
\end{align*}
The inequality follows from the fact that $\ell$ is a $c$-subadditive and the last equality follows from the definition of $f^{\star}$. We now need to bound $\frac{cT}{T^{\beta}}\mathbb{E}\left[\sum_{t: B_t = 1} \ell(\mathcal{A}(x_t|L_{B_{< t}}^{f^{\star}}) , f^{\star}(x_t))\right]$. Using the fact that $\mathcal{A}^{\star}$ is a realizable online learner and gets updated on a stream of instances labeled by $f^{\star}$ only on rounds where $B_t = 1$, we get
\begin{align*}
    \frac{cT}{T^{\beta}}\mathbb{E}\left[\sum_{t: B_t = 1} \ell(\mathcal{A}(x_t|L_{B_{< t}}^{f^{\star}}), f^{\star}(x_t))\right] &=  \frac{cT}{T^{\beta}}\mathbb{E}\left[\mathbb{E}\left[\sum_{t: B_t = 1} \ell(\mathcal{A}(x_t|L_{B_{< t}}^{f^{\star}}), f^{\star}(x_t)) \bigg| B\right]\right]\\
    &\leq \frac{cT}{T^{\beta}}\mathbb{E}\left[ R(|B|, |\text{im}(\mathcal{F})|)\right].
\end{align*}
Putting things together, we find that, 
\begin{align*}
    \mathbb{E}\left[\sum_{t=1}^T \ell(\mathcal{Q}(x_t), y_t) \right] 
    &\leq \mathbb{E}\left[\sum_{t=1}^T \ell(E_{B, \phi_B^{f^{\star}}}(x_t), y_t) \right] + M\mathbb{E}\left[\sqrt{2T\ln(|\mathcal{E}_B|)}\right]\\
    &\leq  \inf_{f \in \mathcal{F}}\sum_{t= 1}^T \ell(f(x_t), y_t) + \frac{cT}{T^{\beta}} \mathbb{E}\left[R(|B|, |\text{im}(\mathcal{F})|)\right] + M\mathbb{E}\left[\sqrt{2T\ln(|\mathcal{E}_B|)}\right]\\
    &\leq \inf_{f\in \mathcal{F}}\sum_{t= 1}^T \ell(f(x_t), y_t) + \frac{cT}{T^{\beta}} \mathbb{E}\left[R(|B|, |\text{im}(\mathcal{F})|)\right] + M\mathbb{E}\left[\sqrt{2T|B|\ln(|\text{im}(\mathcal{F})|)}\right],\\
\end{align*}
where the last inequality follows from the fact that that $|\mathcal{E}_B| \leq (|\text{im}(\mathcal{F})|)^{|B|}$. By Jensen's inequality, we further get that, $\mathbb{E}\left[\sqrt{2T|B|\ln(|\text{im}(\mathcal{F})|)}\right] \leq \sqrt{2T^{\beta + 1}\ln(|\text{im}(\mathcal{F})|)}$, which implies that 

$$\mathbb{E}\left[\sum_{t=1}^T \ell(\mathcal{Q}(x_t), y_t) \right]  \leq \inf_{f\in \mathcal{F}}\sum_{t= 1}^T \ell(f(x_t), y_t) + \frac{cT}{T^{\beta}} \mathbb{E}\left[R(|B|, |\text{im}(\mathcal{F})|)\right] + M\sqrt{2T^{\beta + 1}\ln(|\text{im}(\mathcal{F})|)}.$$

Next,  by Lemma \ref{lem:woess}, there exists a concave sublinear function $ \overline{R}(|B|, |\text{im}(\Fcal)|)$ that upperbounds $ R(|B|, |\text{im}(\Fcal)|)$. By Jensen's inequality, we obtain $ \expect[\overline{R}(|B|, |\text{im}(\Fcal)|)] \leq \overline{R}(T^{\beta}, |\text{im}(\Fcal)|)$, which yields

$$\mathbb{E}\left[\sum_{t=1}^T \ell(\mathcal{Q}(x_t), y_t) \right]  \leq \inf_{f\in \mathcal{F}}\sum_{t= 1}^T \ell(f(x_t), y_t) + \frac{cT}{T^{\beta}}\,\overline{R}(T^{\beta}, |\text{im}(\mathcal{F})|) + M\sqrt{2T^{\beta + 1}\ln(|\text{im}(\mathcal{F})|)}.$$

\noindent This completes the proof as we have shown that $\mathcal{Q}$ is an agnostic online learner for $\mathcal{F}$ with respect to $\ell$ with the stated regret bound. \end{proof}

\subsection{Online Multilabel Classification}
\label{sec:onlineml}

Let $\Ycal = \{-1,1\}^K$. We provide analogs of Theorem \ref{thm:multi_label_batch_hamming} and \ref{thm:batchgenloss} in the online setting. We begin by characterizing the learnability of the Hamming loss and then move to give a characterization of learnability for all losses satisfying the identity of indiscernibles. Similar to the batch setting, we can show that the MCLdim of $\mathcal{F}$ characterizes online multilabel learnability (see Appendix \ref{app:MCLdimchar}), but here, we give a characterization that better exploits the multilabel structure of the problem.

\subsubsection{Characterizing Online Learnability for the Hamming Loss}
\label{sec:onlinehamming}

Theorem \ref{thm:OLHam} characterizes the online learnability of a multilabel function class $\mathcal{F}$ with respect to $\ell_H$. 

\begin{theorem} 
\label{thm:OLHam}
A function class $\mathcal{F} \subseteq\mathcal{Y}^{\mathcal{X}}$ is online learnable with respect to the Hamming loss if and only if each restriction $\mathcal{F}_k \subseteq \mathcal{Y}_k^{\mathcal{X}}$ is online learnable with respect to the 0-1 loss. 
\end{theorem}

\noindent The proof of Theorem \ref{thm:OLHam} is similar to that of Theorem \ref{thm:multi_label_batch_hamming}, so we defer the full proof to Appendix \ref{appdx:proof_OLHam} and only provide a sketch here.  The proof of sufficiency direction is based on a reduction: given oracle access to online learners $\{\mathcal{A}_k\}_{k=1}^K$ for $\{\mathcal{F}_k\}_{k=1}^{K}$ with respect to $\ell_{0\text{-}1}$, we construct an online learner $\mathcal{A}$ for $\mathcal{F}$ with respect to $\ell_H$. In fact, similar to the batch setting, the online multilabel learning algorithm $\mathcal{A}$ is simple: in each round $t \in [T]$, receive $x_t$, query the predictions $\Acal_1(x_t), ..., \Acal_K(x_t)$, and finally predict the concatenation $\hat{y}_t = (\Acal_1(x_t), ..., \Acal_K(x_t))$. Once the true label $y_t = (y_t^1, ..., y_t^K)$ is revealed, update each online learner $\mathcal{A}_k$ by passing $(x_t, y_t^k)$ for $k \in [K]$. Using some algebra, one can show that this prediction rule achieves sublinear regret for $\mathcal{F}$. 

For the necessity direction, given oracle access to an online learner $\mathcal{A}$ for $\mathcal{F}$ with respect to $\ell_H$, we construct an online learner $\mathcal{B}$ for $\mathcal{F}_1$ with respect to $\ell_{0\text{-}1}$. A similar reduction can be used to construct online learners for each restriction $\Fcal_k$.  Similar to the batch setting, the online learning algorithm $\mathcal{B}$ is simple: in each round $t \in [T]$, receive $x_t$, query $\hat{y}_t = \Acal(x_t)$ and predict $\hat{y}^1_t = \Acal_1(x_t)$. Once the true label $y^1_t$ is revealed, update $\Acal$ by passing $(x_t, y_t)$ where $y_t$ = $(y_t^1, \sigma_t^2, ..., \sigma_t^K)$ and $\{\sigma_t^i\}_{i=2}^K$ is an i.i.d sequence of Rademacher random variables. A straightforward analysis shows that such a prediction rule achieves sublinear regret for $\Fcal_1$. 

\subsubsection{Characterizing Online Learnability for General Losses}

Using Theorem \ref{thm:onlreal2agn} and Theorem \ref{thm:OLHam}, we now characterize the learnability of arbitrary multilabel loss functions $\ell$ as long as they satisfy the identity of indiscernibles. The key idea is that since there are only finite number of possible inputs to $\ell$, for any $\ell$ satisfying the identity of indiscernibles, there must exist universal constants $a$ and $b$ such that $a\ell_H(y_1, y_2) \leq \ell(y_1, y_2) \leq b\ell_H(y_1, y_2)$. Then, we can characterize the learnability of $\ell$ by relating it to the learnability of $\ell_H$. In fact, we prove a slightly more general result, showing an equivalence between the learnability of any two arbitrary losses satisfying the identity of indiscernibles. 

 \begin{lemma} 
\label{lem: arbOL}
Let $\ell$ and $\ell^{\prime}$ be any two loss functions satisfying the identity of indiscernibles. A function class $\mathcal{F} \subseteq \mathcal{Y}^{\mathcal{X}}$ is online learnable with respect to $\ell$ if and only if $\mathcal{F}$ is online learnable with respect to $\ell^{\prime}$.
\end{lemma}
\noindent  The proof of Lemma \ref{lem: arbOL} is similar to that of Lemma \ref{lem:hamming_general_batch} with the main difference being the use of Theorem \ref{thm:onlreal2agn} instead of Lemma \ref{realizable_agnostic_equiv}. Since Lemma \ref{lem: arbOL} is our first application of Theorem \ref{thm:onlreal2agn}, we provide the full proof here.

\begin{proof} 
Since $\ell$ and $\ell^{\prime}$ are arbitrary, it suffices to prove only one direction. To that end, suppose $\Fcal$ is online learnable with respect to $\ell$. We now show that $\Fcal$ is online learnable with respect to $\ell^{\prime}$ as well. 

Let $a$ and $b$ be the universal constants such that for all $y_1, y_2 \in \mathcal{Y}$, $a\ell(y_1, y_2) \leq \ell^{\prime}(y_1, y_2) \leq b\ell(y_1, y_2)$. Let $c = \frac{\max_{r \neq t}\ell^{\prime}(r, t)}{\min_{r \neq t}\ell^{\prime}(r, t)}$.  Since $|\text{im}(\mathcal{F})| = 2^K < \infty$ and $\ell^{\prime}$ is a $c$-subadditive, by Theorem \ref{thm:onlreal2agn}, it suffices to give a realizable online learner for $\mathcal{F}$ with respect to $\ell^{\prime}$. Since $\mathcal{F}$ is online learnable with respect to $\ell$, there exists an algorithm $\mathcal{A}$ such that for any  sequence $(x_1, y_1), ..., (x_T, y_T)$, we have 

$$\mathbb{E}\left[\sum_{t=1}^T \ell(\mathcal{A}(x_t), y_t) - \inf_{f \in \mathcal{F}}\sum_{t=1}^T \ell(f(x_t), y_t)\right] \leq R(T, 2^K) $$
where $R(T, 2^K)$ is a sublinear function of $T$. In the realizable setting, we are guaranteed that for any sequence $(x_1, y_1), ..., (x_T, y_T)$ that the online learner may observe, there exists a $f \in \mathcal{F}$ s.t $f(x_t) = y_t$ for all $t \in [T]$. Since $\ell$ satisfies the identity of indiscernibles, we have that for any realizable sequence $(x_1, y_1), ..., (x_T, y_T)$, $\inf_{f \in \mathcal{F}} \sum_{t=1}^T \ell(f(x_t), y_t) = 0$. Thus, we have that $\mathbb{E}\left[\sum_{t=1}^T \ell(\mathcal{A}(x_t), y_t)\right] \leq R(T, 2^K).$ Noting that $\ell(\mathcal{A}(x_t), y_t) \geq \frac{\ell^{\prime}(\mathcal{A}(x_t), y_t)}{b}$ implies that $\mathbb{E}\left[\sum_{t=1}^T \ell^{\prime}(\mathcal{A}(x_t), y_t)\right] \leq bR(T, 2^K)$, showing that $\mathcal{A}$ is also a realizable online learner for $\mathcal{F}$ with respect to $\ell^{\prime}$. For any $\beta \in (0, 1)$, the construction in Theorem \ref{thm:onlreal2agn} can then be used to convert $\mathcal{A}$ into an agnostic online learner for $\mathcal{F}$ with respect to $\ell^{\prime}$ with expected regret bound 
$$
\frac{cbT}{T^{\beta}}\overline{R}(T^{\beta}, 2^K) + M\sqrt{4KT^{1+\beta}}
$$
where $M$  is such that $\ell \leq M$ and $ \overline{R}(T^{\beta}, 2^K)$ is any concave sublinear upperbound of $R(T^{\beta}, 2^K)$. This completes our proof. 

% The reverse direction follows identically and so we only provide a sketch here. Using a similar argument as above, since $\ell$ and $\ell_H$ are zero-matched, an online learner for $\mathcal{F}$ with respect to $\ell$ is a \textit{realizable} online learner for $\mathcal{F}$ with respect to $\ell_H$. Since $\ell_H$ is c-subbadditive and $|\text{im}(\mathcal{F})| = 2^K < \infty$, for any $\beta \in (0, 1)$, Theorem \ref{thm:onlreal2agn}, using the realizable online learner for $\mathcal{F}$ with respect to $\ell_H$,  gives an agnostic online learner for $\mathcal{F}$ with respect to $\ell_H$ with sublinear expected regret. 
\end{proof}

As an immediate consequence of Lemma \ref{lem: arbOL} and Theorem \ref{thm:OLHam}, we get the following theorem characterizing the online learnability of general multilabel losses. 

\begin{theorem}
\label{thm:onlinecharac}
Let $\ell$ be any multilabel loss function that satisfies the identity of indiscernibles. A function class $\mathcal{F} \subseteq \mathcal{Y}^{\mathcal{X}}$ is online learnable with respect to $\ell$ if and only if each restriction $\mathcal{F}_k \subseteq \mathcal{Y}_k^{\mathcal{X}}$ is online learnable with respect to the 0-1 loss. 
\end{theorem}

\noindent \textbf{Remark.} Since the Littlestone dimension characterizes online learnability for binary classification under the 0-1 loss \citep{ben2009agnostic}, Theorem \ref{thm:onlinecharac} also implies that finiteness of $\text{Ldim}(\mathcal{F}_k)$ for all $k \in [K]$ is a necessary and sufficient condition for online multilabel learnability. 

Moreover, if $\text{Ldim}(\mathcal{F}_k)< \infty$ for all $k \in [K]$, then we have $\text{MCLdim}(\Fcal)< \infty$. This follows from the fact that $\text{MCLdim}(\mathcal{F}) \leq \sum_{k=1}^{K}\, \text{Ldim}(\mathcal{F}_k). $ To see this, note that $\text{MCLdim}(\mathcal{F})$ is the lowerbound on the number of mistakes of any deterministic multiclass learner in the realizable setting \cite[Theorem 17]{DanielyERMprinciple}. On the other hand, one can construct a deterministic realizable learner for $\mathcal{F}$ using $K$ different Standard Optimal Algorithms (SOA) for binary function classes $\mathcal{F}_k$'s. Namely, define an algorithm $\Acal$ such that $\mathcal{\Acal}(x) :=(\text{SOA}(\mathcal{F}_1)(x), \ldots, \text{SOA}(\mathcal{F}_K)(x)) \in \{-1,1\}^K$. Since each $\text{SOA}(\mathcal{F}_k)$ makes at most $\text{Ldim}(\mathcal{F}_k)$ number of mistakes, $\mathcal{A}$ makes no more than $\sum_{k=1}^{K}\, \text{Ldim}(\mathcal{F}_k)$ mistakes. We can use this fact to give an improved version of Theorem \ref{thm:onlreal2agn} for classes $\Fcal$ with $\text{MCLdim}(\Fcal)< \infty$. In particular, when $\Ycal = \{-1,1\}^K$, any $\Fcal \subseteq \Ycal^{\Xcal}$  that is learnable in the realizable setting with respect to $\ell$ is also learnable in the agnostic setting with regret  $O\Bigl(B \sqrt{T \text{MCLdim}(\Fcal) \ln(T)}\Bigl) \leq O\Bigl(B \sqrt{T \sum_{k=1}^K \text{Ldim}(\Fcal_k) \ln(T)}\Bigl)$. Here, $B$ is the maximum value $\ell$ can take. The improved regret bound can be found in the sufficiency proof of Theorem \ref{thm:MCLdimchar} in Appendix \ref{app:MCLdimchar}.

\subsection{Bandit Online Multilabel Classification}
\label{sec:bandit}

We extend the results in the previous subsection to the online setting where the learner only observes \textit{bandit} feedback in each round.  Theorem \ref{thm:banditcharac} gives a characterization of bandit online learnability of a function class $\mathcal{F}$ in terms of the online learnability of each restriction. 

\begin{theorem}
\label{thm:banditcharac}
Let $\ell$ be any loss function that satisfies the identity of indiscernibles.  A function class $\mathcal{F} \subseteq \mathcal{Y}^{\mathcal{X}}$ is bandit online learnable  with respect to $\ell$ if and only if each restriction $\mathcal{F}_k \subseteq \mathcal{Y}_k^{\mathcal{X}}$ is online learnable with respect to the 0-1 loss. 
\end{theorem}

\noindent Similar to the full-feedback setting, Theorem \ref{thm:banditcharac} also gives that the finiteness of $\text{Ldim}(\mathcal{F}_k)$ for all $k \in [K]$ is a necessary and sufficiency condition for bandit online multilabel learnability.
The proof of Theorem \ref{thm:banditcharac} uses the realizable-to-agnostic conversion for \textit{bandit} feedback setting when the label space $\mathcal{Y}$ is finite. The following Theorem makes this argument precise.

\begin{theorem}
\label{thm:banditonlreal2agn}
Let $\mathcal{Y}$ be a finite label space,  $\mathcal{F} \subseteq\mathcal{Y}^{\mathcal{X}}$ a multioutput function class, and $\ell: \mathcal{Y} \times \mathcal{Y} \rightarrow \mathbb{R}_{\geq 0}$ be any $c$-subadditive loss function such that $\ell(\cdot, \cdot) \leq M$. If $\mathcal{A}$ is a realizable online learner for $\mathcal{F}$ with respect to $\ell$ under full-feedback with sub-linear expected regret $R(T, |\mathcal{Y}|)$, then for every $\beta \in (0, 1)$, there exists an online learner for $\mathcal{F}$ with respect to $\ell$ with expected regret 

$$\frac{cT}{T^{\beta}}\overline{R}(T^{\beta}, |\mathcal{Y}|) + eM\sqrt{2T^{1+\beta}|\mathcal{Y}|\ln(|\mathcal{Y}|)}, $$
under \emph{bandit feedback}, where $\overline{R}(T,  |\Ycal|)$ is any concave, sublinear upperbound on $R(T, |\Ycal|)$.
\end{theorem}

\noindent The proofs for Theorem \ref{thm:banditcharac} and Theorem \ref{thm:banditonlreal2agn} are provided in Appendix \ref{appdx:proofs_bandit}.

\subsection{Online Multioutput Regression}
In this section, we characterize the online learnability of multioutput function classes. Similar to the batch setting, we consider, without loss of generality, the case when $\Ycal =[0,1]^K \subset \reals^K$ for  $K \in \naturals$. In addition, we consider the same set of decomposable and non-decomposable loss functions as in the batch setting. Namely, our decomposable loss functions satisfy Assumptions \ref{assumption1} and \ref{assumption2}, and our non-decomposable loss functions are $\ell_p$ norms. Informally, our main result asserts that a multioutput function class $\mathcal{F} \subset \mathcal{Y}^{\mathcal{X}}$ is online learnable if and only if each restriction $\mathcal{F}_k$ is online learnable.
% Throughout this section, we give regret bounds assuming an \textit{oblivious} adversary. A standard reduction (Chapter 4 in \cite{cesa2006prediction}) allows us to convert oblivious regret bounds to adaptive regret bounds in the full-information setting. 

\subsubsection{Characterizing Learnability for Decomposable Losses}
In this subsection, we characterize the online learnability of multioutput function classes with respect to decomposable losses satisfying Assumption \ref{assumption1}. Our main theorem is presented below. 

\begin{theorem} \label{thm:onlinereg_decomp}
    Let $\ell$ be a decomposable loss function satisfying \emph{Assumption \ref{assumption1}}. A multioutput function class $\Fcal \subseteq \Ycal^{\Xcal}$ is online  learnable with respect to  $\ell$ if and only if each $\Fcal_k \subseteq \Ycal_k^{\Xcal}$ is online  learnable with respect to $\psi_k \circ d_1$.
\end{theorem}

\begin{proof}
As usual, we prove Theorem \ref{thm:onlinereg_decomp} in two parts: first sufficiency and then necessity.  The sufficiency proof is similar to that of the proof for Hamming loss in Theorems \ref{thm:multi_label_batch_hamming} and \ref{thm:OLHam}. The necessity direction is more involved, but the main idea is to combine the augmentation technique used in the proof of Theorem \ref{thm:batch_reg_decom} with the algorithmic conversion technique developed in the proof of Theorem \ref{thm:onlreal2agn}. 

\noindent\textbf{Part 1: Sufficiency.}
We first prove that online learnability of each restriction $\mathcal{F}_k$ with respect to $\psi_k \circ d_1$ is sufficient for online learnability of $\mathcal{F}$ with respect to $\ell$. Since $\ell(f(x), y) = \sum_{k=1}^K \psi_k \circ d_1 (f_k(x), y_k)$ is decomposable, we can use the exact same strategy as in Section \ref{sec:onlinehamming} to convert online learners $\mathcal{A}_1$, ..., $\mathcal{A}_K$ for $\mathcal{F}_1, ..., \mathcal{F}_K$ with respect to $\psi_k \circ d_1$ to an online learner $\mathcal{A}$ for $\mathcal{F}$ with respect to $\ell$. More specifically, in each round $t \in [T]$, receive $x_t$, query the predictions $\Acal_1(x_t), ..., \Acal_K(x_t)$, and finally predict the concatenation $\hat{y}_t = (\Acal_1(x_t), ..., \Acal_K(x_t))$. Once the true label $y_t = (y_t^1, ..., y_t^K)$ is revealed, update each online learner $\mathcal{A}_k$ by passing $(x_t, y_t^k)$ for $k \in [K]$. Using the exact same proof as in Section \ref{sec:onlinehamming}, it follows that the expected regret of $\mathcal{A}$ is $\sum_{k=1}^K R_k(T)$ where $R_k(T)$ is the regret of online algorithm $\mathcal{A}_k$. Since $K$ is finite, the regret of $\mathcal{A}$ is sublinear in $T$ when evaluated using $\ell$.

%To prove necessity, we adapt the technique used to prove the necessity direction in the batch setting to the online setting. 
\noindent
\textbf{Part 2: Necessity.}
Similar to the batch setting, we prove the necessity direction of Theorem  \ref{thm:onlinereg_decomp} constructively. That is, given oracle access to an online learner $\mathcal{A}$ for $\mathcal{F}$ with respect to $\ell$, we construct an online learner $\mathcal{Q}$ for $\mathcal{F}_1$ with respect to $\psi_1 \circ d_1$. By symmetry, a similar reduction can be used to construct online learners for each restriction $\mathcal{F}_k$. As mentioned before, we assume an oblivious adversary, and therefore the stream of points to be observed by the online learner, denoted $(x_1, y_1), ..., (x_T, y_T) \in (\mathcal{X} \times [0, 1])^T$, is fixed beforehand. Let $f_1^{\star} = \argmin_{f_1 \in \mathcal{F}_1} \sum_{t=1}^T \psi_k \circ d_1(f_1(x_t), y_t)$ denote the optimal function in hindsight and $f^{\star} \in \mathcal{F}$ its completion. 

Since we are trying to construct an online learner for $\mathcal{F}_1$, the targets $y_1, ..., y_T$ are \textit{scalar-valued}. However, $\mathcal{A}$ is an online learner for $\mathcal{F}$ and therefore can only processes \textit{vector-valued} targets. Thus, we need to figure out how to augment the scalar-valued targets $y_1, ..., y_T$ in a way that allows us to use $\mathcal{A}$ to construct an online learner $\mathcal{Q}$ for $\mathcal{F}_1$. Following a similar strategy as in the proof of Theorem \ref{thm:onlreal2agn}, we can construct a set of Experts that simulate online games with $\mathcal{A}$ by augmenting, in all possible ways, the scalar-valued targets of a \textit{sub-sample} of the stream into vector-valued targets using vectors in $\mathcal{Y}_{2:k}^{\alpha}$, the discretized label space for components 2 through $K$.  In particular, our high-level strategy is to: 
\begin{enumerate}
    \item Randomly \textit{sub-sample} points from the stream
    \item Construct a set of Experts, each of which:
        \begin{enumerate}
            \item Uses an independent copy of $\mathcal{A}$ to make predictions $\hat{y}_t = \mathcal{A}_1(x_t)$
            \item Augments the scalar-valued targets of each labeled instance in the \textit{sub-sampled} stream to vector-valued targets using vectors in the discretized image space $\text{im}(\mathcal{F}_{2:K}^{\alpha})$
            \item Simulates an online game with its independent copy of $\mathcal{A}$ over only the augmented \textit{sub-sampled} stream with vector-valued targets
        \end{enumerate}
    \item Run REWA using the set of experts in Step 2 and the $\psi_1 \circ d_1$ loss function over the \textit{original} stream of points. 
\end{enumerate}

We now formalize this idea. For any bitstring $b \in \{0, 1\}^T$, let $\phi: \{t: b_t = 1\} \rightarrow \text{im}(\mathcal{F}_{2:K}^{\alpha})$ denote a function mapping time points where $b_t = 1$ to vectors in the discretized image space $\text{im}(\mathcal{F}_{2:K}^{\alpha})$. Let $\Phi_b \subseteq (\text{im}(\mathcal{F}_{2:K}^{\alpha}))^{\{t: b_t = 1\}}$ denote all such functions $\phi$. For every $f \in \mathcal{F}$, let $\phi_b^f \in \Phi_b$ be the mapping such that for all $t \in \{t: b_t = 1\}$, $\phi_b^f(t) = f_{2:K}^{\alpha}(x_t)$. Let $|b| = |\{t: b_t = 1\}|$. For every $b \in \{0, 1\}^T$ and $\phi \in \Phi_b$, define an Expert $E_{b, \phi}$. Expert $E_{b, \phi}$, formally presented in Algorithm \ref{alg:expert_decompnecc}, uses $\mathcal{A}$ to make predictions in each round. However, $E_{b, \phi}$ only updates $\mathcal{A}$ on those rounds where $b_t = 1$, using $\phi$ to augment the scalar-valued labeled instance $(x_t, y_t)$ to the vector-valued labeled instance $(x_t, (y_t, \phi(t)))$. For every $b \in \{0, 1\}^T$, let $\mathcal{E}_b = \bigcup_{\phi \in \Phi_b} \{E_{b, \phi}\}$ denote the set of all Experts parameterized by functions $\phi \in \Phi_b$. If $b$ is the all zeros bitstring, then $\mathcal{E}_b$ is empty. Therefore, we actually define $\mathcal{E}_b = \{E_0\} \cup \bigcup_{\phi \in \Phi_b} \{E_{b, \phi}\}$, where $E_0$ is the expert that never updates $\mathcal{A}$ and plays $\mathcal{A}_1(x_t)$ for all $t \in [T]$.  Note that $1 \leq |\mathcal{E}_b| \leq (\frac{2}{\alpha})^{K|b|}$. 

\begin{algorithm}
\caption{Expert($b$, $\phi$)}
\label{alg:expert_decompnecc}
\setcounter{AlgoLine}{0}
\KwIn{Independent copy of Online Learner $\mathcal{A}$ for $\ell$}
\For{$t = 1,...,T$} {
    Receive example $x_t$

    Predict $\tilde{y}_t = \mathcal{A}_1(x_t)$

    Receive $y_t$
    
    \uIf{$b_t = 1$}{        
        Update $\mathcal{A}$ by passing $(x_t, (y_t, \phi(t)))$
    }
}
\end{algorithm}

With this notation in hand, we are now ready to present Algorithm \ref{alg:onlinereg_decompnec}, our main online learner $\mathcal{Q}$ for $\mathcal{F}_1$.

\begin{algorithm}
\caption{Online learner $\mathcal{Q}$ for $\Fcal_1$ with respect to $\psi_1 \circ d_1$}
\label{alg:onlinereg_decompnec}
\setcounter{AlgoLine}{0}
\KwIn{ Parameters $0 < \beta < 1$ and $0 < \alpha < 1$}

Let $B \in \{0, 1\}^T$ such that  $B_t \overset{\text{iid}}{\sim} \text{Bernoulli}(\frac{T^{\beta}}{T})$ 

% Let $\Phi_B = (\mathcal{Y}_{2:K}^{\alpha})^{\{t: B_t = 1\}}$ denote the set of all possible functions mapping from  ${\{t: B_t = 1\}}$ to $\mathcal{Y}_{2:K}^{\alpha}$.

Construct the set of experts $\mathcal{E}_B = \{E_0\} \cup \bigcup_{\phi \in \Phi_B} \{E_{B, \phi}\}$ according to Algorithm \ref{alg:expert_decompnecc}

Run REWA $\mathcal{P}$ using $\mathcal{E}_B$ and the loss function $\psi_1 \circ d_1$ over the stream $(x_1, y_1), ..., (x_T, y_T)$

\end{algorithm}

Our goal now is to show that $\mathcal{Q}$ enjoys sublinear expected regret. There are three main sources of randomness: the randomness involved in sampling $B$, the internal randomness  of each independent copy of the online learner $\mathcal{A}$, and the internal randomness  of REWA. Let $B, A$ and $P$ denote the random variable associated with these sources of randomness respectively. By construction, $B, A$, and $P$ are independent. 

Using Theorem 21.11 in \cite{ShwartzDavid} and the fact that $A, P$, and $B$ are independent, REWA guarantees  

$$\mathbb{E}\left[\sum_{t=1}^T \psi_1 \circ d_1(\mathcal{P}(x_t), y_t)\right] \leq \mathbb{E}\left[\inf_{E \in \mathcal{E}_B} \sum_{t=1}^T \psi_1 \circ d_1(E(x_t), y_t)\right] + \mathbb{E}\left[\sqrt{2T\ln(|\mathcal{E}_B|)} \right].$$
Thus,  
\begin{align*}
    \mathbb{E}\left[\sum_{t=1}^T \psi_1 \circ d_1(\mathcal{Q}(x_t), y_t) \right]  &= \mathbb{E}\left[\sum_{t=1}^T \psi_1 \circ d_1(\mathcal{P}(x_t), y_t)\right]\\
    &\leq \mathbb{E}\left[\inf_{E \in \mathcal{E}_B} \sum_{t=1}^T \psi_1 \circ d_1(E(x_t), y_t) \right] + \mathbb{E}\left[\sqrt{2T\ln(|\mathcal{E}_B|)}\right]\\
    &\leq \mathbb{E}\left[\sum_{t=1}^T \psi_1 \circ d_1(E_{B, \phi_B^{f^{\star}}}(x_t), y_t) \right] + \mathbb{E}\left[\sqrt{2T\ln(|\mathcal{E}_B|)}\right].
\end{align*}
In the last step, we used the fact that for all $b \in \{0, 1\}^T$ and $f \in \mathcal{F}$, $E_{b, \phi_b^f} \in \mathcal{E}_b$. 

It now suffices to upperbound $\mathbb{E}\left[\sum_{t=1}^T \psi \circ d_1(E_{B, \phi_B^{f^{\star}}}(x_t), y_t) \right]$. We use the same notation used to prove Theorem \ref{thm:onlreal2agn}, but for the sake of completeness, we restate it here. Given an online learner $\mathcal{A}$ for $\ell$, an instance $x \in \mathcal{X}$, and an ordered sequence of labeled examples $L  \in (\mathcal{X} \times [0, 1]^K)^*$, let $\mathcal{A}(x|L)$ be the random variable denoting the prediction of $\mathcal{A}$ on the instance $x$ after running and updating on $L$. For any $b\in \{0, 1\}^T$, $f_{2:K}^{\alpha} \in \mathcal{F}^{\alpha}_{2:K}$, and $t \in [T]$, let $L^f_{b_{< t}} = \{(x_i, (y_i, f_{2:K}^{\alpha}(x_i))): i < t \text{ and } b_i = i\}$ denote the \textit{subsequence} of the sequence of labeled instances $\{(x_i, (y_i, f_{2:K}^{\alpha}(x_i)))\}_{i=1}^{t-1}$ where $b_i = 1$. Using this notation, we can write

% Let $\mathcal{A}^{\star}$ denote the copy of $\mathcal{A}$ used by expert $E_{B, \phi_B^{f^{\star}}}$ for all $B \in \{0, 1\}^T$. By definition of $E_{B, \phi_B^{f^{\star}}}$ in Algorithm \ref{alg:expert_decompnecc}, we can interpret this quantity as the expected cumulative loss of an online learning algorithm that first samples a bitstring $B \in \{0, 1\}^T$ such that $B_t \sim \text{Bernoulli}(\frac{T^{\beta}}{T})$, plays $\mathcal{A}^{\star}_1(x_t)$ in each round $t \in [T]$, but only updates $\mathcal{A}^{\star}$ on rounds where $B_t = 1$ using the labeled instance $(x_t, (y_t, f^{\star, \alpha}_{2:K}(x_t))$. From this perspective, we have that: 

% \begin{align*}
%     \mathbb{E}\left[\sum_{t=1}^T \psi_1 \circ d_1(E_{B, \phi_B^{f^{\star}}}(x_t), y_t) \right] &=  \mathbb{E}\left[\sum_{t=1}^T \psi_1 \circ d_1(\mathcal{A}^{\star}_1(x_t), y_t) \right]\\
%     &= \mathbb{E}\left[\sum_{t=1}^T \psi_1 \circ d_1(\mathcal{A}^{\star}_1(x_t), y_t)\frac{\mathbbm{P}\left[B_t = 1 \right]}{\mathbbm{P}\left[B_t = 1 \right]} \right]\\
%     &= \frac{T}{T^{\beta}}\sum_{t=1}^T \mathbb{E}\left[\psi_1 \circ d_1(\mathcal{A}^{\star}_1(x_t), y_t)\mathbbm{P}\left[B_t = 1 \right] \right]\\
%     &= \frac{T}{T^{\beta}}\sum_{t=1}^T \mathbb{E}\left[\psi_1 \circ d_1(\mathcal{A}^{\star}_1(x_t), y_t)\mathbbm{1}\{B_t = 1 \}\right].
% \end{align*}

\begin{align*}
    \mathbb{E}\left[\sum_{t=1}^T \psi_1 \circ d_1(E_{B, \phi_B^{f^{\star}}}(x_t), y_t) \right] &=  \mathbb{E}\left[\sum_{t=1}^T \psi_1 \circ d_1(\mathcal{A}_1(x_t|L^{f^{\star}}_{B_{< t}}), y_t) \right]\\
    &= \mathbb{E}\left[\sum_{t=1}^T \psi_1 \circ d_1(\mathcal{A}_1(x_t|L^{f^{\star}}_{B_{< t}}), y_t)\frac{\mathbbm{P}\left[B_t = 1 \right]}{\mathbbm{P}\left[B_t = 1 \right]} \right]\\
    &= \frac{T}{T^{\beta}}\sum_{t=1}^T \mathbb{E}\left[\psi_1 \circ d_1(\mathcal{A}_1(x_t|L^{f^{\star}}_{B_{< t}}), y_t)\mathbbm{P}\left[B_t = 1 \right] \right]\\
    &= \frac{T}{T^{\beta}}\sum_{t=1}^T \mathbb{E}\left[\psi_1 \circ d_1(\mathcal{A}_1(x_t|L^{f^{\star}}_{B_{< t}}), y_t)\mathbbm{1}\{B_t = 1 \}\right].
\end{align*}

To see the last equality, note that the prediction $\mathcal{A}(x_t|L_{B_{< t}}^{f^{\star}})$ (and therefore $\mathcal{A}_1(x_t|L_{B_{< t}}^{f^{\star}})$) only depends on bitstring ($B_1, \ldots, B_{t-1}$) and the internal randomness of $A$, both of which are independent of $B_t$. Thus, we have  

\begin{align*}
    \mathbb{E}\left[\psi_1 \circ d_1(\mathcal{A}_1(x_t|L^{f^{\star}}_{B_{< t}}), y_t)\mathbbm{1}\{B_t = 1 \}\right]  &= \mathbb{E}\left[\psi_1 \circ d_1(\mathcal{A}_1(x_t|L^{f^{\star}}_{B_{< t}}), y_t) \right] \, \mathbb{E}\left[ \indicator \{B_t=1\}\right]\\
    &= \mathbb{E}\left[\psi_1 \circ d_1(\mathcal{A}_1(x_t|L^{f^{\star}}_{B_{< t}}), y_t)\right] \mathbbm{P}[B_t =1]
\end{align*}

% To see the last equality,  observe that $E_{B, \phi_B^{f^{\star}}}$ updates $\Acal^{\star}$  by passing $(x_t, (y_t, f^{\star, \alpha}_{2:K}(x_t))$ whenever $B_t = 1$. However, the prediction of $\Acal^{\star}$ on round $t$ only depends on bitstring ($B_1, \ldots, B_{t-1}$), but is independent of $B_t$. Thus, we have   
%
% \begin{align*}
%     \mathbb{E}\left[\psi_1 \circ d_1(\mathcal{A}^{\star}_1(x_t), y_t)\mathbbm{1}\{B_t = 1 \}\right]  &= \mathbb{E}\left[\psi_1 \circ d_1(\mathcal{A}^{\star}_1(x_t), y_t)\mathbbm{1}\{B_t = 1 \} \right] \, \mathbb{E}\left[ \indicator \{B_t=1\}\right]\\
%     &= \mathbb{E}\left[\psi_1 \circ d_1(\mathcal{A}^{\star}_1(x_t), y_t)\mathbbm{1}\{B_t = 1 \} \right] \mathbbm{P}[B_t =1]
% \end{align*}
%
as needed. Continuing onwards,

\begin{align*}
    \mathbb{E}\left[\sum_{t=1}^T \psi_1 \circ d_1(E_{B, \phi_B^{f^{\star}}}(x_t), y_t) \right] &= \frac{T}{T^{\beta}}\mathbb{E}\left[\sum_{t=1}^T \psi_1 \circ d_1(\mathcal{A}_1(x_t|L^{f^{\star}}_{B_{< t}}), y_t)\mathbbm{1}\{B_t = 1 \}\right]\\
    &= \frac{T}{T^{\beta}}\mathbb{E}\left[\sum_{t: B_t = 1} \psi_1 \circ d_1(\mathcal{A}_1(x_t|L^{f^{\star}}_{B_{< t}}), y_t)\right]\\
    % &= \frac{T}{T^{\beta}}\mathbb{E}_B\left[\mathbb{E}_A\left[\sum_{t: B_t = 1} \psi_1 \circ d_1(\mathcal{A}^{\star}_1(x_t), y_t)\right] \right]\\
    &\leq \frac{T}{T^{\beta}}\mathbb{E}\left[\sum_{t: B_t = 1} \ell(\mathcal{A}(x_t|L^{f^{\star}}_{B_{< t}}), (y_t, f_{2:K}^{\star, \alpha}(x_t)))\right] \\
     &= \frac{T}{T^{\beta}}\mathbb{E}\left[\mathbb{E}\left[\sum_{t: B_t = 1} \ell(\mathcal{A}(x_t|L^{f^{\star}}_{B_{< t}}), (y_t, f_{2:K}^{\star, \alpha}(x_t))) \bigg|B\right]\right]\\
    &\leq \frac{T}{T^{\beta}}\mathbb{E}\left[\sum_{t: B_t = 1} \ell(f^{\star}(x_t), (y_t, f_{2:K}^{\star, \alpha}(x_t))) + R_{\mathcal{A}}(|B|) \right] \\
    &= \frac{T}{T^{\beta}}\mathbb{E}\left[\sum_{t: B_t = 1} \ell(f^{\star}(x_t), (y_t, f_{2:K}^{\star, \alpha}(x_t)))\right] + \frac{T}{T^{\beta}}\mathbb{E}\left[ R_{\mathcal{A}}(|B|) \right]\\
\end{align*}

The first inequality follows from the definition of $\ell$. The second inequality follows from the fact that $\mathcal{A}$ is an online learner for $\ell$ with regret bound $R_{\mathcal{A}}(T)$ and is updated on the stream labeled by $f_{2:K}^{\star, \alpha}$ only when $B_t = 1$. Now, we can upperbound the first term as follows: 

\begin{align*}
\frac{T}{T^{\beta}}\mathbb{E}\left[\sum_{t: B_t = 1} \ell(f^{\star}(x_t), (y_t, f_{2:K}^{\star, \alpha}(x_t)))\right] &= \frac{T}{T^{\beta}}\mathbb{E}\left[\sum_{t: B_t = 1} \left(\psi_1 \circ d_1(f^{\star}_1(x_t), y_t) + \sum_{k=2}^K \psi_k \circ d_1(f_k^{\star}(x_t), f^{\star, \alpha}_k(x_t)) \right) \right]\\
&\leq \frac{T}{T^{\beta}}\mathbb{E}\left[\sum_{t: B_t = 1} \psi_1 \circ d_1(f^{\star}_1(x_t), y_t) + \sum_{t: B_t=1}KL\alpha \right] \\
&\leq \frac{T}{T^{\beta}}\mathbb{E}\left[\sum_{t= 1}^T \psi_1 \circ d_1(f^{\star}_1(x_t), y_t)\mathbbm{1}\{B_t = 1\}\right] + \frac{T}{T^{\beta}}\mathbb{E}\left[|B|KL\alpha \right] \\
&= \frac{T}{T^{\beta}}\sum_{t= 1}^T \psi_1 \circ d_1(f^{\star}_1(x_t), y_t)\frac{T^\beta}{T}+ \frac{T}{T^{\beta}}T^{\beta}KL\alpha \\
&= \sum_{t= 1}^T \psi_1 \circ d_1(f^{\star}_1(x_t), y_t)+ KTL\alpha.\\ 
\end{align*}

The first inequality follows from the fact that $\psi_k$ is $L$-Lipschitz and $d_1(f_k^{\star}(x_t), f^{\star, \alpha}_k(x_t)) \leq \alpha$.  Putting things together, we find that, 

\begin{align*}
    \mathbb{E}\Bigg[\sum_{t=1}^T &\psi_1 \circ d_1(\mathcal{Q}(x_t), y_t) \Bigg] \\
    &\leq \mathbb{E}\left[\sum_{t=1}^T \psi_1 \circ d_1(E_{B, \phi_B^{f^{\star}}}(x_t), y_t) \right] + \mathbb{E}\left[\sqrt{2T\ln(|\mathcal{E}_B|)}\right]\\
    &\leq  \sum_{t= 1}^T \psi_1 \circ d_1(f^{\star}_1(x_t), y_t)+ KTL\alpha + \frac{T}{T^{\beta}}\mathbb{E}\left[ R_{\mathcal{A}}(|B|) \right] + \mathbb{E}\left[\sqrt{2T\ln(|\mathcal{E}_B|)}\right]\\
    &\leq \inf_{f_1 \in \mathcal{F}_1}\sum_{t= 1}^T \psi_1 \circ d_1(f_1(x_t), y_t)+ KTL\alpha + \frac{T}{T^{\beta}}\mathbb{E}\left[ R_{\mathcal{A}}(|B|) \right] + \mathbb{E}\left[\sqrt{2TK|B|\ln(\frac{2}{\alpha})}\right].\\
\end{align*}
where the last inequality follows from the fact that that $|\mathcal{E}_B| \leq (\frac{2}{\alpha})^{K|B|}$ and the definition of $f^{\star}$. By Jensen's inequality, we further get that, $\mathbb{E}\left[\sqrt{2TK|B|\ln(\frac{2}{\alpha})}\right] \leq \sqrt{2T^{\beta + 1}K\ln(\frac{2}{\alpha})}$, which implies that 

$$\mathbb{E}\left[\sum_{t=1}^T \psi_1 \circ d_1(\mathcal{Q}(x_t), y_t) \right]  \leq \inf_{f_1 \in \mathcal{F}_1}\sum_{t= 1}^T \psi_1 \circ d_1(f_1(x_t), y_t)+ KTL\alpha + \frac{T}{T^{\beta}}\mathbb{E}\left[ R_{\mathcal{A}}(|B|) \right] + \sqrt{2T^{\beta + 1}K\ln(\frac{2}{\alpha})}.$$

Next,  by Lemma \ref{lem:woess}, there exists a concave sublinear function $ \overline{R}_{\Acal}(|B|)$ of $|B|$ that upperbounds $ R_{\Acal}(|B|)$. By Jensen's inequality, we obtain $ \expect[\overline{R}_{\Acal}(|B|)] \leq \overline{R}_{\Acal}(T^{\beta})$, which yields

$$\mathbb{E}\left[\sum_{t=1}^T \psi_1 \circ d_1(\mathcal{Q}(x_t), y_t) \right]  \leq \inf_{f_1 \in \mathcal{F}_1}\sum_{t= 1}^T \psi_1 \circ d_1(f_1(x_t), y_t)+ KTL\alpha + \frac{T}{T^{\beta}}\overline{R}_{\mathcal{A}}(T^{\beta}) + \sqrt{2T^{\beta + 1}K\ln(\frac{2}{\alpha})}.$$

Picking $\alpha = \frac{1}{KTL}$ and $\beta \in (0, 1)$, gives that $\mathcal{Q}$ enjoys sublinear expected regret:

$$\mathbb{E}\left[\sum_{t=1}^T \psi_1 \circ d_1(\mathcal{Q}(x_t), y_t) \right]  - \inf_{f_1 \in \mathcal{F}_1}\sum_{t= 1}^T \psi_1 \circ d_1(f_1(x_t), y_t) \leq  1 + \frac{T}{T^{\beta}}\overline{R}_{\mathcal{A}}(T^{\beta}) + \sqrt{4T^{\beta + 1}K\ln(KTL)}.$$

This completes the proof as we have shown that $\mathcal{Q}$ is an online learner for $\mathcal{F}_1$ with respect to $\psi_1 \circ d_1$.

\end{proof}

\subsubsection{A More General Characterization of Learnability for Decomposable Losses}

Theorem \ref{thm:onlinereg_decomp} characterizes the learnability of multioutput function classes $\mathcal{F}$ with respect to decomposable loss functions $\ell$ in terms of the learnability of $\mathcal{F}_k$ with respect to $\psi_k \circ d_1$. Similar to the batch setting, we can remove $\psi_k$, and characterize the learnability of $\mathcal{F}$ with respect to $\ell$ in terms of the learnability of $\mathcal{F}_k$'s with respect to $d_1$. However, to do so, we need to an place additional assumption on the decomposable loss function $\ell$. 
Theorem \ref{thm:onlinereg_ell2d1} below summarizes the main result of this section. 

\begin{theorem} \label{thm:onlinereg_ell2d1}
    Let $\ell$ be any decomposable loss function satisfying Assumptions \ref{assumption1} and \ref{assumption2}. A multioutput function class $\Fcal \subseteq \Ycal^{\Xcal}$ is online learnable with respect to  $\ell$ if and only if each $\Fcal_k \subseteq \Ycal_k^{\Xcal}$ is online learnable with respect to $d_1$.
\end{theorem}

The main tool needed to prove Theorem \ref{thm:onlinereg_ell2d1} is Lemma \ref{lem:onlinereg_scalar}, which relates the online learnability of a scalar-output function class $\mathcal{H} \subset [0, 1]^{\mathcal{X}}$ with respect to $\psi \circ d_1$ to its online learnability with respect to $d_1$, where $\psi$ is any monotonic, Lipschitz function such that $\psi(0) = 0$. The proof of Lemma \ref{lem:onlinereg_scalar} can be found in Appendix \ref{app:onlinereg_scalar}.

\begin{lemma} \label{lem:onlinereg_scalar}
    Let $\psi: \mathbb{R}_{\geq 0} \rightarrow \mathbb{R}_{\geq 0}$ be any monotonic and Lipschitz function such that $\psi(0) = 0$. A scalar-valued function class $\Gcal \subset [0,1]^{\Xcal}$ is online learnable with respect to $\psi \circ d_1$, if and only if $\mathcal{G}$ is online learnable with respect to $d_1$.
\end{lemma}

 Note that for every $1 \leq p < \infty$ and $x \geq 0$, $\psi(x) = x^p$ is a monotonic increasing, Lipschitz function. Therefore, Lemma \ref{lem:onlinereg_scalar} shows that online learnability with respect to $d_p$ is equivalent to online learnability with respect to $d_1$.  Combining Assumption \ref{assumption2}, Theorem \ref{thm:onlinereg_decomp}, and Lemma \ref{lem:onlinereg_scalar} immediately gives Theorem \ref{thm:onlinereg_ell2d1}. Since the Sequential Fat Shattering dimension characterizes online learnability with respect to the absolute loss, Theorem \ref{thm:onlinereg_ell2d1} further implies that for any decomposable loss satisfying Assumptions \ref{assumption1} and \ref{assumption2}, the finiteness of $\text{fat}_{\gamma}^{\text{seq}}(\mathcal{F}_k)$ for all $k \in [K]$ and $\gamma > 0$ is a sufficient and necessary condition for online multioutput learnability.

\subsubsection{Characterizing Learnability of Non-Decomposable Losses}

In this section, we characterize the online learnability of multioutput function classes $\mathcal{F}$ for a natural family of non-decomposable losses, $\ell_p$ norms for $1 \leq p \leq \infty$. We prove an analogous theorem to Theorem \ref{thm:batch_reg_nondecom}, by relating the online learnability of $\mathcal{F}$ with respect to $\ell_p$ to the online learnability of each $\mathcal{F}_k$ with respect to $d_1$. We note that the proof of only uses the fact that $\ell_p$ norms are equivalent (up to a $K$ dependent constant) to the $\ell_1$ norm. Since any two norms in a finite dimensional space are equivalent, Theorem \ref{thm:onlinereg_ellp} actually holds true for \textit{any} norm in $\mathbb{R}^K$. But we only consider $\ell_p$ norms here due to their practical importance.

\begin{theorem} \label{thm:onlinereg_ellp}
    Let $1 \leq p \leq \infty$. A function class $\Fcal \subseteq \Ycal^{\Xcal}$ is online  learnable with respect to  $\ell_p$ if and only if each $\Fcal_k \subseteq \Ycal_k^{\Xcal}$ is online  learnable with respect to $d_1$.
\end{theorem}

 By Theorem \ref{thm:onlinereg_decomp}, $\mathcal{F}$ is online learnable with respect to $\ell_1$ if and only if each restriction $\mathcal{F}_k$ is online learnable with respect to $d_1$. Thus, to prove Theorem \ref{thm:onlinereg_ellp} it suffices to show that $\mathcal{F}$ is online learnable with respect to $\ell_p$ if and only if $\mathcal{F}$ is online learnable with respect to $\ell_1$ for $ p > 1$.  At a high-level, the proof  Theorem \ref{thm:onlinereg_ellp} follows a similar route as the proof of Theorem \ref{thm:batch_reg_nondecom}: convert an agnostic learner for $\Fcal$
into a realizable learner for $\Fcal^{\alpha}$ and then use realizable-to-agnostic conversion for $\Fcal^{\alpha}$. 
\begin{proof} Fix $p > 1$. By the argument above, it suffices to show that $\mathcal{F}$ is online learnable with respect to $\ell_p$ if and only if $\mathcal{F}$ is online learnable with respect to $\ell_1$. We begin by proving sufficiency - if $\mathcal{F}$ is online learnable with respect to $\ell_1$ then $\mathcal{F}$ is online learnable with respect to $\ell_p$.

Let $\mathcal{A}$ be an online learner for $\mathcal{F}$ with respect to $\ell_1$. Our goal is to construct an online learner $\mathcal{Q}$ for $\mathcal{F}$ with respect to $\ell_p$. We assume an oblivious adversary, and therefore the stream of points to be observed by the online learner $\mathcal{Q}$, denoted $(x_1, y_1), ..., (x_T, y_T) \in (\mathcal{X} \times [0, 1]^K)^T$, is fixed beforehand. Let $f^{\star} = \argmin_{f \in \mathcal{F}} \sum_{t=1}^T \ell_p(f(x_t), y_t)$ also denote the optimal function in hindsight with respect to the $\ell_p$ loss. 

Our strategy follows three steps. First, we show that $\mathcal{A}$ is a realizable online learner for $\mathcal{F}^{\alpha}$ with respect to $\ell_p$. Then, since $|\text{im}(\mathcal{F}^{\alpha})| \leq (\frac{2}{\alpha})^K < \infty$ is finite and $\ell_p$ is a $1$- subadditive (by triangle inequality), Theorem \ref{thm:onlreal2agn} allows to convert the realizable online learner $\mathcal{A}$ for $\mathcal{F}^{\alpha}$ with respect to $\ell_p$ into an agnostic online learner $\mathcal{Q}$ for $\mathcal{F}^{\alpha}$ with respect to $\ell_p$. Finally, for an appropriately selected discretization parameter $\alpha$, we show that $\mathcal{Q}$ is also an agnostic online for  $\mathcal{F}$ with respect to $\ell_p$, which completes the proof.  To that end, let $(x_1, f^{\star, \alpha}(x_1)), ..., (x_{T}, f^{\star, \alpha}(x_{T})$ denote a (realizable) sequence of instances labeled by some function $f^{\star, \alpha} \in \mathcal{F}^{\alpha}$. Since $||\cdot||_q \leq  || \cdot ||_p$ for $q > p$, we have that 

$$\mathbb{E}\left[\sum_{t=1}^{T} \ell_p(\mathcal{A}(x_t),f^{\star, \alpha}(x_t)) \right] \leq \mathbb{E}\left[\sum_{t=1}^{T} \ell_1(\mathcal{A}(x_t), f^{\star, \alpha}(x_t)) \right]$$

Since $\mathcal{A}$ is an online learner for $\mathcal{F}$ with respect to $\ell_1$, we get, 

$$\mathbb{E}\left[\sum_{t=1}^{T} \ell_1(\mathcal{A}(x_t), f^{\star, \alpha}(x_t)) \right] \leq \inf_{f\in \mathcal{F}}\sum_{t=1}^{T} \ell_1(f(x_t), f^{\star, \alpha}(x_t)) + R_{\mathcal{A}}(T) \leq \alpha KT + R_{\mathcal{A}}(T)$$
where $R_{\mathcal{A}}(T)$ is the regret of online learner $\mathcal{A}$.  Combining things together, we have that 

$$\mathbb{E}\left[\sum_{t=1}^{T} \ell_p(\mathcal{A}(x_t),f^{\star, \alpha}(x_t)) \right] \leq \alpha KT + R_{\mathcal{A}}(T),$$

showing that $\mathcal{A}$ is a realizable online learner for $\mathcal{F}^{\alpha}$ with respect to $\ell_p$ for a small enough $\alpha$. Now, since $|\text{im}(\mathcal{F}^{\alpha})| \leq (\frac{2}{\alpha})^K < \infty$ is a finite, $\ell_p$ is a $1$- subadditive loss function,  and $\mathcal{A}$ is a realizable online learner, for any $\beta \in (0, 1)$, Theorem \ref{thm:onlreal2agn} gives an agnostic online learner $\mathcal{Q}$ for $\mathcal{F}^{\alpha}$ with respect to $\ell_p$ with the following regret guarantee over the original stream $(x_1, y_1), ..., (x_T, y_T)$: 

$$\mathbb{E}\left[\sum_{t=1}^T \ell_p(\mathcal{Q}(x_t), y_t) \right]  \leq \inf_{f^{\alpha} \in \mathcal{F}^{\alpha}}\sum_{t= 1}^T \ell_p(f^{\alpha}(x_t), y_t) + \frac{T}{T^{\beta}}(\alpha K T^{\beta} + \overline{R}_{\mathcal{A}}(T^{\beta})) + K\sqrt{2T^{\beta + 1}K\ln(\frac{2}{\alpha})},$$
where  $\alpha K T^{\beta}+\overline{R}_{\mathcal{A}}(T^{\beta})$ is any concave sublinear upperbound of $ \alpha K T^{\beta}+R_{\mathcal{A}}(T^{\beta})$. We also use the fact that the function $T \mapsto \alpha K T^{\beta}$  is a concave sublinear function of $T$ and the sum of two concave functions is itself a concave function.
Noting that $\ell_p(f^{\alpha}(x_t), y_t) \leq \ell_p(f(x_t), y_t) + \ell_p(f^{\alpha}(x_t), f(x_t)) \leq \ell_p(f(x_t), y_t) + \alpha K$, gives 

$$\mathbb{E}\left[\sum_{t=1}^T \ell_p(\mathcal{Q}(x_t), y_t) \right]  \leq \inf_{f \in \mathcal{F}}\sum_{t= 1}^T \ell_p(f(x_t), y_t) + \alpha K T + \frac{T}{T^{\beta}}(\alpha K T^{\beta} + \overline{R}_{\mathcal{A}}(T^{\beta})) + K\sqrt{2T^{\beta + 1}K\ln(\frac{2}{\alpha})}.$$

Combining like terms together, we have 

$$\mathbb{E}\left[\sum_{t=1}^T \ell_p(\mathcal{Q}(x_t), y_t) \right]  \leq \inf_{f \in \mathcal{F}}\sum_{t= 1}^T \ell_p(f(x_t), y_t) + 2\alpha K T + \frac{T}{T^{\beta}} \overline{R}_{\mathcal{A}}(T^{\beta}) + K\sqrt{2T^{\beta + 1}K\ln(\frac{2}{\alpha})}.$$

Finally, picking $\alpha = \frac{1}{2KT}$ gives that 

$$\mathbb{E}\left[\sum_{t=1}^T \ell_p(\mathcal{Q}(x_t), y_t) \right] - \inf_{f \in \mathcal{F}}\sum_{t= 1}^T \ell_p(f(x_t), y_t)  \leq  1 + \frac{T}{T^{\beta}} \overline{R}_{\mathcal{A}}(T^{\beta}) + K\sqrt{2T^{\beta + 1}K\ln(4KT)}.$$

Since $\overline{R}_{\mathcal{A}}(T^{\beta})$ is sublinear in $T^{\beta}$ and $\beta \in (0, 1)$, $\mathcal{Q}$ enjoys sublinear expected regret. Thus, we have shown that $\mathcal{Q}$ is also an agnostic online learner for $\mathcal{F}$ with respect to $\ell_p$. 

The reverse direction follows identically and uses the fact that for any $p > 1$, $||\cdot||_p \leq ||\cdot||_1 \leq K||\cdot||_p$. In particular, using the exact same argument, we can show that if $\mathcal{A}$ is an online learner for $\mathcal{F}$ with respect to $\ell_p$, then $\mathcal{A}$ is also a realizable online learner for $\mathcal{F}^{\alpha}$ with respect to $\ell_1$ with expected regret bound: 

$$\mathbb{E}\left[\sum_{t=1}^{T} \ell_1(\mathcal{A}(x_t),f^{\star, \alpha}(x_t)) \right] \leq K\mathbb{E}\left[\sum_{t=1}^{T} \ell_p(\mathcal{A}(x_t),f^{\star, \alpha}(x_t)) \right] \leq \alpha K^2 T + KR_{\mathcal{A}}(T).$$

Using $\mathcal{A}$ as the realizable learner in Theorem \ref{thm:onlreal2agn}, for any $\beta \in (0, 1)$,  picking $\alpha = \frac{1}{(K + K^2)T}$ gives a regret bound:

$$\mathbb{E}\left[\sum_{t=1}^T \ell_1(\mathcal{Q}(x_t), y_t) \right] - \inf_{f \in \mathcal{F}}\sum_{t= 1}^T \ell_1(f(x_t), y_t)  \leq  1 + \frac{KT}{T^{\beta}} \overline{R}_{\mathcal{A}}(T^{\beta}) + K\sqrt{2T^{\beta + 1}K\ln(4K^2T)},$$
where $\overline{R}_{\mathcal{A}}(T^{\beta})$ is any concave sublinear upperbound of $ R_{\mathcal{A}}(T^{\beta})$. Since $\beta \in (0,1)$, $\mathcal{Q}$ is an online learner for $\mathcal{F}$ with respect to $\ell_1$ as needed. This completes the proof of Theorem \ref{thm:onlinereg_ellp}.
\end{proof}

\noindent \textbf{Remark.} As with decomposable losses, Theorem \ref{thm:onlinereg_ellp} also implies that for any $\ell_p$ norm loss, the finiteness of $\text{fat}_{\gamma}^{\text{seq}}(\mathcal{F}_k)$, for all $k \in [K]$ and fixed $\gamma > 0$, is a sufficient and necessary condition for online multioutput learnability.

\section{Discussion}
In this work, we give a characterization of multioutput learnability in four settings: batch classification, online classification, batch regression, and online regression. In all four settings, we show that a multioutput function class is learnable if and only if each restriction is learnable. All of our bounds in this paper scale with $K$, preventing our current techniques from extending to the case when $K$ is infinite. Accordingly, we pose it as an open question to characterize multioutput learnability when $K$ is infinite (e.g. function-space valued regression). Furthermore, we also leave it open to find combinatorial dimensions that provide tight quantitative characterizations of batch and online multioutput learnability.

\section*{Acknowledgements}

We acknowledge the support of NSF via grants IIS-2007055 (AT) and DMS-2413089 (AT, US). VR acknowledges the support of the NSF Graduate Research Fellowship.

\bibliography{sample}

\newpage 
\appendix

\section{Complexity Measures}

\subsection{Complexity Measures for Batch Learning}
\label{appdx:complexity-batch}
In binary classification, the Vapnik-Chervonenkis (VC) dimension of a function class characterizes its learnability. 
\begin{definition}[Vapnik-Chervonenkis Dimension]\label{vc}
A set $S = \{x_1, \ldots, x_d\}$ is shattered by a binary function class $\Hcal \subseteq \{-1,1\}^d$ if for every $\sigma \in \{-1,1\}^d$, there exists a hypothesis $h_{\sigma} \in \Hcal$ such that for all $i \in [d]$, we have $h_{\sigma}(x_i) = \sigma_i$. The VC dimension of $\Hcal$, denoted $\text{VC}(\Hcal)$, is the size of the largest shattered set $S \subseteq \Xcal$. If the size of the shattered set can be arbitrarily large, we say that $\text{VC}(\Hcal) = \infty$. 
\end{definition}
The learnability of a multiclass function class is characterized by its Natarajan dimension. 
 \begin{definition}[Natarajan Dimension]\label{ndim}
A set $S = \{x_1, \ldots, x_d\}$ is shattered by a multiclass function class $\Hcal \subset \Ycal^{\Xcal}$ if there exist two witness functions $f, g: S \to \Ycal$ such that  $f(x_i) \neq  g(x_i) $ for all $i \in [d]$, and for every $\sigma \in \{-1,1\}^d$, there exists a function $h_{\sigma} \in \Hcal$ such that for all $i \in [d]$, we have
\[h_{\sigma}(x_i) = \begin{cases} f(x_i) \, \quad \text{ if } \sigma_i =1 \\ g(x_i) \, \quad \text{ if } \sigma_i =-1  \end{cases}.\]

The Natarajan dimension of $\Hcal$, denoted $\text{Ndim}(\Hcal)$,  is the size of the largest shattered set $S \subseteq\Xcal$. If the size of the shattered set can be arbitrarily large, we say that $\text{Ndim}(\Hcal) = \infty$.
    
\end{definition}
For real-valued regression problems, the learnability is characterized in terms of the fat-shattering dimension of a function class.
\begin{definition}[Fat-Shattering Dimension]
A real-valued function class $\Gcal \subseteq [0,1]^{\Xcal}$ shatters points $S=\{x_1, x_2, \ldots, x_d\}$ at scale $1>\gamma>0$, if there exists witness functions $r:S \to [0,1]$ such that, for every $\sigma \in \{\pm 1\}^{d}$, there exists $g_{\sigma} \in \Gcal$ such that $\forall i \in [d]$, $\sigma_i (g_{\sigma}(x_i) - r(x_i)) \geq \gamma$. The fat-shattering dimension of $\Gcal$ at scale $\gamma$, denoted $\text{fat}_{\gamma}(\Gcal) $, is the size of the largest set that can be $\gamma$-shattered by $\Gcal$. If the size of the shattered set can be arbitrarily large, then we say that $\text{fat}_{\gamma}(\Gcal) = \infty$. 
\end{definition}

We also define a general notion of complexity called Rademacher complexity that provides a sufficient and necessary condition of uniform convergence. Since uniform convergence implies  learnability, we use Rademacher complexity to argue sufficient conditions for learnability. 
\begin{definition}[Empirical Rademacher Complexity]
Let $\mathcal{D}$ be a distribution over $\mathcal{X} \times \mathcal{Y}$. For a bounded loss function $\ell$, define the loss class to be $\ell \circ \Fcal = \{(x, y) \mapsto \ell(f(x), y)\}$.   If $S = \{(x_1, y_1), ..., (x_n, y_n)\}$ be a set of i.i.d samples drawn from $\mathcal{D}$, then the empirical Rademacher complexity of $\mathcal{\ell \circ \Fcal}$ is defined as 
$$\mathfrak{R}_n(\mathcal{\ell \circ \Fcal}) = \mathbb{E}_{\sigma}\left[\sup_{f \in \mathcal{F}} \left(\frac{1}{n}\sum_{i=1}^n \sigma_i \ell(f(x_i), y_i) \right)\right],$$
where $\sigma \in \{\pm 1\}^n$ is a sequence of $n$ i.i.d. Rademacher random variables. 
\end{definition}

\subsection{Complexity Measures for Online Learning}
\label{app:complexonline}
For the complexity measures below, it is useful to define a $\mathcal{Z}$-valued binary tree \citep{rakhlin2015sequential}. A binary tree $\mathcal{T}$ of depth $d$ is $\mathcal{Z}$-valued if each of its internal nodes are labelled by elements of $\mathcal{Z}$. Such a tree can be identified by a sequence $(\mathcal{T}_1, ..., \mathcal{T}_d)$ labelling functions $\mathcal{T}_i: \{\pm1\}^{i-1} \rightarrow \mathcal{Z}$ which provide labels for each internal node. A path of length $d$ is given by a sequence $\sigma = (\sigma_1, ..., \sigma_d) \in \{\pm1\}^d$. Then, $\mathcal{T}_i(\sigma_1, ..., \sigma_{i-1})$ gives the label of node following the path $(\sigma_1, ..., \sigma_{i-1})$ starting from the root, going ``right" if $\sigma_j = +1$ and ``left" if $\sigma_j = -1$. Note that,  $\mathcal{T}_1 \in \mathcal{Z}$ is the label for the root node. For brevity, we slightly abuse notation by letting $\mathcal{T}_i(\sigma_1, ..., \sigma_{i-1}) = \mathcal{T}_i(\sigma_{<i})$, but it is understood that $\mathcal{T}_i$ only depends on the prefix $(\sigma_1, ..., \sigma_{i-1})$. We are now ready to formally define complexity measures in the online setting. 

When $\mathcal{Y} = \{-1, +1\}$ is binary, the Littlestone Dimension (\cite{Littlestone1987LearningQW}) tightly characterizes the online learnability of a function class $\mathcal{H} \subseteq\mathcal{Y}^{\mathcal{X}}$ with respect to the 0-1 loss.

\begin{definition}[Littlestone Dimension]
   Let $\Tcal$ denote a complete binary tree of depth $d$ whose internal nodes are labeled by elements  $ \Xcal$ and two edges from parent to child nodes are labeled by $-1$ and $+1$. The tree is shattered by a binary hypothesis class $\mathcal{G} \subseteq \{-1, +1\}^{\Xcal}$ if for every $\sigma \in \{-1, +1\}^d$, there exists a hypothesis $g_{\sigma} \in \mathcal{G}$ such that the root to leaf path $(x_1, \ldots, x_d)$ obtained by taking left when $\sigma_t =-1$ and right when $\sigma_t =+1$ satisfies $g_{\sigma}(x_t) = \sigma_t$ for all $ 1\leq t \leq d$.  The Littlestone Dimension of $\mathcal{G}$, denoted $\text{Ldim}(\mathcal{G})$, is the maximal depth of the complete binary tree shattered by $\mathcal{G}$. If $\mathcal{G}$ can shatter a tree of arbitrary depth, we say that $\text{Ldim}(\mathcal{G}) = \infty$.      
\end{definition}
  
For finite label spaces $\mathcal{Y}$, the Multiclass Littlestone Dimension \citep{DanielyERMprinciple} tightly characterizes the online learnability of a function class $\mathcal{H} \subseteq\mathcal{Y}^{\mathcal{X}}$ with respect to the 0-1 loss. 

\begin{definition} [Multiclass Littlestone Dimension] Let $\mathcal{T}$ denote a $\mathcal{X}$-valued binary tree of depth $d$  whose edges are labelled by elements from $\mathcal{Y}$, such that the edges from a single parent to its child-nodes are each labeled with a different label. The tree $\mathcal{T}$ is shattered by a function class $\mathcal{H} \subseteq\mathcal{Y}^{\mathcal{X}}$ if, for every path $\sigma \in \{\pm1\}^d$, there is a function $h_{\sigma} \in \mathcal{H}$ such that $h_{\sigma}(\mathcal{T}_i(\sigma_{<i})) = y(\sigma_i)$, where $y(\sigma_i)$ is the label of the edge between nodes $(\mathcal{T}_i(\sigma_{<i}), \mathcal{T}_{i+1}(\sigma_{<i+1}))$. The Multiclass Littlestone Dimension (MCLdim) of $\mathcal{H}$, denoted $\text{MCLdim}(\mathcal{H})$, is the maximal depth of a complete binary tree that is shattered by $\mathcal{H}$. If $\text{MCLdim} = \infty$, then there exists shattered trees of arbitrarily large depth. 
\end{definition}

When $\mathcal{Y}$ is a bounded subset of $\mathbb{R}$, the sequential fat-shattering dimension \citep{rakhlin2015online} at scale $\gamma$ characterizes the learnability of $\mathcal{H} \subseteq[0, 1]^{\mathcal{X}}$ with respect to to the absolute loss $d_1$. 

\begin{definition} [Sequential Fat-Shattering Dimension] \label{def:seqfatshat}

Let $\mathcal{T}$ denote a $\mathcal{X}$-valued binary tree of depth $d$. The tree $\mathcal{T}$ is $\gamma$-shattered by a function class $\mathcal{H} \subseteq [0, 1]^{\mathcal{X}}$ if there exists an $\mathbb{R}$-valued binary tree $\mathcal{R}$ of depth $d$ such that for all $\sigma \in \{\pm 1\}^d$, there exists $h_{\sigma} \in \mathcal{H}$ such that for all $t \in [d]$, 

$$\sigma_t(h_{\sigma}(\mathcal{T}_t(\sigma_{<t})) - \mathcal{R}_t(\sigma_{<t})) \geq \gamma$$

The tree $\mathcal{R}$ is called the witness to shattering. The sequential fat shattering dimension of $\mathcal{H}$ at scale $\gamma$, denoted $\text{fat}^{\text{seq}}_{\gamma}(\mathcal{H})$, is the maximal depth of a complete binary tree that is $\gamma$-shattered by $\mathcal{H}$. If  there exists $\gamma$-shattered trees of arbitrarily large depth, then $\text{fat}^{\text{seq}}_{\gamma}(\mathcal{H}) = \infty$. 
\end{definition}

Beyond both finite and bounded label spaces, the sequential Rademacher complexity \citep{rakhlin2015online} provides a useful tool for giving sufficient conditions for learnability. 

\begin{definition} [Sequential Rademacher Complexity]
Let $\sigma = \{\sigma_i\}_{i=1}^T$ be a sequence of independent Rademacher random variables. Let $\mathcal{T}$ be a $\mathcal{Z}$-valued binary tree of depth $d$. The sequential Rademacher complexity of a function class $\mathcal{H} \subseteq\mathbb{R}^{\mathcal{Z}}$ on $\mathcal{T}$ is defined as
$$\mathfrak{R}_T^{\text{seq}}(\mathcal{H}; \mathcal{T}) = \mathbb{E}_{\sigma \sim \{\pm1\}^{n}}\left[\sup_{h \in \mathcal{H}} \frac{1}{T}\sum_{t=1}^T \sigma_t h(\mathcal{T}_t(\sigma_{<t})) \right].$$
Then, the worst-case sequential Rademacher complexity is defined as $\mathfrak{R}_T^{\text{seq}}(\mathcal{H}) = \sup_{\mathcal{T}}\mathfrak{R}_T^{\text{seq}}(\mathcal{H}; \mathcal{T})$.
\end{definition}

\section{Natarajan Dimension Characterizes Batch Multilabel Learnability}
\label{app:natdimchar}

\noindent A multilabel classification problem where labels (i.e. bitstrings) in $\Ycal$  are of length $K$ can also be viewed as multiclass classification on the target space with $2^K$ labels. 
Given this observation, the Natarajan dimension of the function class $\Fcal$ continues to characterize the multilabel learnability with respect to any loss function $\ell$ satisfying the identity of indiscernibles.
\begin{theorem}[\cite{David_Bianchi}]\label{thm:natarajan}
    Let $\ell$ be any loss function satisfying the identity of indiscernibles. A function class $\Fcal \subseteq\Ycal^{\Xcal}$ is agnostic learnable with respect to $\ell$ in the batch setting if and only if $Ndim(\Fcal) < \infty$. 
\end{theorem}
\noindent The proof in \cite{David_Bianchi} involves arguments based on growth function. Here, we provide proof that uses  realizable and agnostic learnability due to \cite{hopkins22a}. 
\vspace{0.50cm}
\begin{proof}(of sufficiency)
    We first show that the finiteness of Ndim($\Fcal)$ is sufficient for learnability. Suppose Ndim($\Fcal) < \infty$. Then, we know that $\Fcal$ is agnostic learnable with respect to $0$-$1$ loss \citep{David_Bianchi}. Since the target space $\Ycal$ as well as the range space of $\Fcal$ is finite, for every loss $\ell$ satisfying the identity of indiscernibles,  there exists an $a > 0$ such that $a \, \ell(h(x), y) \leq \indicator\{h(x) \neq y\}$ for any $(x,y) \in \Xcal \times \Ycal$ and function $h \in \Ycal^{\Xcal}$.  Let $\Dcal$ be a realizable distribution to $\Fcal$ with respect to $\ell$. Since $\ell(y_1, y_2)=0$ if and only if $\indicator\{y_1 \neq y_2\}=0$, the distribution $\Dcal$ is also realizable with respect to $0$-$1$ loss. Since $\Fcal$ is learnable with respect to $0$-$1$ loss, there exists a learning algorithm $\Acal$ with the following property: for any $\epsilon, \delta > 0$, for a sufficiently large $S \sim \Dcal^n$, the algorithm outputs a predictor $h = \Acal(S)$ such that, with probability $1-\delta$ over $S \sim \Dcal^n$, we have $\expect_{\Dcal}[\indicator\{h(x) \neq y\}] \leq a \, \epsilon. $ Using the inequality stated above pointwise, the predictor $h$ also satisfies $\expect_{\Dcal}[\ell(h(x), y)] \leq \epsilon$. Therefore, $\Acal$ is also a realizable algorithm with respect to $\ell$. Since $\ell$ satisfies the identity of indiscernible, Lemma \ref{realizable_agnostic_equiv} guarantees the existence of agnostic PAC learner $\Bcal$ for $\Fcal$ with respect to $\ell$. 
%     One agnostic PAC learner $\Bcal$ is Algorithm \ref{alg:batch_real_agnos_equiv} that has sample complexity
% $m_{\Bcal}(\epsilon, \delta, K) 
%  \leq m_{\Acal}(\frac{\epsilon}{2c}, \delta/2, K) + O \left( \frac{m_{\Acal}(\frac{\epsilon}{2c}, \delta/2, K) \, K + \log{\frac{1}{\delta}}}{\epsilon^2} \right)$,
% where $m_{\Acal}(\cdot,\cdot, K)$ is the sample complexity of $\Acal$ and $c> 1$ is the subadditivity constant of $\ell$.
\end{proof}

\begin{proof} (of necessity)
    Suppose $\Fcal$ is learnable with respect to $\ell$. Since the target space is finite, there must exist a constant $b > 0$ such that $\indicator\{h(x) \neq y\}  \leq b \, \ell(h(x), y)$ for any $(x,y) \in \Xcal \times \Ycal$ and any function $h \in \Ycal^{\Xcal}$. Let $\Dcal$ be a realizable distribution with respect to $\ell$. Due to the 0 alignment property, $\Dcal$ is also realizable with respect to $0$-$1$ loss. Since $\Fcal$ is learnable with respect to $\ell$ loss, there exists a learning algorithm $\Acal$ with the following property: for any $\epsilon, \delta > 0$, for a sufficiently large $S \sim \Dcal^n$, the algorithm outputs a predictor $h = \Acal(S)$ such that, with probability $1-\delta$ over $S \sim \Dcal^n$, we have $\expect_{\Dcal}[\ell(h(x), y)] \leq b \, \epsilon. $ In particular, using the inequality stated above pointwise, we obtain $\expect_{\Dcal}[\indicator\{h(x) \neq y\}] \leq \epsilon$. Therefore, $\Fcal$ is learnable with respect to $0\text{-}1$ loss in the realizable setting. As the finiteness of the Natarajan dimension is necessary for the learnability of $\Fcal$ under the $0$-$1$ loss \citep{Natarajan1989},  we must have $\text{Ndim}(\Fcal) < \infty$.
\end{proof}

\section{Proofs for Batch Multioutput Regression}

\subsection{Proof of Sufficiency in Theorem \ref{thm:batch_reg_decom}}\label{appdx:batch_reg_decom_suff}
\begin{proof}
    We first prove that the agnostic  learnability of each $\Fcal_k$  is sufficient for the agnostic learnability of $\Fcal$. As in the classification setting, the proof here is  based on a reduction. That is, given oracle access to agnostic  learners $\Acal_k$  for each $\Fcal_k$ with respect to $\psi_k \circ d_1$ loss, we construct an agnostic  learner $\Acal$ for $\Fcal$ with respect to loss $\ell$.

 Denote $\Dcal_k$ to be the marginal distribution of $\Dcal$ restricted to $\Xcal \times \Ycal_k$. Let us use $m_{k}(\epsilon, \delta)$  to denote the sample complexity of $\Acal_k$. Then, for all $k \in [K]$, the marginal samples  $S_k = \{(x_i, y_{i}^k)\}_{i=1}^n $  with scalar-valued targets are iid samples form $\Dcal_k$. For each $k \in [K]$, define $g_k = \Acal_k(S_k)$
to be the predictor returned by algorithm $\Acal_k$ when trained on $S_k$. 
 Since $\Acal_k$ is an agnostic  learner for $\Fcal_k$, we have that for sample size $n \geq \max_{k} m_k(\frac{\epsilon}{K}, \frac{\delta}{K}) $, with probability at least $1- \delta/K$ over samples $S_k \sim \Dcal_k^n$, 
\[\expect_{\Dcal_k}[\psi_k \circ d_1(g_k(x), y^k)] \leq \inf_{f_k \in \Fcal_k} \expect_{\Dcal_k}[\psi_k \circ d_1(f_k(x), y^k)] + \frac{\epsilon}{K}.\]
Summing these risk bounds over all $k$ coordinates and union bounding over the success probabilities, we get that with probability at least $1- \delta$ over samples $S \sim \Dcal^n$, 
\[\sum_{k=1}^K \expect_{\Dcal_k}[\psi_k \circ d_1(g_k(x), y^k)] \leq \sum_{k=1}^K\inf_{f_k \in \Fcal_k} \expect_{\Dcal_k}[\psi_k \circ d_1(f_k(x), y^k)] + \epsilon.\]
Using the fact that the sum of infimums over individual coordinates is at most the overall infimum of sums followed by the linearity of expectation, we can write the expression above as
\[\expect_{\Dcal}\left[\sum_{k=1}^K \psi_k \circ d_1(g_k(x), y^k) \right] \leq  \inf_{f \in \Fcal}\expect_{\Dcal}\left[\sum_{k=1}^K \psi_k \circ d_1(f_k(x), y^k)\right] + \epsilon.\]

This shows that the learning rule that runs $\Acal_k$ on marginal samples $S_k$ and concatenates the resulting scalar-valued predictors to get a vector-valued predictor  is an agnostic  learner for $\mathcal{F}$ with respect to loss $\ell$ with sample complexity at most $\max_{k} m_k(\epsilon/K, \delta/K)$. This completes our proof of sufficiency.
\end{proof}

\subsection{Equivalence of $d_1$ and $\psi \circ d_1$ Learnability in Batch Regression }\label{appdx:equiv_d1_psi_d1}
In this section, we provide proof of Lemma \ref{lem:batch_NFLT}, which establishes the equivalence of $d_1$ and $\psi \circ d_1$ learnability in scalar-valued batch regression. 

\begin{proof}(of  Lemma \ref{lem:batch_NFLT})
To prove sufficiency first, let $\Gcal$ be agnostically learnable with respect to $d_1$. This implies that the fat-shattering dimension of $\Gcal$ is finite at every scale \citep{BARTLETT1996} and uniform convergence holds over the loss class $d_1 \circ \Gcal$. Since $\psi$ is a Lipschitz function, a simple application of Talagrand's contraction lemma on Rademacher complexity \citep{Bartlett_Mendelson} implies that uniform convergence holds over the loss class $\psi \circ d_1 \circ \Gcal$ as well. Thus, $\Gcal$ is learnable with respect to $\psi \circ d_1$ via ERM.

Next, we show that if $\Gcal \subseteq [0,1]^{\Xcal}$ is learnable with respect to $\psi \circ d_1$, then $\Gcal$ is learnable with respect to $d_1$. Since the fat-shattering dimension of $\Gcal$ characterizes $d_1$ learnability of $\Gcal$, it suffices to show that $\Gcal$ being learnable with respect to $\psi \circ d_1$ implies  $\text{fat}_{\gamma}(\Gcal) < \infty$ for every $\gamma \in (0,1)$. 

Suppose, for the sake of contradiction, $\Gcal$ is learnable with respect to $\psi \circ d_1$ but there exists a scale $\gamma \in (0,1)$ such that $\text{fat}_{\gamma}(\Gcal) = \infty$. Then, for every $d \in \naturals$, there exists $X = \{x_1, \ldots, x_d\} \subseteq \Xcal$ and a witness function $r: \Xcal \to [0,1]$ such that for every $\sigma \in \{-1,1\}^{d}$, there exists a $g_{\sigma} \in \Gcal$ such that $\sigma_i(g_{\sigma}(x_i) - r(x_i)) \geq \gamma$ for all $i \in [d]$. Define $\Gcal_X = \{g_{\sigma} \in \Gcal \mid \sigma \in \{-1,1\}^d\}$ be the set of functions that shatters $X$. Define $\Hcal = \{-1,1\}^{X}$ to be a set of all functions from $X$ to $\{-1,1\}$. By definition of $\Hcal$, we must have $\text{VC}(\Hcal) = d$. We use an agnostic learner for $\Gcal$ with respect to $\psi \circ d_1$ to construct an agnostic learner for $\Hcal$ whose sample complexity, for large enough $d$, is smaller than the known lower bound for VC classes. Since $\text{fat}_{\gamma}(\Gcal) = \infty$, $d$ can be made arbitrarily large and thus we derive a contradiction.

 Let $\Acal$ be the promised agnostic learner for $\Gcal$ with respect to $\psi \circ d_1$ with sample complexity $m(\epsilon, \delta)$. For all $f \in [0,1]^{X}$, define a threshold function $h_f : X \to \{-1, 1\} $  as $h_f(x) = 2\,\indicator\{f(x) \geq r(x)\} -1 $. Let $\Dcal$ be an arbitrary distribution on $X \times \{-1, 1\}$ and $\Dcal_X$ be its marginal on $X$.  

\begin{algorithm}[H]
\caption{Agnostic  PAC learner for $\Hcal$ }
\label{alg:binary_red}
\setcounter{AlgoLine}{0}
\KwIn{Agnostic  learner $\Acal$ for $\Gcal$, unlabeled samples $S_U \sim \Dcal_X$, and another independent labeled samples $ S_L \sim \Dcal$}

Define $S_{\text{aug}} = \{(S_U, g^{\alpha}(S_U)) \mid g \in \Gcal_X\}$, all possible augmentations of $S_U$ by $\alpha$-discretization of functions in $\Gcal_X$ for $\alpha \leq \gamma/2$.

Run $\Acal$ over all possible augmentations to get $C(S_U) := \left\{\Acal\big(S \big) 
 \mid S \in S_{\text{aug}}\right\}.$

Define  $C_{\pm 1}(S_U) = \{ h_f \mid f \in C(S_U)\}$, a thresholded class of $C(S_U)$. 

Return  the predictor in $C_{\pm 1}(S_U)$ with the lowest empirical  $0\text{-}1$ risk over $S_L$. 
\end{algorithm}

We now show that Algorithm \ref{alg:binary_red} is an agnostic learner for $\Hcal$.  Consider $d \gg S_U + S_L$.  Then, $|C_{\pm 1}(S_U)| = |S_{\text{aug}}| \leq (2/\alpha)^{ |S_U|}$ can be much smaller than $2^d$.  Let $h^{\star} := \argmin_{h \in \Hcal}\expect_{\Dcal}[\indicator\{h(x) \neq y\}]$ be the optimal hypothesis for $\Dcal$. Note that, by definition of shattering, for every $h \in \Hcal$, there exists a $g \in \Gcal_X$ such that $h(x)= h_{g}(x)$ for all $x \in X$. In particular, there must exist $g^{\star} \in \Gcal_X$ such that $h^{\star}(x) = h_{g^{\star}}(x) := 2\,\indicator\{g^{\star}(x) \geq r(x)\} -1 $ for all $x \in X$.  Let $g^{\alpha}$ denote the $\alpha$-discretization of $g$ as defined in Equation \eqref{eq:discretization} for some $\alpha \leq \gamma/2$. Now, consider a sample $(S_U, g^{\star, \alpha}(S_U)) \in S_{\text{aug}}$. Let $\hat{g} = \Acal((S_U, g^{\star, \alpha}(S_U)) )$ be the predictor returned by the algorithm when run on a sample labeled by $g^{\star, \alpha}$. Define $h_{\hat{g}} = 2\,\indicator\{\hat{g}(x) \geq r(x)\}-1$ to be its thresholded function. Then, using the triangle inequality on the indicator function, we have
\begin{equation}\label{lem7:eq1}
\begin{split}
    \expect_{\Dcal}[\indicator\{h_{\hat{g}}(x) \neq y\}] \leq \expect_{\Dcal}[\indicator\{h^{\star}(x)\neq y \}] + \expect_{\Dcal_X}[\indicator\{h_{\hat{g}}(x) \neq h^{\star}(x)\} ].
\end{split}
\end{equation}
Note that $\indicator\{h_{\hat{g}}(x) \neq h^{\star}(x)\} =\indicator\{h_{\hat{g}}(x) \neq h_{g^{\star}}(x)\} \leq \indicator\{|\hat{g}(x) - g^{\star}(x)| \geq \gamma\} $. To see why the last inequality is true, we only have to consider the case where the indicator on the left is $1$, otherwise, the inequality is trivial. Recall that $\indicator\{h_{\hat{g}}(x) \neq h_{g^{\star}}(x)\} =1$ whenever $ \hat{g}(x)$ and $g^{\star}(x)$ lie on the opposite side of witness $r(x)$. Since $g^{\star}$ has to be at least $\gamma$ away from the witness $r(x)$, we obtain $\indicator\{|\hat{g}(x) - g^{\star}(x)| \geq \gamma\} =1$. Next, using the fact that $\alpha \leq \gamma/2$, we have $\indicator\{|\hat{g}(x) - g^{\star}(x)| \geq \gamma\} \leq \indicator\{|\hat{g}(x) - g^{\star, \alpha}(x)| \geq \gamma/2\}$ because discretization can decrease the distance between these functions by at most $\gamma/2$. Furthermore, using monotonicity of $\psi$, we get $ \indicator\{|\hat{g}(x) - g^{\star, \alpha}(x)| \geq \gamma/2\}  \leq \indicator\{\psi(|\hat{g}(x) - g^{\star, \alpha}(x)|) \geq \psi(\gamma/2)\}$ . Combining everything, we get a pointwise inequality 
\[ \indicator\{h_{\hat{g}}(x) \neq h^{\star}(x)\} \leq \indicator\{\psi(|\hat{g}(x) - g^{\star, \alpha}(x)|) \geq \psi(\gamma/2)\} \leq \frac{1}{\psi(\gamma/2)}\psi(|\hat{g}(x) - g^{\star, \alpha}(x)|).\]
Using this inequality gives an upperbound on the risk of $h_{\hat{g}}$, namely
\begin{equation}\label{lem:eq2}
    \expect_{\Dcal_X}[\indicator\{h_{\hat{g}}(x) \neq h^{\star}(x)\} ] \leq \frac{1}{\psi(\gamma/2)} \expect_{\Dcal}[\psi(|\hat{g}(x) - g^{\star, \alpha}(x)|) ].
\end{equation}

Since $\hat{g} = \Acal((S_U, g^{\star, \alpha}(S_U)) )$, we can use the algorithm's guarantee to get a further upperbound on the expectation above. In particular, if $|S_U|\geq  m( \frac{\epsilon \psi(\gamma/2)}{4}, \delta/2)$, then with probability at least $1-\delta/2$ over sampling $S_U \sim \Dcal_X$, we have
\[\expect_{\Dcal_X}[\psi(|\hat{g}(x) - g^{\star, \alpha}(x)|) ] \leq \inf_{g \in \Gcal} \expect_{\Dcal_X}[\psi(|g(x) - g^{\star, \alpha}(x)|)] + \frac{\epsilon \psi(\gamma/2)}{4}.\]
Note that $\inf_{g \in \Gcal} \expect_{\Dcal_X}[\psi(|g(x) - g^{\star, \alpha}(x)|)] \leq \expect_{\Dcal_X}[\psi(|g^{\star}(x) - g^{\star, \alpha}(x)|)] \leq \psi(\alpha)$, where the last step uses the fact that $|g^{\star}(x) - g^{\star, \alpha}(x)| \leq \alpha$ and  $\psi$ is monotonic. Using $L$-Lipschitzness of $\psi$ and the fact that $\psi(0) =0$, we get $\psi(\alpha) \leq L \alpha$. Picking $\alpha = \min(\gamma/2,\frac{\epsilon \psi(\gamma/2)}{4L}  )$, we get $\inf_{g \in \Gcal} \expect_{\Dcal_X}[\psi(|g(x) - g^{\star, \alpha}(x)|)] \leq \frac{\epsilon \psi(\gamma/2)}{4}.$ Plugging this back to the inequality in the display above, we get to $\expect_{\Dcal_X}[\psi(|\hat{g}(x) - g^{\star, \alpha}(x)|) ] \leq \frac{\epsilon \psi(\gamma/2)}{2}$. Using this guarantee on \eqref{lem:eq2}, we obtain
\[ \expect_{\Dcal_X}[\indicator\{h_{\hat{g}}(x) \neq h^{\star}(x)\} ] \leq \frac{\epsilon}{2}.\]
This bound applied to \eqref{lem7:eq1} yields 
\[\expect_{\Dcal}[\indicator\{h_{\hat{g}}(x) \neq y\}] \leq \inf_{h \in \Hcal}\expect_{\Dcal}[\indicator\{h(x)\neq y \}] + \frac{\epsilon}{2}.\]
Thus, we have shown the existence of a predictor $h_{\hat{g}} \in C_{\pm 1}(S_U)$ that achieves agnostic PAC bounds for $\Hcal$. Let $\hat{h}$ be the predictor returned by step $4$ of the algorithm. 
Next, we show that for sufficiently large $S_L$, the predictor $\hat{h}$ also attains agnostic PAC bounds. Recall that by Hoeffding's Inequality and union bound, with probability at least $1-\delta/2$, the empirical risk of every hypothesis in $C_{\pm 1}(S_L)$ on a sample of size $\geq  \frac{8 }{\epsilon^2} \log{\frac{4 |C_{\pm 1}(S_U)|}{\delta}} $ is at most $\epsilon/4$ away from its true error. So, if $|S_L| \geq  \frac{8}{\epsilon^2} \log{\frac{ 4|C_{\pm 1}(S_U)|}{\delta}} $, then with probability at least $1-\delta/2$, the empirical risk of the predictor $h_{\hat{g}}(x)$ is
\[\frac{1}{|S_L|}\sum_{(x,y) \in S_L} \indicator\{h_{\hat{g}}(x) \neq y\} \leq \expect_{\Dcal}[\indicator\{h_{\hat{g}}(x) \neq y\}] + \frac{\epsilon}{4} \leq\inf_{h \in \Hcal}\expect_{\Dcal}[\indicator\{h(x)\neq y \}] + \frac{3\epsilon}{4} ,\]
where the last inequality follows from the risk guarantee of $h_{\hat{g}}$ established above. Since $\hat{h}$ is the empirical risk minimizer over $S_L$, we must have 
\[\frac{1}{|S_L|}\sum_{(x,y) \in S_L} \indicator\{\hat{h}(x) \neq y\} \leq \frac{1}{|S_L|}\sum_{(x,y) \in S_L} \indicator\{h_{\hat{g}}(x) \neq y\} \leq \inf_{h \in \Hcal}\expect_{\Dcal}[\indicator\{h(x)\neq y \}] + \frac{3\epsilon}{4}.\]
Finally, as the population risk of $\hat{h}$ is at most $\epsilon/4$ away from its empirical risk, we have
\[ \expect_{\Dcal}[\indicator\{\hat{h}(x) \neq y\}] \leq \inf_{h \in \Hcal} \expect_{\Dcal}[\indicator\{h(x) \neq y\}] + \epsilon,\]
which is the agnostic PAC guarantee for $\Hcal$. Applying union bounds, the entire process, running algorithm $\Acal$ on the dataset augmented by $g^{\star, \alpha}$ and the ERM in step 4, succeeds with probability $1- \delta$.  This establishes that the Algorithm \ref{alg:binary_red} is an agnostic PAC learner for $\Hcal$.  The sample complexity of Algorithm \ref{alg:binary_red} is the number of samples required for Algorithm $\Acal$ to succeed and the ERM in step $4$ to succeed. Thus, the overall sample complexity of Algorithm \ref{alg:binary_red}, denoted $m_{\Hcal}(\epsilon, \delta)$, can be bounded as 
\begin{equation*}
    \begin{split}
        m_{\Hcal}(\epsilon, \delta) &\leq m_{\Acal}\left(\frac{\epsilon \, \psi(\gamma/2)}{4} , \frac{\delta}{2}\right) + \frac{8}{\epsilon^2} \log{\frac{ 4|C_{\pm 1}(S_U)|}{\delta}} \\
        &\leq m_{\Acal}\left(\frac{\epsilon \, \psi(\gamma/2)}{4} , \frac{\delta}{2}\right) \left( 1+ \frac{8}{\epsilon^2} \log{\left( \frac{2}{\min(\gamma/2, \epsilon \psi(\gamma/2)/4)}\right)}\right)  + \frac{8}{\epsilon^2} \log{\frac{4}{\delta}} \\
    \end{split}
\end{equation*}
where the second inequality follows because  $|C_{\pm 1}(S_U)| = |S_{\text{aug}}| \leq (2/\alpha)^{ |S_U|}$ and we need $|S_U| $ to be of size $m_{\Acal}\left(\frac{\epsilon \, \psi(\gamma/2)}{4} , \frac{\delta}{2}\right) $. We also use the fact that $ \alpha =\min( \frac{\gamma}{2},  \frac{\epsilon\psi(\frac{\gamma}{2})}{4}). $ 

However, it is well known \citep[Theorem 6.8]{ShwartzDavid} that the sample complexity of learning $\Hcal$ in agnostic setting is 
\[C \frac{d+\log(2/\delta)}{\epsilon^2}\]
for some $C > 0$. Thus, we must have $m_{\Hcal}(\epsilon, \delta) \geq C (d+ \log(2/\delta))/\epsilon^2$. However, this is a contradiction because $ d$ can be arbitrarily large but $m_{\Hcal}(\epsilon, \delta)$ must have a finite upper bound for every fixed $\epsilon, \delta$. Therefore, the function class $\Gcal$ cannot be learnable with respect to $\psi \circ d_1$ whenever there exists a scale $\gamma \in (0,1)$ such that $\text{fat}_{\gamma}(\Gcal) = \infty$. \end{proof}

\section{Rademacher Based Proof for  Batch  Regression}\label{appdx:rademacher_batch}

 To show that the learnability of each $\Fcal_k$ with respect to $d_1$ is sufficient for the learnability of $\Fcal$ with respect to $\ell_p$ norms for $p \geq 1$, we use the fact that $\ell_p(f(x), y)$ is a $K$-Lipschitz in its first argument with respect to $\norm{ \cdot }_{\infty} $ norm, that is $|\ell_p(f(x), y) - \ell_{p}(g(x), y)| \leq K \norm{f(x) -g(x)}_{\infty}$, and use the following bound on the Rademacher complexity of the loss class $\ell \circ \Fcal = \{(x, y) \mapsto \ell(f(x), y)\} \, |\,  f \in \Fcal \}$.
\begin{lemma} [\cite{Foster2019contraction}]
\label{lem: contraction-batch}
Let $\mathcal{F} \subseteq\Ycal^{\Xcal}$ be a multioutput function class. For any $\delta \in (0,1)$, there exists a constants $0 < c < 1$ and $C > 0$  such that 
\[\mathfrak{R}_n(\ell_p \circ \Fcal) \leq  K \inf_{\alpha > 0} \left\{4\alpha + \frac{ C}{\sqrt{n}}  \sum_{k=1}^K \int_{\alpha}^1\sqrt{\emph{\text{fat}}_{c\epsilon}(\mathcal{F}_k)\log^{1+\delta}\left(\frac{  \: e\: n}{\epsilon}\right)} d\epsilon\right\}.\]
\end{lemma}
\noindent The result presented here is in fact the intermediate result in \cite{Foster2019contraction}, and we provide a sketch of how their argument can be adapted to our setting. \begin{proof}
   Note that for $f, g \in \Fcal$, we have $|\ell_p(f(x), y) - \ell_p(g(x), y)| \leq |\ell_p(f(x), g(x))| \leq \ell_1(f(x), g(x)) \leq K \norm{f(x) - g(x)}_{\infty}$. Furthermore, we have that $|f_k(x) -y^k| \leq 1$, so we obtain $|\ell_p(f(x), y)| \leq K$. Define the normalized $\ell_p$ loss as $\bar{\ell}_p(f(x), y) := \ell_{p}(f(x), y)/K$. 
   By standard chaining argument, we know that 
    \[\mathfrak{R}_n(\ell \circ \Fcal) = K\,  \mathfrak{R}_n(\bar{\ell} \circ \Fcal) \leq K\, \inf_{\alpha > 0} \left\{4\alpha + \frac{ 12}{\sqrt{n}}  \int_{\alpha}^1\sqrt{\log \Ncal_2(\bar{\ell}_p \circ \Fcal, \epsilon, n)} d\epsilon\right\}. \]

    Since a cover with $||\cdot||_{\infty}$ norm is also a cover with respect to $||\cdot||_2$ norm, we have that $\log \Ncal_2(\bar{\ell}_p \circ \Fcal, \epsilon, n) \leq \log \Ncal_{\infty}(\bar{\ell} \circ \Fcal, \epsilon, n) $. Since $\bar{\ell}_p(f(x), y)$ is $1-$Lipschitz with respect to $||\cdot||_{\infty}$ norm, following Lemma 1 of \cite{Foster2019contraction}, we obtain $\log \Ncal_{\infty}(\bar{\ell}_p \circ \Fcal, \epsilon, n) \leq \sum_{k=1}^K \log \Ncal_{\infty}(\Fcal_k, \epsilon, n)$. 
    
     A result due to \cite{rudelson2006combinatorics} states that for any $\delta \in (0,1)$, there exists constants $0 < c_k < 1$ and $C_k > 0$ such that 
     \[ \log \Ncal_{\infty}( \Fcal_k,  \epsilon, n) \leq C_k\, \text{fat}_{c_k\epsilon}(\Fcal_k)\, \log^{1+ \delta}{(en/\epsilon)}.\]
     Picking $C = \max_k C_k$ and $c = \min_k c_k$, we obtain the contraction inequality
    \[ \mathfrak{R}_n(\ell \circ \Fcal) \leq K\, \inf_{\alpha > 0} \left\{4\alpha + \frac{ 12\, C }{\sqrt{n}}  \int_{\alpha}^1\sqrt{ \sum_{k=1}^K \, 
 \text{fat}_{c\epsilon}(\mathcal{F}_k)\log^{1+\delta}\left(\frac{ e\: n}{  \: \epsilon}\right)} d\epsilon\right\}. \]
    Using $\sqrt{ \sum_{k=1}^K \, 
 \text{fat}_{c\epsilon}(\mathcal{F}_k)\log^{1+\delta}\left(\frac{ e\: n}{  \: \epsilon}\right)} \leq \sum_{k=1}^K \sqrt{ \text{fat}_{c\epsilon}(\mathcal{F}_k)\log^{1+\delta}\left(\frac{ e\: n}{  \: \epsilon}\right) }$ yields the desired contraction inequality.

\end{proof}

With Lemma \ref{lem: contraction-batch} in our repertoire, the sufficiency proof is a routine uniform convergence argument. 
\begin{proof}(of sufficiency in Theorem \ref{thm:batch_reg_nondecom})
    Suppose each restriction $\Fcal_k$ is learnable with respect to $d_1$. Then, we know that for all $k \in [K]$ and for all $1 > \gamma > 0$, we have $ \text{fat}_{\gamma} (\Fcal_k) < \infty$ \citep{BARTLETT1996}, \cite[Chatper 19]{anthony_bartlett_1999}). Using Lemma \ref{lem: contraction-batch}, for $\delta = 1/2$, we can find constants $c, C$ such that 
    \[\mathfrak{R}_n(\ell_p \circ \Fcal) \leq  K \inf_{\alpha > 0} \left\{4\alpha + \frac{ C}{\sqrt{n}}  \sum_{k=1}^K \int_{\alpha}^1\sqrt{\text{fat}_{c\epsilon}(\mathcal{F}_k)\log^{3/2}\left(\frac{  \: e\: n}{\epsilon}\right)} d\epsilon\right\}.\]
    Fix $\alpha > 0$. The second term inside infimum vanishes  as $n \to \infty$, yielding $\mathfrak{R}_n(\ell \circ \Fcal) \leq 4 \alpha K $. As $\alpha > 0$ is arbitrary, the Rademacher complexity $\mathfrak{R}_n(\ell \circ \Fcal)$ goes to $0$  as $n \to \infty$. This argument can be readily turned into non-asymptotic bounds on $\mathfrak{R}_n(\ell_p \circ \Fcal)$ if the precise form of $\text{fat}_{\gamma}(\Fcal_k)$ as a function of $\gamma$ is known.    Since the empirical Rademacher complexity vanishes,  uniform convergence holds over the loss class $\ell_p \circ \Fcal$ and thus $\Fcal$ is learnable with respect to $\ell_p$ via empirical risk minimization. 
\end{proof}

\section{ Online Multilabel Learnability with respect to Hamming Loss }\label{appdx:proof_OLHam}
In this section, we provide the proof of Theorem \ref{thm:OLHam}. 
\begin{proof} We first prove that the online learnability of each restriction is sufficient for the online learnability of $\ell_H$.

\noindent\textbf{Part 1: Sufficiency.}
 Our proof is based on a reduction: given oracle access to online learners $\{\mathcal{A}_k\}_{k=1}^K$ for $\{\mathcal{F}_k\}_{k=1}^{K}$ with respect to $\ell_{0\text{-}1}$, we construct an online learner $\mathcal{A}$ for $\mathcal{F}$ with respect to $\ell_H$. In fact, similar to the batch setting, the online multilabel learning algorithm $\mathcal{A}$ is simple: in each round $t \in [T]$, receive $x_t$, query the predictions $\Acal_1(x_t), ..., \Acal_K(x_t)$, and finally predict the concatenation $\hat{y}_t = (\Acal_1(x_t), ..., \Acal_K(x_t))$. Once the true label $y_t = (y_t^1, ..., y_t^K)$ is revealed, update each online learner $\mathcal{A}_k$ by passing $(x_t, y_t^k)$ for $k \in [K]$. It suffices to show that the expected regret of $\mathcal{A}$ is sublinear in $T$ with respect to $\ell_H$. By Definition \ref{def:agnOL}, we have that for all $k \in [K]$, 

$$\mathbb{E}\left[\sum_{t=1}^T \mathbbm{1}\{\Acal_k(x_t) \neq y_t^k\} - \inf_{f_k \in \mathcal{F}_k}\sum_{t=1}^T \mathbbm{1}\{f_k(x_t) \neq y_t^k\}\right] \leq R_k(T)$$
where $R_k(T)$ is some sublinear function in $T$. Summing the regret bounds across all $k \in [K]$ splitting up the expectations, and using linearity of expectation, we get that
$\mathbb{E}\left[\sum_{k=1}^K \sum_{t=1}^T \mathbbm{1}\{\Acal_k(x_t) \neq y_t^k\} \right] - \mathbb{E} \left[ \sum_{k=1}^K  \inf_{f_k \in \mathcal{F}_k}\sum_{t=1}^T \mathbbm{1}\{f_k(x_t) \neq y_t^k\}\right] \leq \sum_{k=1}^K R_k(T).$
Noting that $\sum_{k=1}^K  \inf_{f_k \in \mathcal{F}_k}\sum_{t=1}^T \mathbbm{1}\{f_k(x_t) \neq y_t^k\} \leq \inf_{f \in \mathcal{F}}\sum_{k=1}^K  \sum_{t=1}^T \mathbbm{1}\{f_k(x_t) \neq y_t^k\}$, swapping the order of summations, and using the definition of $\ell_H$ we have that, 
$$\mathbb{E}\left[\sum_{t=1}^T \ell_H(\mathcal{A}(x_t), y_t)\right] - \mathbb{E} \left[\inf_{f \in \mathcal{F}}  \sum_{t=1}^T \ell_H(f(x_t), y_t) \right]  \leq \sum_{k=1}^K R_k(T),$$
where $\mathcal{A}(x_t) = (\mathcal{A}_1(x_t), ..., \mathcal{A}_K(x_t))$. This concludes the proof of this direction since $\sum_{k=1}^K R_k(T)$ is still a sublinear function in $T$.

\noindent\textbf{Part 2: Necessity.}
Next we prove that if $\mathcal{F}$ is online learnable with respect to to $\ell_H$, then each $\mathcal{F}_k$ is online learnable with respect to $\ell_{0\text{-}1}$. Namely, given oracle access to an online learner $\mathcal{A}$ for $\mathcal{F}$ with respect to $\ell_H$, we construct an online learner $\mathcal{B}$ for $\mathcal{F}_1$ with respect to $\ell_{0\text{-}1}$. A similar reduction can be used to construct online learners for each restriction $\Fcal_k$.  Similar to the batch setting, the online learning algorithm $\mathcal{B}$ is simple: in each round $t \in [T]$, receive $x_t$, query $\hat{y}_t = \Acal(x_t)$ and predict $\hat{y}^1_t = \Acal_1(x_t)$. Once the true label $y^1_t$ is revealed, update $\Acal$ by passing $(x_t, y_t)$ where $y_t$ = $(y_t^1, \sigma_t^2, ..., \sigma_t^K)$ and $\{\sigma_t^i\}_{i=2}^K$ is an i.i.d sequence of Rademacher random variables.

It suffices to show that the expected regret of $\mathcal{B}$ is sublinear in $T$ with respect to $\ell_{0\text{-}1}$. As previously mentioned, we  assume that the sequence  $(x_1, y_1^1), ..., (x_T, y_T^1)$ is chosen by an oblivious adversary, and thus is not random.  Let $y_t = (y^1_t, \sigma^2_t, ..., \sigma^K_t)$. By Definition \ref{def:agnOL}, we have that,  

$$\mathbb{E}\left[\sum_{t=1}^T \ell_H(\mathcal{A}(x_t), y_t) - \inf_{f \in \mathcal{F}}\sum_{t=1}^T \ell_H(f(x_t), y_t)\right] \leq R(T, 2^K) $$
where the expectation is over both the randomness of $\mathcal{A}(x_t)$ and $(\sigma^2_t, ..., \sigma^K_t)$  and $R(T, 2^K)$ is a sub-linear function of $T$. Splitting up the expectation, using the definition of the Hamming loss, and by the linearity of expectation, we have that

$$\sum_{t=1}^T \sum_{k=1}^K \mathbb{E}\left[\mathbbm{1} \{\Acal_k(x_t) \neq y^k_t\} \right] - \inf_{f \in \mathcal{F}}\sum_{t=1}^T \sum_{k=1}^K \mathbb{E}\left[\mathbbm{1} \{f_k(x_t) \neq y^k_t\} \right] \leq R(T, 2^K).$$

Next, observe that for every $t \in [T]$, for every $k \in \{2, ..., K\}$, the randomness of $y_t^k = \sigma_t^k$ implies  $\mathbb{E}\left[\mathbbm{1} \{\Acal_k(x_t) \neq y^k_t\} \right] = \mathbb{E}\left[\mathbbm{1} \{f(x_t) \neq y^k_t\} \right] = \frac{1}{2}$. Thus, 
$$\sum_{t=1}^T \mathbb{E}\left[\mathbbm{1} \{\Acal_1(x_t) \neq y^1_t\} + \frac{K-1}{2} \right] - \inf_{f \in \mathcal{F}}\sum_{t=1}^T \mathbb{E}\left[\mathbbm{1} \{f_1(x_t) \neq y^1_t\} + \frac{K-1}{2} \right] \leq R(T, 2^K).$$
Canceling constant factors gives,  
$\mathbb{E}\left[\sum_{t=1}^T \mathbbm{1} \{\Acal_1(x_t) \neq y^1_t\} \right] - \inf_{f_1 \in \mathcal{F}_1}\sum_{t=1}^T \mathbbm{1} \{f_1(x_t) \neq y^1_t\} \leq R(T, 2^K),$ showing that $\mathcal{B}$ is an online agnostic learner for $\mathcal{F}_1$ with respect to $\ell_{0\text{-}1}$.
\end{proof}

\section{MCLdim Characterizes Online Multilabel Learnability}
\label{app:MCLdimchar}
In this section, we show that the MCLdim characterizes the online learnability of a multilabel function class $\mathcal{F} \subseteq\mathcal{Y}^{\mathcal{X}}$ with respect to to any loss $\ell$ that satisfies the identity of indiscernibles. Theorem \ref{thm:MCLdimchar} makes this more precise. 

\begin{theorem}
    \label{thm:MCLdimchar}
    Let $\ell$ be any loss function satisfying the identity of indiscernibles. A function class $\mathcal{F} \subseteq\mathcal{Y}^{\mathcal{X}}$ is online learnable with respect to $\ell$ if and only if $\text{MCLdim}(\mathcal{F}) < \infty$.
\end{theorem}

\begin{proof} (of sufficiency)
We first show that finiteness of MCLdim is sufficient for online learnability. The proof follows exactly like the proof of Lemma \ref{lem: arbOL}. We include it here again for completeness sake. Let $\ell$  be any loss function satisfying the identity of indiscernibles and $\mathcal{F} \subseteq\mathcal{Y}^{\mathcal{X}}$ be a multilabel function class such that $\text{MCLdim}(\mathcal{F}) = d < \infty$.  Since $\mathcal{F}$ has finite MCLdim, the deterministic Multiclass Standard Optimal Algorithm for $\mathcal{F}$ , hereinafter denoted $\text{MCSOA}(\mathcal{F})$, achieves mistake-bound $d$ in the realizable setting \citep{DanielyERMprinciple}. Therefore, following the same procedure as in \cite{DanielyERMprinciple}, we can construct a finite set of experts $\mathcal{E}$ of size $|\mathcal{E}| = \sum_{j=0}^{d} \binom{T}{j}|\text{im}(\mathcal{F})|^j \leq (2^KT)^d$ such that for any (oblivious) sequence of instances $x_1, ..., x_T$, for any function $f \in \mathcal{F}$, there exists an expert $E_f \in \mathcal{E}$, such that $f(x_t) = E(x_t)$ for all $t \in [T]$. Finally, running the celebrated Randomized Exponential Weights Algorithm (REWA) using $\mathcal{E}$ as the set of experts and the scaled loss function $\frac{\ell}{B} \in [0,1]$  guarantees that for any labelled sequence $(x_1, y_1), ..., (x_T, y_T)$, 
\begin{align*}
\mathbb{E}\left[\sum_{t=1}^T \ell(\hat{y}_t, y_t) - \inf_{E \in \mathcal{E}} \sum_{t=1}^T \ell(E(x_t), y_t)\right] &\leq \mathbb{E}\left[\sum_{t=1}^T \ell(\hat{y}_t, y_t) - \inf_{f \in \mathcal{F}} \sum_{t=1}^T \ell(f(x_t), y_t)\right]\\
&\leq O\left(B\sqrt{T\ln(|\mathcal{E}|)}\right) \leq O\left(B\sqrt{dTK\ln(T)}\right)\\
\end{align*}
where $\hat{y}_t$ is the prediction of REWA in the $t$'th round. Thus, running REWA over the set of experts $\mathcal{E}$ using $\frac{\ell}{B}$ gives an online learner for $\mathcal{F}$ with respect to $\ell$.
\end{proof}

\begin{proof} (of necessity)
    To prove necessity, we need to show that if $\mathcal{F}$ is online learnable with respect to $\ell$, then $\text{MCLdim}(\mathcal{F}) < \infty$. To do so, we show that if $\mathcal{F}$ is online learnable with respect to $\ell$, then $\mathcal{F}$ is online learnable with respect to $\ell_{0\text{-}1}$ in the realizable setting. Let $\mathcal{A}$ be an online learner for $\mathcal{F}$ with respect to $\ell$. Then, by definition, 
    $$\mathbb{E}\left[\sum_{t=1}^T \ell(\mathcal{A}(x_t), y_t) - \inf_{f \in \mathcal{F}}\sum_{t=1}^T \ell(f(x_t), y_t)\right] \leq R(T, 2^K) $$
    where $R(T, 2^K)$ is a sublinear function in $T$. Since $\ell$ satisfies the identity of indiscernibles, in the realizable setting, $\inf_{f \in \mathcal{F}}\sum_{t=1}^T \ell(f(x_t), y_t) = 0$. Therefore, under realizability, 
     $$\mathbb{E}\left[\sum_{t=1}^T \ell(\mathcal{A}(x_t), y_t) \right] \leq R(T, 2^K).$$
     Because there are only a finite number of inputs to $\ell$, there must exist a universal constant $a$ such that $a\ell_{0\text{-}1} \leq \ell$. Substituting in gives that, 

      $$\mathbb{E}\left[\sum_{t=1}^T \ell_{0\text{-}1}(\mathcal{A}(x_t), y_t) \right] \leq \frac{R(T, 2^K)}{a}.$$

      Since $a$ is a universal constant that does not depend on $T$, $\frac{R(T, 2^K)}{a}$ is still a sublinear function in $T$, implying that $\mathcal{A}$ is also a realizable online learner for $\mathcal{F}$ with respect to $\ell_{0\text{-}1}$. This completes the proof as MCLdim characterizes realizable learnability and so we must have $\text{MCLdim}(\mathcal{F}) < \infty$.
\end{proof}

\section{Proofs for Bandit Online Multilabel Classification }\label{appdx:proofs_bandit}
In this section, we provide proofs for the characterization of online multilabel classification under bandit feedback. 

\begin{proof}(of Theorem \ref{thm:banditonlreal2agn})
The proof of Theorem \ref{thm:banditonlreal2agn} is nearly identical to the proof of Theorem \ref{thm:onlreal2agn}. The only difference is that in Algorithm \ref{alg:onlreal2agn}, we now need use the bandit Expert's algorithm EXP4 from \cite{auer2002nonstochastic} instead of REWA. Similar to REWA, based on Theorem 2.3 in \cite{daniely2013price} and the fact that $A, B$ and $P$ are independent, EXP4 guarantees that 

$$\mathbb{E}\left[\sum_{t=1}^T \ell(\mathcal{P}(x_t), y_t)\right] \leq \mathbb{E}\left[\inf_{E \in \mathcal{E}_B} \sum_{t=1}^T \ell(E(x_t), y_t)\right] + eM\mathbb{E}\left[\sqrt{2T|\mathcal{Y}|\ln(|\mathcal{E}_B|)}\right],$$

where $\mathcal{P}(x_t)$ denotes EXP4's prediction in round $t$. The remaining proof for deriving the upper bound

$$\mathbb{E}\left[\inf_{E \in \mathcal{E}_B} \sum_{t=1}^T \ell(E(x_t), y_t)\right]  \leq \inf_{f\in \mathcal{F}}\sum_{t= 1}^T \ell(f(x_t), y_t) + \frac{cT}{T^{\beta}}\overline{R}(T^{\beta}, |\mathcal{Y}|)$$ 
is identical to that in Theorem \ref{thm:onlreal2agn}, so we omit it here. Putting these pieces together gives the stated guarantee. 
\end{proof}

\begin{proof}(of Theorem \ref{thm:banditcharac})
 Let $c = \frac{\max_{r \neq t}\ell(r, t)}{\min_{r \neq t}\ell(r, t)}$. We first show necessity: if $\mathcal{F}$ is bandit online learnable with respect to $\ell$, then each restriction $\mathcal{F}_k$ is online learnable with respect to $\ell_{0\text{-}1}$. This follows trivially from the fact that if $\mathcal{A}$ is a bandit online learner for $\mathcal{F}$, then $\mathcal{A}$ is also an online learner for $\mathcal{F}$ under full-feedback. Thus, by Theorem \ref{thm:onlinecharac}, online learnability of $\mathcal{F}$ with respect to $\ell$ implies online learnability of restriction $\mathcal{F}_k$ with respect to the 0-1 loss. 

We now focus on showing sufficiency: if for all $k \in [K]$, $\mathcal{F}_k$ is online learnable with respect to 0-1 loss, then $\mathcal{F}$ is \textit{bandit} online learnable with respect to loss $\ell$. Since $|\mathcal{Y}| = 2^K < \infty$ and $\ell$ is a $c$-subadditive, by Theorem \ref{thm:banditonlreal2agn}, it suffices to show that there exists a realizable online learner for $\mathcal{F}$ with respect to $\ell$. However, using Theorem \ref{thm:onlinecharac},  online learnability of each restriction $\mathcal{F}_k$ with respect to 0-1 the loss implies (agnostic) online learnability of $\mathcal{F}$ with respect to $\ell$. Since an agnostic online learner is trivially a realizable online learner, the proof is complete. 
\end{proof}

\section{Equivalence of $d_1$ and $\psi \circ d_1$ Online Learnability} \label{app:onlinereg_scalar}

\begin{proof} (of Lemma \ref{lem:onlinereg_scalar})
Since $\psi$ is a Lipschitz function, the proof of sufficiency follows immediately from  Corollary 5 in \cite{rakhlin2015online}, a contraction Lemma for the sequential Rademacher complexity. Thus, we focus on proving necessity - if $\mathcal{G}$ is online learnable with respect to $\psi \circ d_1$, then $\mathcal{G}$ is online learnable with respect to $d_1$.

Since the sequential fat shattering dimension of $\mathcal{G}$ characterizes $d_1$ learnability \citep{rakhlin2015online}, it suffices to show that $\mathcal{G}$ being online learnable with respect to $\psi \circ d_1$ implies $\text{fat}^{\text{seq}}_{\gamma}(\mathcal{G}) < 0$ for every $\gamma \in (0, 1)$. Like in the batch setting, we prove this via contradiction. 

Suppose, for the sake of contradiction, $\mathcal{G}$ is online learnable with respect to $\psi \circ d_1$ but there exists a scale $\gamma \in (0, 1)$ such that $\text{fat}^{\text{seq}}_{\gamma}(\mathcal{G}) = \infty$. Then, for every $T \in \mathbb{N}$, there exists a $\mathcal{X}$-valued binary tree $\mathcal{T}$ and a $[0, 1]$-valued binary witness tree $\mathcal{R}$ both of depth $T$ such that for all $\sigma \in \{-1, 1\}^T$, there exists $g_{\sigma} \in \mathcal{G}$ such that for all $t \in [T]$, $\sigma_t(g_{\sigma}(\mathcal{T}_t(\sigma_{<t})) - \mathcal{R}_t(\sigma_{<t})) \geq \gamma$. Without loss of generality, assume that for any path $\sigma \in \{-1, 1\}^T$, the set of instances $\{\mathcal{T}_t(\sigma_{<t})\}_{t=1}^T$ are distinct. This is true because we can construct a $\gamma$-shattered tree of much bigger depth and prune away repeated instances along a path to get a tree of depth $T$. Define $\mathcal{G}_T = \{g_{\sigma} \in \mathcal{G}| \sigma \in \{-1, 1\}^T\}$ to be the set of functions that shatter $\mathcal{T}$ with witness $\mathcal{R}$. Let $X \subseteq\mathcal{X}$ denote the set of examples that label the internal nodes of $\mathcal{T}$.  Consider the binary hypothesis class $\mathcal{H} = \{-1, 1\}^{X}$ which contains all possible functions from $X$ to $\{-1, 1\}$. By definition of $\mathcal{H}$, we must have $\text{Ldim}(\mathcal{H}) \geq T$. Therefore, $\mathcal{T}$, with left and right edges labeled by $-1$ and $+1$ respectively, is shattered by $\mathcal{H}$. Let $\mathcal{T}_{\pm}$ denote such a tree. Note that for all $t \in [T]$, we have $\mathcal{T}_{t}(\sigma_{<t}) = \mathcal{T}_{\pm,t}(\sigma_{<t})$. Since $\text{Ldim}(\mathcal{H}) \geq T$,  any realizable online learner for $\mathcal{H}$ must make at least $\frac{T}{2}$ mistakes in expectation for an adversary that plays according to a root-to-leaf path in $\mathcal{T}_{\pm}$ chosen \textit{uniformly at random}. However, using an agnostic online learner for $\mathcal{G}$ with respect to $\psi \circ d_1$, we construct a realizable online learner for $\mathcal{H}$ that achieves a \textit{sublinear} regret bound when an adversary plays according to a root-to-leaf path in $\mathcal{T}_{\pm}$ chosen uniformly at random. Since $\text{fat}^{\text{seq}}_{\gamma}(\mathcal{G}) = \infty$, $T$ can be made arbitrarily large, eventually giving us a contradiction. 

To that end,  let $\mathcal{A}$ be an online learner for $\mathcal{G}$ with respect to $\psi \circ d_1$ with regret $R_{\mathcal{A}}(T)$. 
% For any path prefix given by $\sigma_{<i} \in \{-1, 1\}^{i-1}$ and any function $f \in [0, 1]^{X}$, define the binary threshold function $h(x; f, \sigma_{<i}) = \mathbbm{1}\{f(x) \geq \mathcal{R}_i(\sigma_{<i})\}$. We are now ready to begin our construction. 
Let $\sigma \sim \{-1, 1\}^T$ denote a sequence of $T$ i.i.d Rademacher random variables and $\{(\mathcal{T}_{\pm,t}(\sigma_{<t}), \sigma_t)\}_{t=1}^T$ the associated sequence of labeled instances determined by traversing $\mathcal{T}_{\pm}$ using $\sigma$. Note that $\{(\mathcal{T}_{\pm,t}(\sigma_{<t}), \sigma_t)\}_{t=1}^T$ is a sequence of labeled instances corresponding to a root-to-leaf path in $\mathcal{T}_{\pm}$ chosen \textit{uniformly at random}. By construction of $\mathcal{H}$, there exists a $h_{\sigma}^{\star} \in \mathcal{H}$ such that $h^{\star}_{\sigma}(\mathcal{T}_{\pm,t}(\sigma_{<t})) = \sigma_t$ for all $t \in [T]$ and therefore the stream is realizable by $\mathcal{H}$. Let $g_{\sigma}^{\star} \in \mathcal{G}$ be the function at the end of the root-to-leaf path corresponding to $\sigma$ in $\mathcal{T}$, the original tree shattered by $\mathcal{G}$. 
% Then, observe that for all $t \in [T]$,  $h(\mathcal{T}_{\pm,t}(\sigma_{<t}); g^{\star}_{\sigma}, \sigma_{<t})  = h^{\star}_{\sigma}(\mathcal{T}_{\pm,t}(\sigma_{<t}))$ by the definition of shattering. 

We now use $\mathcal{A}$ to construct an  agnostic online learner for $\mathcal{H}$ with sublinear regret on the stream $\{(\mathcal{T}_{\pm,t}(\sigma_{<t}), \sigma_t)\}_{t=1}^T$. Our algorithm is very similar to realizable-to-agnostic conversion in Theorem \ref{thm:onlreal2agn}. Namely, we construct a finite set of experts, each of which uses $\mathcal{A}$ to make predictions, but only updates $\mathcal{A}$ on certain rounds. Finally,  we run REWA using this set of Experts over our stream. For completeness' sake, we provide the full description below. 

For any bitstring $b \in \{0, 1\}^T$, let $\phi: \{t: b_t = 1\} \rightarrow \text{im}(\mathcal{G}^{\alpha})$ denote a function mapping time points where $b_t = 1$ to elements in the discretized image space $\text{im}(\mathcal{G}^{\alpha})$. Let $\Phi_b: (\text{im}(\mathcal{G}^{\alpha}))^{\{t: b_t = 1\}}$ denote all such functions $\phi$. For every $g \in \mathcal{G}$, let $\phi_b^g \in \Phi_b$ be the mapping such that for all $t \in \{t: b_t = 1\}$, $\phi_b^g(t) = g^{\alpha}(\mathcal{T}_{t}(\sigma_{<t}))$. Let $|b| = |\{t: b_t = 1\}|$. For every $b \in \{0, 1\}^T$ and $\phi \in \Phi_b$, define an Expert $E_{b, \phi}$. Expert $E_{b, \phi}$, formally presented in Algorithm  \ref{alg:expert_contrad}, uses $\mathcal{A}$ to make predictions in each round. However, $E_{b, \phi}$ only updates $\mathcal{A}$ on those rounds where $b_t = 1$, using $\phi$ to produce a labeled instance $(\mathcal{T}_{t}(\sigma_{<t}), \phi(t))$. For every $b \in \{0, 1\}^T$, let $\mathcal{E}_b = \bigcup_{\phi \in \Phi_b} \{E_{b, \phi}\}$ denote the set of all Experts parameterized by functions $\phi \in \Phi_b$. If $b$ is the all zeros bitstring, then $\mathcal{E}_b$ is empty. Therefore, we actually define $\mathcal{E}_b = \{E_0\} \cup \bigcup_{\phi \in \Phi_b} \{E_{b, \phi}\}$, where $E_0$ is the expert that never updates $\mathcal{A}$.  Note that $1 \leq |\mathcal{E}_b| \leq (\frac{2}{\alpha})^{|b|}$.

\begin{algorithm}
\caption{Expert($b$, $\phi$)}
\setcounter{AlgoLine}{0}
\label{alg:expert_contrad}
\KwIn{Independent copy of Online Learner $\mathcal{A}$ for $\psi \circ d_1$}
\For{$t = 1,...,T$} {
    Receive example $\mathcal{T}_{\pm,t}(\sigma_{<t})$

    Predict $\hat{y}_t = 2\,\mathbbm{1}\{\mathcal{A}(\mathcal{T}_{\pm,t}(\sigma_{<t}))  \geq \mathcal{R}_t(\sigma_{<t})\} - 1$

    Receive $y_t = \sigma_t$
    
    \uIf{$b_t = 1$}{        
        Update $\mathcal{A}$ by passing $(\mathcal{T}_{\pm,t}(\sigma_{<t}), \phi(t))$
    }
}
\end{algorithm}

With this notation in hand, we are now ready to present Algorithm \ref{alg:onlcontrad}, our main online learner $\mathcal{Q}$ for $\mathcal{H}$ with respect to 0-1 loss. The analysis is similar to the one before, but we include it below for completeness sake. 

\begin{algorithm}
\caption{Online learner $\mathcal{Q}$ for $\Hcal$ with respect to 0-1 loss}
\setcounter{AlgoLine}{0}
\label{alg:onlcontrad}
\KwIn{ Parameters $0 < \beta < 1$ and $0 < \alpha < \frac{\gamma}{2}$}

Let $B \in \{0, 1\}^T$ such that  $B_t \overset{\text{iid}}{\sim} \text{Bernoulli}(\frac{T^{\beta}}{T})$ 

% Let $\Phi_B = (\mathcal{Y}_{2:K}^{\alpha})^{\{t: B_t = 1\}}$ denote the set of all possible functions mapping from  ${\{t: B_t = 1\}}$ to $\mathcal{Y}_{2:K}^{\alpha}$.

Construct the set of experts $\mathcal{E}_B = \{E_0\} \cup \bigcup_{\phi \in \Phi_B} \{E_{B, \phi}\}$ according to Algorithm \ref{alg:expert_contrad}.

Run REWA $\mathcal{P}$ using $\mathcal{E}_B$ and the 0-1 loss over the stream $(\mathcal{T}_{\pm,1}(\sigma_{<1}), \sigma_1), ... ,(\mathcal{T}_{\pm,T}(\sigma_{<T}), \sigma_T)$

\end{algorithm}

Our goal now is to show that $\mathcal{Q}$ enjoys sublinear expected regret. There are three main sources of randomness: the randomness involved in sampling $B$, the internal randomness  of each independent copy of the online learner $\mathcal{A}$, and the internal randomness  of REWA. Let $B, A$ and $P$ denote the random variable associated with these sources of randomness respectively. By construction, $A, B$, and $P$ are independent. 

Using Theorem 21.11 in \cite{ShwartzDavid} and the fact that $A,B $ and $P$, are independent, REWA guarantees,  
$$\mathbb{E}\left[\sum_{t=1}^T \mathbbm{1}\{\mathcal{P}(\mathcal{T}_{\pm,t}(\sigma_{<t})) \neq \sigma_t \}\right] \leq \mathbb{E}\left[\inf_{E \in \mathcal{E}_B} \sum_{t=1}^T \mathbbm{1}\{E(\mathcal{T}_{\pm,t}(\sigma_{<t})) \neq \sigma_t \}\right] + \mathbb{E}\left[\sqrt{2T\ln(|\mathcal{E}_B|)} \right].$$

 Therefore,  
\begin{align*}
   \mathbb{E}\left[\sum_{t=1}^T \mathbbm{1}\{\mathcal{Q}(\mathcal{T}_{\pm,t}(\sigma_{<t})) \neq \sigma_t \}\right]  &= \mathbb{E}\left[\sum_{t=1}^T \mathbbm{1}\{\mathcal{P}(\mathcal{T}_{\pm,t}(\sigma_{<t})) \neq \sigma_t \}\right]\\
    &\leq \mathbb{E}\left[\inf_{E \in \mathcal{E}_B} \sum_{t=1}^T \mathbbm{1}\{E(\mathcal{T}_{\pm,t}(\sigma_{<t})) \neq \sigma_t \}\right] + \mathbb{E}\left[\sqrt{2T\ln(|\mathcal{E}_B|)}\right]\\
    &\leq \mathbb{E}\left[\sum_{t=1}^T \mathbbm{1}\{E_{B, \phi_B^{g^{\star}_{\sigma}}}(\mathcal{T}_{\pm,t}(\sigma_{<t})) \neq \sigma_t \}\right] + \mathbb{E}\left[\sqrt{2T\ln(|\mathcal{E}_B|)}\right].
\end{align*}
In the last step, we used the fact that for all $b \in \{0, 1\}^T$ and $g \in \mathcal{G}$, $E_{b, \phi_b^g} \in \mathcal{E}_b$. 

It now suffices to upperbound $\mathbb{E}\left[\sum_{t=1}^T \mathbbm{1}\{E_{B, \phi_B^{g^{\star}_{\sigma}}}(\mathcal{T}_{\pm,t}(\sigma_{<t})) \neq \sigma_t \}\right]$. We use the same notation used to prove Theorem \ref{thm:onlreal2agn}, but for the sake of completeness, we restate it here. Given an online learner $\mathcal{A}$ for $\psi \circ d_1$, an instance $x \in \mathcal{X}$, and an ordered sequence of labeled examples $L \in (\mathcal{X} \times [0, 1])^*$, let $\mathcal{A}(x|L)$ be the random variable denoting the prediction of $\mathcal{A}$ on the instance $x$ after running and updating on $L$. For any $b\in \{0, 1\}^T$, $g^{\alpha} \in \mathcal{G}^{\alpha}$, and $t \in [T]$, let $L^g_{b_{< t}} = \{(\mathcal{T}_{\pm,t}(\sigma_{<i}), g^{\alpha}(\mathcal{T}_{\pm,t}(\sigma_{<i}))): i < t \text{ and } b_i = 1\}$ denote the \textit{subsequence} of the sequence of labeled instances $\{(\mathcal{T}_{\pm,t}(\sigma_{<i}), g^{\alpha}(\mathcal{T}_{\pm,t}(\sigma_{<i})))\}_{i=1}^{t-1}$ where $b_i = 1$. Using this notation, we can write
\begin{align*}
\mathbb{E}\Bigg[\sum_{t=1}^T \mathbbm{1} \{ E_{B, \phi_B^{g^{\star}_{\sigma}}}(\mathcal{T}_{\pm,t}(\sigma_{<t})) \neq \sigma_t \}\Bigg] &=  \mathbb{E}\left[\sum_{t=1}^T \mathbbm{1}\Big\{2\,\mathbbm{1}\{\mathcal{A}(\mathcal{T}_{\pm,t}(\sigma_{<t})|L^{g^{\star}_{\sigma}}_{B_{< t}})  \geq \mathcal{R}_t(\sigma_{<t})\} - 1 \neq \sigma_t \Big\} \right]\\
    &= \mathbb{E}\left[\sum_{t=1}^T \mathbbm{1}\Big\{2\,\mathbbm{1}\{\mathcal{A}(\mathcal{T}_{\pm,t}(\sigma_{<t})|L^{g^{\star}_{\sigma}}_{B_{< t}}) \geq \mathcal{R}_t(\sigma_{<t})\} - 1 \neq h^{\star}_{\sigma}(\mathcal{T}_{\pm,t}(\sigma_{<t}))\Big\} \right]\\
   % &= \mathbb{E}\left[\sum_{t=1}^T \mathbbm{1}\Big\{\mathbbm{1}\{\mathcal{A}(\mathcal{T}_{\pm,t}(\sigma_{<t})|L^{g^{\star}}_{B_{< t}}) \geq \mathcal{R}_t(\sigma_{<t})\} \neq h(\mathcal{T}_{\pm,t}(\sigma_{<t}); g^{\star}_{\sigma}, \sigma_{<t})\Big\} \right]\\
    &\leq \mathbb{E}\left[\sum_{t=1}^T \mathbbm{1}\left\{|\mathcal{A}(\mathcal{T}_{\pm,t}(\sigma_{<t})|L^{g^{\star}_{\sigma}}_{B_{< t}}) - g^{\star}_{\sigma}(\mathcal{T}_{\pm,t}(\sigma_{<t}))| \geq \gamma \right\} \right] \\
    &\leq \mathbb{E}\left[\sum_{t=1}^T \mathbbm{1}\left\{|\mathcal{A}(\mathcal{T}_{\pm,t}(\sigma_{<t})|L^{g^{\star}_{\sigma}}_{B_{< t}}) - g^{\star, \alpha}_{\sigma}(\mathcal{T}_{\pm,t}(\sigma_{<t}))| \geq \frac{\gamma}{2} \right\} \right] \\
    &\leq \mathbb{E}\left[\sum_{t=1}^T \mathbbm{1}\left\{\psi \circ d_1(\mathcal{A}(\mathcal{T}_{\pm,t}(\sigma_{<t})|L^{g^{\star}_{\sigma}}_{B_{< t}}), g^{\star, \alpha}_{\sigma}(\mathcal{T}_{\pm,t}(\sigma_{<t}))) \geq \psi(\frac{\gamma}{2})\right\} \right] \\\\
    &\leq \frac{1}{\psi(\frac{\gamma}{2})}\mathbb{E}\left[\sum_{t=1}^T \psi \circ d_1(\mathcal{A}(\mathcal{T}_{\pm,t}(\sigma_{<t})|L^{g^{\star}_{\sigma}}_{B_{< t}}), g^{\star, \alpha}_{\sigma}(\mathcal{T}_{\pm,t}(\sigma_{<t}))) \right]\\
\end{align*}
The first inequality follows from $\gamma$-shattering. Indeed, if $2\mathbbm{1}\{\mathcal{A}(\mathcal{T}_{\pm,t}(\sigma_{<t})|L^{g^{\star}_{\sigma}}_{B_{< t}}) \geq \mathcal{R}_t(\sigma_{<t})\} - 1 \neq h^{\star}_{\sigma}(\mathcal{T}_{\pm,t}(\sigma_{<t}))$, then $\mathcal{A}(\mathcal{T}_{\pm,t}(\sigma_{<t})|L^{g^{\star}_{\sigma}}_{B_{< t}})$ and $g^{\star}_{\sigma}$ must lie on opposite sides of the witness $\mathcal{R}_t(\sigma_{<t})$. The second inequality stems from the choice of $\alpha < \frac{\gamma}{2}$. The third inequality follows from the monotonicity of $\psi$. The last inequality follows from Markov's.  Now, we can continue like before. 
\begin{align*}
\mathbb{E}\Bigg[\sum_{t=1}^T \psi \circ d_1(&\mathcal{A}(\mathcal{T}_{\pm,t}(\sigma_{<t})|L^{g^{\star}}_{B_{< t}}), g^{\star, \alpha}_{\sigma}(\mathcal{T}_{\pm,t}(\sigma_{<t}))) \Bigg] \\
&= \mathbb{E}\left[\sum_{t=1}^T \psi \circ d_1(\mathcal{A}(\mathcal{T}_{\pm,t}(\sigma_{<t})|L^{g^{\star}}_{B_{< t}}), g^{\star, \alpha}_{\sigma}(\mathcal{T}_{\pm,t}(\sigma_{<t}))) \frac{\mathbbm{P}\left[B_t = 1 \right]}{\mathbbm{P}\left[B_t = 1 \right]}\right]\\
% &= \frac{T}{T^{\beta}}\mathbb{E}\left[\sum_{t=1}^T \psi \circ d_1(\mathcal{A}(\mathcal{T}_{\pm,t}(\sigma_{<t})|L^{g^{\star}}_{B_{< t}}), g^{\star, \alpha}_{\sigma}(\mathcal{T}_{\pm,t}(\sigma_{<t}))) \mathbbm{P}\left[B_t = 1 \right]\right]\\
&= \frac{T}{T^{\beta}}\mathbb{E}\left[\sum_{t=1}^T \psi \circ d_1(\mathcal{A}(\mathcal{T}_{\pm,t}(\sigma_{<t})|L^{g^{\star}}_{B_{< t}}), g^{\star, \alpha}_{\sigma}(\mathcal{T}_{\pm,t}(\sigma_{<t}))) \mathbbm{1}\{B_t = 1 \}\right]\\
\end{align*}
To see the last equality, note that the prediction $\mathcal{A}(\mathcal{T}_{\pm,t}(\sigma_{<t})|L^{g^{\star}}_{B_{< t}})$ only depends on bitstring ($B_1, \ldots, B_{t-1}$), the string $(\sigma_1, ..., \sigma_{t-1})$, and the internal randomness of $A$, all of which are independent of $B_t$. Thus, we have  
\begin{align*}
    \mathbb{E}\left[\psi \circ d_1(\mathcal{A}(\mathcal{T}_{\pm,t}(\sigma_{<t})|L^{g^{\star}}_{B_{< t}}), g^{\star, \alpha}_{\sigma}(\mathcal{T}_{\pm,t}(\sigma_{<t}))) \mathbbm{1}\{B_t = 1 \}\right]  &= \mathbb{E}\left[\psi_1 \circ d_1(\mathcal{A}_1(x_t|L^{g^{\star}_{\sigma}}_{B_{< t}}), y_t) \right] \, \mathbb{E}\left[ \indicator \{B_t=1\}\right]\\
    &= \mathbb{E}\left[\psi_1 \circ d_1(\mathcal{A}_1(x_t|L^{g^{\star}_{\sigma}}_{B_{< t}}), y_t) \right] \mathbbm{P}[B_t =1]
\end{align*}
\noindent as needed. Continuing onwards, 
\begin{align*}
   \mathbb{E}\Bigg[\sum_{t=1}^T \psi \circ d_1(&\mathcal{A}(\mathcal{T}_{\pm,t}(\sigma_{<t})|L^{g^{\star}_{\sigma}}_{B_{< t}}), g^{\star, \alpha}_{\sigma}(\mathcal{T}_{\pm,t}(\sigma_{<t}))) \Bigg] \\
   &= \frac{T}{T^{\beta}}\mathbb{E}\left[\sum_{t=1}^T \psi \circ d_1(\mathcal{A}(\mathcal{T}_{\pm,t}(\sigma_{<t})|L^{g^{\star}_{\sigma}}_{B_{< t}}), g^{\star, \alpha}_{\sigma}(\mathcal{T}_{\pm,t}(\sigma_{<t}))) \mathbbm{1}\{B_t = 1 \}\right]\\
   &= \frac{T}{T^{\beta}}\mathbb{E}\left[\sum_{t:B_t = 1} \psi \circ d_1(\mathcal{A}(\mathcal{T}_{\pm,t}(\sigma_{<t})|L^{g^{\star}_{\sigma}}_{B_{< t}}), g^{\star, \alpha}_{\sigma}(\mathcal{T}_{\pm,t}(\sigma_{<t})))\right]\\
    &= \frac{T}{T^{\beta}}\mathbb{E}\left[\mathbb{E}\left[\sum_{t: B_t = 1} \psi \circ d_1(\mathcal{A}(\mathcal{T}_{\pm,t}(\sigma_{<t})|L^{g^{\star}_{\sigma}}_{B_{< t}}), g^{\star, \alpha}_{\sigma}(\mathcal{T}_{\pm,t}(\sigma_{<t}))) \bigg|B\right]\right]\\
    &\leq \frac{T}{T^{\beta}}\mathbb{E}\left[\sum_{t: B_t = 1} \psi \circ d_1(g^{\star}_{\sigma}(\mathcal{T}_{\pm,t}(\sigma_{<t})), g^{\star, \alpha}_{\sigma}(\mathcal{T}_{\pm,t}(\sigma_{<t}))) + R_{\mathcal{A}}(|B|) \right]\\
    % &= \frac{T}{T^{\beta}}\mathbb{E}\left[\sum_{t: B_t = 1} \psi \circ d_1(g^{\star}(\mathcal{T}_{\pm,t}(\sigma_{<t})), g^{\star, \alpha}_{\sigma}(\mathcal{T}_{\pm,t}(\sigma_{<t})))\right] + \frac{T}{T^{\beta}}\mathbb{E}\left[ R_{\mathcal{A}}(|B|) \right]\\
\end{align*}
The last inequality follows from the fact that $\mathcal{A}$ is an online learner for $\psi \circ d_1$ with regret bound $R_{\mathcal{A}}(T)$ and is updated using a stream labeled by $g^{\star, \alpha}$ only when $B_t = 1$. Now, we can upperbound: 
\begin{align*}
\frac{T}{T^{\beta}}\mathbb{E}\left[\sum_{t: B_t = 1} \psi \circ d_1(g^{\star}(\mathcal{T}_{\pm,t}(\sigma_{<t})), g^{\star, \alpha}_{\sigma}(\mathcal{T}_{\pm,t}(\sigma_{<t})))\right] &+ \frac{T}{T^{\beta}}\mathbb{E}\left[ R_{\mathcal{A}}(|B|) \right]\\
&\leq \frac{T}{T^{\beta}}\mathbb{E}\left[\sum_{t: B_t = 1} \psi (\alpha) \right] + \frac{T}{T^{\beta}}\mathbb{E}\left[ R_{\mathcal{A}}(|B|) \right]\\
&\leq \frac{T}{T^{\beta}} \mathbb{E}\left[\alpha L|B|\right] + \frac{T}{T^{\beta}}\mathbb{E}\left[ R_{\mathcal{A}}(|B|) \right]\\
&= \alpha LT + \frac{T}{T^{\beta}}\mathbb{E}\left[ R_{\mathcal{A}}(|B|) \right]
\end{align*}
The first two inequalities follow from the fact that $\psi$ is monotonic, $L$-Lipschitz, $\psi(0) = 0$, and $d_1(g^{\star}(\mathcal{T}_{\pm,t}(\sigma_{<t})), g^{\star, \alpha}_{\sigma}(\mathcal{T}_{\pm,t}(\sigma_{<t})) \leq \alpha$.  Putting things together, we find that
\begin{align*}
    \mathbb{E}\left[\sum_{t=1}^T \mathbbm{1}\{\mathcal{Q}(\mathcal{T}_{\pm,t}(\sigma_{<t})) \neq \sigma_t \}\right]
    &\leq \mathbb{E}\left[\sum_{t=1}^T \mathbbm{1}\{E_{B, \phi_B^{g^{\star}_{\sigma}}}(\mathcal{T}_{\pm,t}(\sigma_{<t})) \neq \sigma_t \}\right] + \mathbb{E}\left[\sqrt{2T\ln(|\mathcal{E}_B|)}\right]\\
    &\leq   \frac{\alpha LT + \frac{T}{T^{\beta}}\mathbb{E}\left[ R_{\mathcal{A}}(|B|) \right]}{\psi(\frac{\gamma}{2})} + \mathbb{E}\left[\sqrt{2T\ln(|\mathcal{E}_B|)}\right]\\
    &\leq \frac{\alpha LT + \frac{T}{T^{\beta}}\mathbb{E}\left[ R_{\mathcal{A}}(|B|) \right]}{\psi(\frac{\gamma}{2})} + \mathbb{E}\left[\sqrt{2T|B|\ln(\frac{2}{\alpha})}\right].\\
\end{align*}
where the last inequality follows from the fact that that $|\mathcal{E}_B| \leq (\frac{2}{\alpha})^{|B|}$. By Jensen's inequality, we further get that, $\mathbb{E}\left[\sqrt{2T|B|\ln(\frac{2}{\alpha})}\right] \leq \sqrt{2T^{\beta + 1}\ln(\frac{2}{\alpha})}$, which implies that 
$$ \mathbb{E}\left[\sum_{t=1}^T \mathbbm{1}\{\mathcal{Q}(\mathcal{T}_{\pm,t}(\sigma_{<t})) \neq \sigma_t \}\right]  \leq \frac{\alpha LT + \frac{T}{T^{\beta}}\mathbb{E}\left[ R_{\mathcal{A}}(|B|) \right]}{\psi(\frac{\gamma}{2})} + \sqrt{2T^{\beta + 1}\ln(\frac{2}{\alpha})}.$$
Next,  by Lemma \ref{lem:woess}, there exists a concave sublinear function $ \overline{R}_{\Acal}(|B|)$ that upperbounds $R_{\Acal}(|B|)$. By Jensen's inequality, we obtain $\expect[\overline{R}_{\Acal}(|B|)] \leq \overline{R}_{\Acal}(T^{\beta})$, which yields
$$\mathbb{E}\left[\sum_{t=1}^T \mathbbm{1}\{\mathcal{Q}(\mathcal{T}_{\pm,t}(\sigma_{<t})) \neq \sigma_t \}\right]  \leq \frac{\alpha LT + \frac{T}{T^{\beta}} \overline{R}_{\mathcal{A}}(T^{\beta})}{\psi(\frac{\gamma}{2})} + \sqrt{2T^{\beta + 1}\ln(\frac{2}{\alpha})}.$$
Picking $\alpha = \frac{1}{LT}$ and $\beta \in (0, 1)$, gives that $\mathcal{Q}$ enjoys sublinear expected regret
$$\mathbb{E}\left[\sum_{t=1}^T \mathbbm{1}\{\mathcal{Q}(\mathcal{T}_{\pm,t}(\sigma_{<t})) \neq \sigma_t \}\right] \leq  \frac{1}{\psi(\frac{\gamma}{2})} + \frac{T}{\psi(\frac{\gamma}{2})T^{\beta}}\overline{R}_{\mathcal{A}}(T^{\beta}) + \sqrt{4T^{\beta + 1}\ln(LT)}.$$
Since $\overline{R}_{\mathcal{A}}(T^{\beta}) $ is sublinear in $T^{\beta}$, $\mathcal{Q}$ is a realizable online learner for $\mathcal{H}$ with \textit{sublinear} regret. Thus, for a sufficiently large $T$, $\mathbb{E}\left[\sum_{t=1}^T \mathbbm{1}\{\mathcal{Q}(\mathcal{T}_{\pm,t}(\sigma_{<t})) \neq \sigma_t \}\right] < \frac{T}{2}$. This is a contradiction because  $\{(\mathcal{T}_{\pm,t}(\sigma_{<t}), \sigma_t)\}_{t=1}^T$ is a realizable sequence of instances corresponding to a root-to-leaf path in $\mathcal{T}_{\pm}$ chosen \textit{uniformly at random} and thus any realizable online learner must suffer expected regret at least $\frac{T}{2}$. Thus, if there exists a scale $\gamma > 0$ such that $\text{fat}^{\text{seq}}_{\gamma}(\mathcal{G}) = \infty$, there cannot exist an online learner for $\mathcal{G}$ with respect to $\psi \circ d_1$.
\end{proof}

\end{document}